\documentclass[11pt]{article}
\usepackage[utf8]{inputenc}
\usepackage[T1]{fontenc}
\usepackage{graphicx,amsmath,amsfonts,amssymb,bm,url,color,latexsym,amsthm}
\usepackage[numbers]{natbib}
\usepackage{wrapfig}
\usepackage[small,bf]{caption}
\usepackage[outdir=./]{epstopdf}
\usepackage{subcaption}
\usepackage{bbm}
\usepackage{microtype}
\usepackage[toc, page]{appendix}
\usepackage[ruled]{algorithm2e}
\usepackage{algpseudocode}
\usepackage{xr-hyper}
\usepackage{hyperref}
\usepackage[capitalize]{cleveref}
\usepackage{verbatim}
\usepackage{framed}
\usepackage{comment}
\usepackage{tikz}
\usepackage{fge}
\usepackage{mathdots}
\usepackage{enumitem}

\makeatletter
\newcommand*{\addFileDependency}[1]{
  \typeout{(#1)}
  \@addtofilelist{#1}
  \IfFileExists{#1}{}{\typeout{No file #1.}}
}
\makeatother

\allowdisplaybreaks
\numberwithin{equation}{section}
\hypersetup{
    colorlinks=true,%
    citecolor=blue,%
    filecolor=blue,%
    linkcolor=blue,%
    urlcolor=blue
}

\newtheorem{theorem}{Theorem}
\newtheorem{lemma}{Lemma}
\newtheorem{corollary}{Corollary}

\newtheorem{assumption}{Assumption}
\newtheorem{fact}{Fact}

\theoremstyle{remark}
\newtheorem{remark}{Remark}

\newcommand{\mb}{\mathbf}
\newcommand{\mc}{\mathcal}

\newcommand{\bb}{\mathbb}

\newcommand{\eps}{\varepsilon}

\newcommand{\E}{\bb E}

\newcommand{\norm}[2]{\left\| #1 \right\|_{#2}}
\newcommand{\abs}[1]{\left| #1 \right|}
\newcommand{\innerprod}[2]{\left\langle #1,  #2 \right\rangle}

\newcommand{\R}{\bb R}

\newcommand{\Z}{\bb Z}

\newcommand{\Sp}{\bb S}

\renewcommand{\P}{\mathbb{P}}
\newcommand{\EE}{\mathbb E}

\newcommand{\Brac}[1]{\left\{#1 \right\} }

\newcommand{\paren}[1]{ \left( #1 \right) }

\DeclareMathOperator{\diag}{diag}

\DeclareMathOperator{\rank}{rank}

\DeclareMathOperator{\grad}{grad}

\DeclareMathOperator*{\argmin}{argmin}
\DeclareMathOperator*{\argmax}{argmax}

\DeclareFontFamily{U}{mathx}{\hyphenchar\font45}
\DeclareFontShape{U}{mathx}{m}{n}{
      <5> <6> <7> <8> <9> <10>
      <10.95> <12> <14.4> <17.28> <20.74> <24.88>
      mathx10
      }{}
\DeclareSymbolFont{mathx}{U}{mathx}{m}{n}
\DeclareFontSubstitution{U}{mathx}{m}{n}
\DeclareMathAccent{\widecheck}{0}{mathx}{"71}

\newcommand{\wh}{\widehat}
\newcommand{\wc}{\widecheck}
\newcommand{\wt}{\widetilde}

\newcommand{\bY}{\mb Y}
\newcommand{\bX}{\mb X}
\newcommand{\bA}{\mb A}
\newcommand{\bq}{q}

\newcommand{\bN}{\mb N}
\newcommand{\bI}{\mathbf{I}}
\newcommand{\bP}{\mb P}

\newcommand{\bD}{\mb D}
\newcommand{\bU}{\mb U}
\newcommand{\bV}{\mb V}
\newcommand{\bE}{\mb E}
\newcommand{\bL}{\mb L}

\newcommand{\bW}{\mb W}
\newcommand{\bZ}{\mb Z}
\newcommand{\bS}{\mb S}

\newcommand{\bzeta}{\mb \zeta}
\newcommand{\bLambda}{\mb \Lambda}
\newcommand{\bSigma}{\mb\Sigma}

\newcommand{\bB}{\mb B}
\newcommand{\bC}{\mb C}
\newcommand{\bR}{\mb R}
\newcommand{\bDelta}{\mb \Delta}
\newcommand{\bM}{\mb M}
\newcommand{\bQ}{\mb Q}
\newcommand{\bG}{\mb G}
\newcommand{\bOmega}{\mb \Omega}
\newcommand{\bdelta}{\mb \delta}
\newcommand{\bH}{\mb H}
\newcommand{\bO}{\mb O}

\newcommand{\cE}{\mathcal E}
\newcommand{\cN}{\mathcal N}
\newcommand{\cA}{\mathcal A}

\newcommand{\0}{\mb 0}

\newcommand{\cO}{\mathcal{O}}
\newcommand{\bbO}{\mathbb{O}}

\newcommand{\T}{\top}

\newcommand{\tr}{{\rm tr}}
\newcommand{\op}{{\rm op}}
\newcommand{\Cov}{{\rm Cov}}
\newcommand{\Var}{{\rm Var}}
\newcommand{\rI}{{\rm I}}
\newcommand{\rII}{{\rm II}}
\newcommand{\rIII}{{\rm III}}
\newcommand{\rIV}{{\rm IV}}
\newcommand{\init}{{\rm init}}
\newcommand{\rvec}{{\rm vec}}

\makeatletter
\newcommand*{\rom}[1]{\expandafter\@slowromancap\romannumeral #1@}
\makeatother
\usepackage{fullpage}
\usepackage{authblk}
\begin{document}
	
	\title{Optimal vintage factor analysis with deflation varimax}

\author[1]{Xin Bing\thanks{xin.bing@utoronto.ca}}
\author[2]{Xin He\thanks{he.xin17@mail.shufe.edu.cn}}
\author[3]{Dian Jin\thanks{dj370@scarletmail.rutgers.edu}}
\author[3]{Yuqian Zhang\thanks{yqz.zhang@rutgers.edu}}
\affil[1]{Department of Statistical Sciences, University of Toronto}

\affil[2]{School of Statistics and Data Science, Shanghai University of Finance and Economics}
\affil[3]{Department of Electrical and Computer Engineering, Rutgers University, New Brunswick}

\maketitle
\begin{abstract}
    Vintage factor analysis is one important type of factor analysis that aims to first find a low-dimensional representation of the original data, and then to seek a rotation such that the rotated low-dimensional representation is scientifically meaningful. Perhaps the most widely used vintage factor analysis is the Principal Component Analysis (PCA) followed by the varimax rotation. Despite its popularity, little theoretical guarantee can be provided mainly because varimax rotation requires to solve a non-convex optimization over the set of orthogonal matrices.  

    In this paper, we propose a deflation varimax procedure that solves each row of an orthogonal matrix sequentially. In addition to its net computational gain and flexibility, we are able to fully establish theoretical guarantees for the proposed procedure in a broad context.
    Adopting this new varimax approach as the second step after PCA, we 
    further analyze this two step procedure under a general class of factor models. Our results show that it estimates the factor loading matrix in the optimal rate when the signal-to-noise-ratio (SNR) is moderate or large. In the low SNR regime, we offer possible improvement over using PCA and the deflation procedure when the additive noise under the factor model is structured. The modified procedure is shown to be optimal in all SNR regimes.  Our theory is valid for finite sample and allows the number of the latent factors to grow with the sample size as well as the ambient dimension to grow with, or even exceed, the sample size. 
    Extensive simulation and real data analysis further corroborate our theoretical findings. 
\end{abstract}
	
	
	
	

	\section{Introduction}
	\label{sec:intro}
	Vintage factor analysis is one important type of classical factor analysis, first developed by \cite{Thurstone1935,Thurstone1947} and coined by \cite{kaiser1958varimax}. It aims to first find a low-dimensional representation of the original data, and then to find a rotation such that the rotated low-dimensional representation is ``scientifically meaningful'' \citep{Thurstone1935}. One widely used vintage factor analysis is the Principal Component Analysis (PCA) followed by the varimax rotation, first proposed by \cite{kaiser1958varimax}. In this paper, we propose a variant of such two-step approach and analyze it under the following latent factor model 
	\begin{equation}\label{model_X}
		X = \bLambda Z + E
	\end{equation}
	where $X\in \R^p$ is an observable $p$-dimensional random vector, $Z\in\R^r$ is an unobserved, latent factor with some $r\le p$ and $E\in \R^p$ is an additive error that is independent of $Z$. The loading matrix  $\bLambda\in \R^{p\times r}$ is treated as deterministic and is the main quantity of our interest.  
	
	Suppose the observed $p\times n$ data  matrix $\bX = (X_1,\ldots, X_n)$ consists of $n$ i.i.d. samples of $X$ according to \eqref{model_X}. The  first $r$ Principal Components (PCs)  of $\bX$ obtained from PCA are used as the low-dimensional representation of $\bZ = (Z_1,\ldots, Z_n)$. Collect these PCs in the $r\times n$ matrix $\bU_{(r)}$ subjecting to $\bU_{(r)} \bU_{(r)}^\T = n~ \bI_r$. Varimax rotation then seeks an $r\times r$ orthogonal matrix $\bQ$ by solving the following optimization problem\footnote{The second term in the parenthesis does not affect  the maximizer due to orthonormal rows of $\bU_{(r)}$.}
	\begin{equation}\label{obj_varimax}
		\begin{split}
			&\argmax_{\bQ\in\bbO_{r\times r}} ~ 
			\sum_{k=1}^r  \sum_{i=1}^n \left(
			[\bQ\bU_{(r)}]_{ki}^4 - \left({1\over n}\sum_{i=1}^n  [\bQ\bU_{(r)}]_{ki}^2\right)^2
			\right) 
		\end{split}
	\end{equation} 
	We use the notion $\bbO_{p\times r}$ to denote the set of $p\times r$ matrices with orthonormal columns. 
	Despite being used widely,  there has been long-standing controversy on the meaning and interpretation of varimax rotation. Recently, \cite{rohe2020vintage} shows that the global solution to the varimax criterion in \eqref{obj_varimax} indeed makes statistical inference when the factor $Z$ contains independent components, and each of which has a leptokurtic distribution, an assumption we will also adopt in this paper (see, \cref{ass_Z} of \cref{sec_assumption} for details).
	However,  optimizing over the Stiefel manifold $\bbO_{r\times r}$ renders \eqref{obj_varimax} difficult to solve computationally. Furthermore, \eqref{obj_varimax} is highly non-convex and its optimization landscape is  notoriously difficult to analyze due to existence of various local solutions and spurious stationary points. Consequently, little theoretical guarantee to date can be made for any of the existing algorithms to solve \eqref{obj_varimax}.
	
	In this paper, we propose a new deflation varimax approach that computes each row of an orthogonal matrix in a sequential manner by solving a variant of \eqref{obj_varimax}.  Comparing to  solving an $r\times r$ orthogonal matrix simultaneously from \eqref{obj_varimax}, computing each row one at a time is in general more computationally efficient, and is also more flexible in practice as one can stop after obtaining  certain number of  rows.  In addition to its net computational gain and flexibility, we are able to provide complete theoretical guarantees on the obtained rotation matrix in a broader context. Under model \eqref{model_X}, we further adopt this deflation varimax procedure together with PCA to estimate the loading matrix $\bLambda$. Our theory shows that this approach estimates the loading matrix $\bLambda$ optimally in the minimax sense. 

	 \subsection{Related literature}\label{sec_literature} 
	 We begin by discussing related literature.\\ 
	 
	 {\em Factor analysis and factor rotation.} 
	Model \eqref{model_X} has been extensively studied in the literature of classical factor analysis where the primary interest is typically on estimation of the column space of the loading matrix $\bLambda$, rather than $\bLambda$ itself.  See, for instance, \cite{hotelling1933analysis,lawley1962factor,anderson1956statistical}.   Although there exists various identification conditions on $\bLambda$ and the covariance matrices of $Z$ and $E$ under which $\bLambda$ becomes identifiable and can be consistently estimated \citep{BaiLi2012}, most of them are ``more or less arbitrary ways of determining the factor loadings uniquely. They do not correspond to any theoretical considerations of psychology; there is no inherent meaning on them.'' \citep{anderson1956statistical}.  For the same reason, a large body of the literature focus on recovering the latent space of  the  factor $Z$ under model \eqref{model_X}. PCA is arguably one of the most popular methods for predicting low-dimensional spaces, and a few leading PCs are typically used  for predicting $\bZ$ up to  orthogonal transformations of its rows. Analysis of the PCs for such purpose have been studied in existing literature   \cite{Bai-factor-model-03,bai2008large,SW2012,fan2013large,bing2020prediction}, just to name a few. 
	 
	 A different line of factor analysis that aims to identify and estimate the loading matrix $\bLambda$ is studied recently in \cite{bunea2020model,bing2020clustering,bing2023detecting} where $\bLambda$ is assumed to have multiple rows that are parallel to the canonical basis in $\R^r$ and the noise $E$ needs to have independent components. In this paper we do not resort to these conditions and instead require components of $Z$ to be independent and leptokurtic.  Under this line of assumption, a recent work \cite{rohe2020vintage} studies the rotated PCs by using varimax rotation under a different class of latent factor model. Their results reveal that performing varimax rotation eliminates the indeterminacy due to orthogonal transformations  so that the rotated PCs consistently predict $\bZ$ only up to sign flip and permutation of its rows. 
	 However, the  results in \cite{rohe2020vintage} only hold  for the global solution to \eqref{obj_varimax}.
  
     Finding meaningful rotation in factor analysis has been a traditional topic, dating back to 1940s, see, \cite{holzinger1941factor,mulaik2009foundations} for detailed reviews and other references in the more recent work \cite{rohe2020vintage}. Orthogonal rotation due to its analytic tractability has been firstly studied and widely used since then. The two most popular examples are the quartimax rotation \citep{NeuhausWrigley} (see, also its  variant but equivalent criteria in \cite{carroll1953analytical,saunders1953analytic,ferguson1954learning}) and the (raw) varimax rotation in \eqref{obj_varimax}. Criteria of these orthogonal rotations are mathematically formulated to explicate Thurstone's five formal rules for simple structure  \citep{Thurstone1947}. It is worth mentioning that when the rows of $\bU_{(r)}$ are orthonormal, both the quartimax rotation and the varimax rotation are equivalent. In \cite{kaiser1958varimax}, a variant of the varimax rotation based on normalizing the $\ell_2$-norm of each column of $\bU_{(r)}$ in \eqref{obj_varimax} is further advocated, but later in \cite{kaiser1974little} the authors suggested to remove such normalization. On the other hand, computing the optimal orthogonal rotation over the Stiefel manifold has been a longstanding problem. In \cite{kaiser1958varimax,Kaiser1959} a closed-form, necessary condition for solving a pair of orthonormal rows is derived and the author proposes to repeat such procedure for solving all distinct pairs of rows. This procedure becomes computationally expensive for moderate $r$ and large $n$, and is also found to accumulate errors rapidly. Later in \cite{horst1965factor} and \cite{sherin1966matrix} fixed-point iterative algorithms are proposed to solve \eqref{obj_varimax} for an orthogonal matrix simultaneously. Both algorithms are in vein similar to the projected gradient ascent algorithm which has been widely adopted for solving \eqref{obj_varimax}, and is still used to date in the \textsc{R} package \textsf{varimax}. Unfortunately, to the best of our knowledge, none of the existing algorithms has provable guarantee.\\

	 {\em Independent Component Analysis (ICA) and complete dictionary learning.} ICA is one particular factor analysis that aims to extract independent, non-Gaussian signals \citep{hyvarinen2000independent}. A popular ICA model assumes that the observable feature $X$ follows model \eqref{model_X} with $E=0$ and $p = r$.
	 In the ICA literature, one of the most widely used algorithms for estimating $\bLambda$ is the FastICA  which has two versions. The symmetric FastICA\footnote{FastICA has the flexibility of choosing different loss functions. In our context, we compare with the versions using kurtosis (the $\ell_4$-norm).} \citep{hyvarinen2000independent} is in vein similar to solving the varimax criterion in \eqref{obj_varimax} via a fixed-point algorithm while the deflation FastICA \citep{hyvarinen1997fast} is its sequential counterpart that aims to solve each row one at a time. Algorithmically, \cite{hyvarinen1997fast} shows that the symmetric FastICA  converges  to a stationary point in a cubic rate  while  \cite{Douglas2003}  characterizes the stationary points of deflation FastICA   at the population level when we have infinitely many samples. Although statistical properties of both versions of FastICA have been studied in various papers, for instance,  \cite{Hyvarinen97,OjaYuan06,tichavsky2006performance,ollila2009deflation,hyvarinen2000independent,nordhausen2018independent}, these results are all established for the global solution instead of the actual ones computed by FastICA. Algorithms with both polynomial running time and provable guarantees are recently studied in \cite{anandkumar2014tensor,anandkumar2014guaranteed}. Their procedure is based on tensor decomposition via power iterations. More recently, \cite{auddy2023} establishes the minimax optimal rate of convergence, and the optimal sample complexity,  of estimating $\bLambda$ under the ICA model.  Moreover,  inference of entries in $\bLambda$ is also studied in \cite{auddy2023}.

	  Progress on understanding the optimization landscape  and providing provable guarantees of solving \eqref{obj_varimax} in a deflation manner  has also been made in the past decade in a sub-model of ICA where entries of $Z$ are further assumed to be i.i.d. from a Bernoulli-Gaussian distribution \citep{spielman2013exact,sun2016complete,qu2020geometric,zhai2020complete, zhang2018structured,jin2021unique}. This sub-model finds its application in complete Dictionary Learning (DL), an important problem in the field of compressed sensing and signal recovery. In their language, columns of $\bLambda$ are called atoms and they form the dictionary which is the main quantity of  interest.  However, to the best of our knowledge, the existing results under this sub-model are only established  for the noise-free case, namely, $E = 0$ (see also \cref{rem_analysis} for detailed comparison).

	 Although some of the aforementioned results in the ICA and DL literature could potentially be helpful for solving and analyzing the varimax rotation, the matrix  $\bU_{(r)}$ in \eqref{obj_varimax} unfortunately  has both an additive error and an approximation error in addition to the signal that we want to recover. What's worse is that $\bU_{(r)}$  also depends on an  unknown  rotation matrix which itself is random and dependent on the data $\bX$ (see details in \cref{thm_PCs}). Dealing with such complication requires new procedures as well as new analysis. We defer to \cref{rem_analysis} of \cref{sec_theory_alg} for more detailed technical comparison with the existing results of ICA and DL.

	\subsection{Our contributions}\label{sec_contri}
	
	
	
	Our first contribution is to propose in  \cref{sec_method} a new approach, {\em deflation varimax rotation}, of computing an orthogonal rotation matrix row by row via solving the following variant of the varimax criterion in \eqref{obj_varimax},
	\begin{align}\label{obj_L4_intro}
		\min_{q\in \mathbb{S}^{r}}  - \bigl\|\bU_{(r)}^\T  ~ q \bigr\|_4^4.
	\end{align}
	Here $\mathbb{S}^{r} = \{v\in \R^r: \|v\|_2 = 1\}$ denotes the unit sphere in $\R^r$. To solve \eqref{obj_L4_intro}, we adopt the projected gradient descent (PGD) algorithm which is widely used for solving constrained optimization problems due to its simplicity and computational efficiency. 
	Specifically, for computing each row of the rotation matrix, we propose to solve  the  {\em same} optimization problem  \eqref{obj_L4_intro} via PGD  but from {\em different} initializations. Choosing initialization is crucial due to the non-convexity of \eqref{obj_L4_intro}. In \cref{sec_init} we offer and analyze two possible initialization schemes: one is based on random initialization, and the other one is based on  method of moments estimation which is of interest on its own. After computing all rows, a symmetric orthogonalization is performed to ensure that the final rotation matrix is orthogonal.  	In \cref{sec_method}, we further adopt the proposed deflation varimax rotation 
	in conjunction with PCA to estimate the loading  matrix $\bLambda$ under model \eqref{model_X}. This new procedure  is termed as  {\em PCA with deflation varimax}.


	To provide theoretical guarantees for PCA with deflation varimax, our second contribution is to establish a new decomposition for the  matrix $\bU_{(r)}$ of the PCs.  We write the singular value decomposition (SVD) of the loading matrix $\bLambda=\bL \bS\bA$ with $\bL\in\bbO_{p\times r}$ and $\bA \in \bbO_{r\times r}$. We prove a new composition in \cref{thm_PCs} of \cref{sec_theory}   that  
	\begin{equation}\label{decomp_Ur_intro}
		\bU_{(r)} =  \bR^\T  \left(\bA  \bZ +    \bN  +   \bOmega\right) \in \R^{r\times n},
	\end{equation}
	where $\bR$ is some $r\times r$ random orthogonal matrix, $\bN$ is an additive error  with  i.i.d. columns  and $\bOmega$ can be viewed as the approximation error of using PCA for estimating $\bL$ and  $\bS$ (see, the paragraph after \cref{thm_PCs} for detailed discussion). Our results in \cref{thm_PCs} also state the explicit structure of the additive error $\bN$, and quantify the approximation error $\bOmega$. 
	As mentioned before, the existing analysis of the PCs either assumes additional identifiability conditions or  analyze $(\bU_{(r)} - \bH \bZ)$ for an $r\times r$ invertible matrix $\bH$ which is pre-specified mainly for mathematical convenience. Our new analysis of $\bU_{(r)}$ requires weaker conditions, and is more tailored for analyzing the subsequent varimax rotation. Indeed, the decomposition in \eqref{decomp_Ur_intro} suggests the deflation varimax rotation applied to $\bU_{(r)}$ estimate the orthogonal matrix $\bar{\bA} := \bR^\T\bA$.

	
	Our third contribution is to  provide a complete analysis   of the proposed deflation varimax rotation in \cref{sec_theory}.  In \cref{thm_critical} of \cref{sec_theory_alg}
	we first analyze the optimization landscape of \eqref{obj_L4_intro} by providing explicit characterization of its  stationary points, in particular, the ones that are close to columns of $\bar{\bA}$. Such result is generic, and is independent of {\em any} optimization algorithm one might deploy to solve \eqref{obj_L4_intro}. 
    The proximity of the stationary points, also known as the statistical error, in \cref{thm_critical} depends on the interplay of the approximation error $\bOmega$, the covariance matrix of columns in the additive error $\bN$, the low-dimension $r$  as well as the sample size $n$. 
    Particularizing to PGD, we show in \cref{thm_one_col}  that any stationary point from using PGD to solve   \eqref{obj_L4_intro}  provably recovers one column of $\bar{\bA}$ provided that the initialization of PGD lies in the vicinity of that column.  Our results also capture such requirement of initialization and are valid  essentially for {\em any generic} initialization scheme. In \cref{sec_init} we discuss and analyze two specific initialization schemes which are new for their own interests.  
	In \cref{thm_RA} of \cref{sec_theory_alg}, we further establish theoretical guarantees for the whole deflation varimax rotation obtained by solving \eqref{obj_L4_intro} $r$ times and followed by a symmetric orthogonalization. Our results show that the deflation varimax estimates $\bar{\bA}$ up to an $r\times r$ signed permutation matrix  and its estimation error in the Frobenius norm is shown to only increase {\em linearly in the dimension $r$} comparing to that of estimating one column. 
	 Although our results in \cref{sec_theory_alg} are mainly stated for the deflation varimax rotation applied to the PCs, our analysis are in fact carried out under a broader context where both $\bN$ and $\bOmega$ in \eqref{decomp_Ur_intro} are allowed to be in general forms.  As a result, 
	both our method and theory could be applied to other applications including both ICA and  complete DL. See, \cref{rem_analysis} for details.
	

	
	Our fourth contribution is to establish the minimax optimal rates of estimating the loading matrix $\bLambda$ under model \eqref{model_X}. In \cref{thm_A_general} of  \cref{sec_rotation_loading}, we 
	derive finite sample upper bounds of the estimation error in the Frobenius norm for the proposed estimator  using PCA with deflation varimax.
	To benchmark the upper bound,  we further establish in \cref{thm_lowerbounds}  minimax lower bounds of estimating $\bLambda$ under  model \eqref{model_X}. As a result, PCA with deflation varimax is minimax optimal (up to a logarithmic factor of $n$) when the signal-to-noise-ratio (SNR) in our problem is moderate or large. 
	In the weak signal  regime, our theory reveals that the sub-optimality of PCA with deflation varimax is caused by both the additive error $\bN$ via PGD  and the approximation error $\bOmega$ via PCA.  In \cref{sec_ext} we consider when the noise $E$ under model \eqref{model_X} is structured and offer two steps of improvement over using  PCA in the first step and deploying PGD to solve \eqref{obj_L4_intro} in the second step. Specializing to the case when the covariance matrix of $E$ is proportional to the identity matrix, we show in \cref{thm_A_corr}  that the sub-optimality of PCA with deflation varimax could be remedied and 
	the improved estimator  achieves the minimax optimal rate in all SNR regimes. \\

	This paper is organized as follows. In \cref{sec_method} we state both the deflation varimax rotation for computing the rotation matrix and the proposed PCA with deflation varimax for estimating the loading matrix. In \cref{sec_assumption} we state the main assumptions used in our analysis. Theoretical analysis of the decomposition of the PCs is presented in \cref{sec_PCs}. \cref{sec_theory_alg} contains a complete theory of analyzing the deflation varimax rotation. In \cref{sec_rotation_loading} we  analyze the proposed estimator of the loading matrix. Initialization schemes are discussed  in \cref{sec_init}. In \cref{sec_ext} we discuss possible improvements when the noise $E$ is structured. \cref{sec_real_data} contains our real data analysis  while simulation studies are presented in \cref{sec_sim}. All the proof is deferred to the appendix.

	\subsection{Notations}
	
	For any positive integer $d$, we write $[d] := \{1,\ldots, d\}$.
	For any vector $v$, we use $\|v\|_q$ to denote its $\ell_q$ norm for $0\le q\le \infty$.  
	We use uppercase bold fonts for matrices.  For any $\bM\in \R^{d\times q}$, we use $\sigma_1(\bM)\ge \sigma_2(\bM)\ge \cdots \ge \sigma_{\min(d,q)}(\bM)$ to denote its singular values. We define the operator norm $\|\bM\|_{\op}=\sigma_1(\bM)$ and the Frobenius norm $\|\bM\|_F^2 = \sum_{i,j}M_{ij}^2$. For any $i\in [d]$ and $j\in [q]$, we use $M_{i\cdot}$ and $M_j$ to denote the $i$th row and the $j$th column, respectively. 
	For a symmetric positive semi-definite matrix $\bS\in \R^{p\times p}$, we use $\lambda_1(\bS)\ge \lambda_2(\bS)\ge \cdots \ge \lambda_p(\bS)$ to denote its eigenvalues. 
	For any two sequences $a_n$ and $b_n$, we write $a_n\lesssim b_n$ if there exists some constant $C$ such that $a_n \le Cb_n$. 
	The notation $a_n\asymp b_n$ stands for $a_n \lesssim b_n$ and $b_n \lesssim a_n$. 
	We use $\bI_d$ to denote the $d\times d$ identity matrix and use $1$ ($0$) to denote the vector with all ones (zeroes) whose dimension depends on the context. For $d_1\ge d_2$, we use $\bbO_{d_1\times d_2}$ to denote the set of all $d_1\times d_2$ matrices with orthonormal columns. We use $\Sp^d$ to denote the unit sphere in $\R^d$.
	Lastly, we use $c,c',C,C'$ to denote positive absolute constants that unless otherwise indicated may vary from line to line.

	\section{Methodology}\label{sec_method}

	In this section, we state our procedure of estimating the loading matrix $\bLambda$ for a given low-dimension $r$. Selection of $r$ is  an important problem and we defer to  \cref{sec_select_r} for detailed discussion. As mentioned in the Introduction, we first apply PCA to the $p\times n$ data matrix $\bX$. Specifically, write the eigen-decomposition 
	\begin{equation}\label{eq_eigen_Y}
		{1\over n}\bX\bX^\T = \sum_{j=1}^r d_j  v_j v_j^\T + \sum_{j=r+1}^p d_j v_j v_j^\T 
	\end{equation}
	with $\bD_{(r)}= \diag(d_1, \ldots, d_r)$ containing the largest $r$ non-zero eigenvalues in non-increasing order, and the corresponding eigenvectors $\bV_{(r)} = (v_1, \ldots, v_r)$. The first $r$ principal components of $\bX$ are 
	\begin{equation}\label{def_PCs}
		\bU_{(r)}  := \bD_{(r)}^{-1/2}\bV_{(r)}^\T \bX \in \R^{r\times n}.
	\end{equation}
	Due to $\bV_{(r)} \in \bbO_{p\times r}$, we have   $\bU_{(r)} \bU_{(r)}^\T = n\bI_r$. 
	
	
	\subsection{Deflation varimax rotation}\label{sec_method_varimax}
	
	The second step after PCA is to compute an $r\times r$ orthogonal matrix based on the first $r$ PCs. 
	Instead of solving the varimax criterion in \eqref{obj_varimax}, we propose to compute its rows one at a time by solving the following optimization problem 
	\begin{align}\label{obj_L4}
		\min_{q\in \mathbb{S}^{r}}~ F\left(q; \bU_{(r)}\right),\qquad \text{with}\quad  F(q;\bU_{(r)}) =-\frac{1}{12 n}\bigl\|\bU_{(r)}^\T q\bigr\|_4^4.
	\end{align}
	To solve \eqref{obj_L4}, we adopt the projected gradient descent (PGD) algorithm due to its computational efficiency. We choose to use the Riemannian gradient over the classical gradient as the former is more suitable for solving constrained optimization problem  over manifolds. Specifically, let $\grad F$ be the Riemannian gradient of $F$ over the unit sphere $\Sp^r$, defined as 
	\begin{equation}\label{def_grad_F}
		\grad F\left(q; \bU_{(r)}\right) 
		=  -{1\over 3  n} \bP_{q}^\perp \sum_{t=1}^n \left(q^\T  U_t\right)^3   U_t.
	\end{equation}
	Such basic definition is reviewed in \cref{app_pre}. Here we write $\bU_{(r)} = (U_1, \ldots, U_n)$. Let $P_{\Sp^r}$ be the projection of any $r$-dimensional vector to $\Sp^r$. 
	For each $k \in [r]$, we solve the $k$th row by iterating
	\begin{align}\label{iter_PGD}
		\wh  Q_k^{(\ell+1)}= P_{\Sp^r}\left(\wh  Q_k^{(\ell)} -\gamma\grad F\bigl(\wh  Q_k^{(\ell)}; \bU_{(r)}\bigr)\right),\quad \text{for }\ \ell = 0,1,2,\ldots
	\end{align} 
	until convergence. 
	Here $\gamma>0$ is  some step size and we defer its choice to \cref{thm_one_col} and \cref{rem_stepsize}. 
	The only difference across $k \in [r]$ lies in choosing the initializations in \eqref{iter_PGD} for $\ell = 0$. In \cref{sec_init} we offer  two initialization schemes.

	After computing  $\wh\bQ = (\wh  Q_1,\ldots, \wh  Q_r)$ from \eqref{iter_PGD},  we further orthogonalize it by solving
	\begin{equation}\label{def_wc_Q}
		\wc \bQ = \argmin_{\bQ\in\bbO_{r\times r}}\|\wh\bQ - \bQ\|_F.
	\end{equation}
	Problem \eqref{def_wc_Q} can be computed simply via the singular value decomposition of $\wh\bQ$. 
	
	Finally, the  rotation matrix $\wc\bQ$ can be used to estimate the loading matrix $\bLambda$ as
	\begin{equation}\label{def_Lambda_hat}
		\wh \bLambda = {\bV_{(r)} \bD_{(r)}^{1/2}\wc\bQ \over \|\bV_{(r)} \bD_{(r)}^{1/2}\wc\bQ\|_\op}.
	\end{equation} 
	The normalization is due to the scaling constraint $\sigma_1(\bLambda) =1$ in \cref{ass_A_general} of \cref{sec_assumption}.
	For the reader's convenience, we summarize in \cref{alg_rotate}  the deflation varimax rotation  as well as  in \cref{alg_A} the PCA with deflation varimax.	
	
	\begin{remark}[Comparison with the classical deflation procedure]\label{rem_alg}
		The commonly used strategy of sequentially recovering orthonormal vectors is to consider the constrained optimization problem based on previously recovered solutions. In our context, for computing $\wh Q_k$ with $2\le k\le r$, this corresponds to solving \eqref{obj_L4} over $\bq\in\Sp^r$ subject to the constraint $\bq^\T \wh Q_{i} = 0$ for all $1\le i\le k-1$. 
		Although this is a natural deflation procedure that produces an orthonormal $\wh\bQ$, it oftentimes has unstable practical performance as the resulting estimator could suffer from exploding error that accumulates exponentially as $r$ increases \citep{arora2012provable,vempala2011structure} (see, also  \cite{ollila2009deflation,arora2012provable,wei2015convergence} in other similar problems). In our context, we find that such sequentially {\em constrained} optimization problems are not needed. Instead, the same {\em unconstrained} optimization coupled with proper initializations suffices to guarantee optimal estimation (see, Theorems \ref{thm_RA}, \ref{thm_A_general} and \ref{thm_A_corr}). 

	\end{remark}

	\begin{algorithm}[H]
		\caption{Deflation varimax rotation}\label{alg_rotate}
		\SetAlgoLined 
		\KwData{A matrix $\bU_{(r)}\in\R^{r\times n}$ and a positive integer $1\le s\le \min\{r,n\}$.}
		\KwResult{A matrix $\wc\bQ\in \bbO_{r\times s}$.}
		Set $\wh\bQ =\0$\;
		\For{$k = \{1,2,\ldots, s\}$}{
			Obtain any stationary point $\wh  Q_k$ from solving 
			by  (\ref{iter_PGD})\;
		}  
		$\textrm{SVD}(\wh\bQ) =  \bU \bSigma \bV^\T$ with $\bU\in\bbO_{r\times s}$ and $\bV\in \bbO_{s\times s}$\;
		Compute $\wc\bQ= \bU  \bV^\T$.
	\end{algorithm}

	\begin{algorithm}[H]
		\caption{Principal component analysis with deflation varimax}\label{alg_A}
		\SetAlgoLined
		\KwData{A matrix $\bX\in\R^{p\times n}$ and a positive integer $1\le r\le \min\{p,n\}$.}
		\KwResult{A matrix $\wh\bLambda\in \R^{p\times r}$ with $\sigma_1(\wh\bLambda) = 1$.}
		Compute $\bV_{(r)}$ and $\bD_{(r)}^{1/2}$ from \eqref{eq_eigen_Y}\;
		Compute $\bU_{(r)}$ from \eqref{def_PCs}\;  
		Apply \cref{alg_rotate} to $\bU_{(r)}$ with $s = r$ to obtain $\wc\bQ \in \bbO_{r\times r}$\;
		Compute $\wh\bLambda$ from \eqref{def_Lambda_hat}.
	\end{algorithm}

	\section{Assumptions}\label{sec_assumption}
	
	We state the main assumptions that will be used in our analysis. Notice from model \eqref{model_X} that the loading matrix $\bLambda$ cannot be identified even in the noiseless case $E = 0$. This can be seen by inserting any $r\times r$ invertible matrix $\bH$ such that 
	$\wt \bLambda = \bLambda \bH$ and $\wt Z = \bH^{-1} Z$ form the same model. We resort to the following assumptions on $\bLambda$ and $Z$.

	\begin{assumption}\label{ass_Z}
		Assume that $Z\in \R^r$ from model \eqref{model_X} has independent entries with zero mean, the second moment $\EE[Z_j^2] = \sigma^2>0$ and the excess kurtosis
		\[
		\kappa :=  {1\over 3}\left({\EE[Z_j^4]\over \sigma^4}- 3\right) > 0
		\]
		for all $j\in [r]$. 
		Further assume 
		each $Z_j/\sigma$ is sub-Gaussian\footnote{A random variable $W$ is said to be sub-Gaussian with sub-Gaussian constant $\gamma$ if for any $t\ge 0$, $\EE[\exp(tW)] \le \exp(t^2\gamma^2/2)$.} with some sub-Gaussian constant $\gamma_z<\infty$. 
	\end{assumption}
	
	There are two key ingredients in \cref{ass_Z}: (1) entries of $Z$ are independent, and (2) each entry $Z_j$ has strictly positive excess kurtosis, that is, has {\em leptokurtic} distribution. Similar assumption has been adopted in \cite{rohe2020vintage} to study the statistical inference of varimax rotation. As commented therein, a sufficient condition for $Z_j$ being leptokurtic is that $\P\{Z_j = 0\} > 5/6$, namely, realizations of $Z_j$ could be exactly sparse (see, also Appendix G.2 of \cite{rohe2020vintage}, for discussion when entries of $Z$ are approximately sparse). One such instance that is widely used in Dictionary Learning for modelling the sparsity of $Z_j$ is the Bernoulli-Gaussian distribution  \citep{spielman2013exact,sun2016complete,qu2020geometric,zhai2020complete, zhang2018structured,xue2021efficient}.  \cref{ass_Z} is also commonly used in the literature of Independent Component Analysis \citep{hyvarinen2000independent} since the excess kurtosis can be used to distinguish Gaussian from non-Gaussian distribution \citep{FioriZenga}.  Meanwhile we remark that $Z_j$ is not required to be identically distributed and the same second moment of $Z_j$ can be assumed without loss of generality. Assuming the same excess kurtosis across $Z_j$ is also made for simplicity, and under which the loading matrix $\bLambda$ can only be identified up to a signed permutation (see, \cref{thm_A_general}). In fact, if components of $Z$ have distinct kurtoses, one can achieve stronger identifiability of $\bLambda$ by identifying the permutation as well. Finally, the sub-Gaussian tail of $Z_j$ is not essential and can be replaced by sub-exponential tail or some moment conditions. 

	\begin{assumption}\label{ass_A_general}
		Assume the $p\times r$ loading matrix $\bLambda$ has rank $r$ with $\sigma_1(\bLambda) = 1$ and $\sigma_r(\bLambda) \ge c_\Lambda$ for some absolute constant $c_\Lambda\in (0,1]$.
	\end{assumption}
	
	Assuming $\sigma_1(\bLambda) = 1$ fixes the scale invariance between $\bLambda$ and $Z$, and can be made without loss of generality. The lower bound for $\sigma_r(\bLambda)$ is made for more transparent presentation and one can track its dependence throughout our analysis.

	Since our model \eqref{model_X} also contains the additive noise $E$, we need a condition on its distribution to separate it away from the signal $\bLambda Z$.

	\begin{assumption}\label{ass_E_general} 
		With $\sigma^2$ defined in \cref{ass_Z}, assume  $E \sim \cN_p(0,  \sigma^2 \epsilon^2 \bSigma_E)$ for some $\epsilon^2 \ge 0$ and a semi-positive definite $p\times p$ matrix $\bSigma_E$  
		with $\|\bSigma_E\|_\op = 1$.
	\end{assumption} 
	The Gaussianity in \cref{ass_E_general} is only made for theoretical convenience and can be relaxed to any centered, sub-Gaussian distribution in most of our results, see  also \cref{rem_ext}.

	We allow correlated entries in $E$ by considering a general covariance matrix $\bSigma_E$. When $\bSigma_E = \bI_p$, the principal component analysis (PCA) yields the maximum likelihood estimator of $\bLambda\bZ$ (treating $\bZ$ as deterministic), see, for instance,  \cite{lawley1962factor, BaiLi2012}. However, even in this case, it is well-known that PCA is consistent only if the noise $\bE = (E_1, \ldots, E_n) \in \R^{p\times n}$ is dominated by the signal $\bLambda \bZ$ in the sense that 
	\[
	{\sigma_1(\bE\bE^\T) \over \sigma_r(\bLambda \bZ\bZ^\T \bLambda^\T)}  \to 0,\quad \text{in probability as }n\to \infty.
	\]
	In our setting at the population level, this gets translated to
	\begin{equation}\label{cond_snr}
		{\epsilon^2 \|\bSigma_E\|_\op   \over  \sigma_r^2(\bLambda)} \le  {\epsilon^2  \over c_\Lambda^2} \to 0, \quad \text{as }n\to \infty.
	\end{equation}
	Our theory in \cref{sec_theory} shows that \eqref{cond_snr} ensures the validity of using PCA for estimating $\bLambda$ not only for $\bSigma_E = \bI_p$ but also for a general $\bSigma_E$.    In \cref{sec_ext} we offer further improvement over using PCA in the first step when $\bSigma_E$ can be consistently estimated. In such case, the requirement \eqref{cond_snr} can be relaxed.

	\section{Theoretic guarantees}\label{sec_theory}

	We first state in \cref{sec_PCs} a decomposition of the principal components,  which is informative for analyzing the subsequent deflation varimax rotation. In \cref{sec_theory_alg} we provide analysis of the proposed deflation varimax rotation (\cref{alg_rotate}). We state theoretical guarantees on the proposed estimator of $\bLambda$ in \cref{sec_rotation_loading}.

	\subsection{Theoretical guarantees on the principal components}\label{sec_PCs}
	In this section we analyze the principal components $\bU_{(r)}$  given by \eqref{def_PCs}. For future reference, write the singular value decomposition (SVD) of the loading matrix $\bLambda$ as 
	\begin{equation}\label{svd_A}
		\bLambda = \bL \bS \bA
	\end{equation}
	with $\bL\in\bbO_{p\times r}$, $\bA\in \bbO_{r\times r}$ and $\bS = \diag(\sigma_1(\bLambda), \ldots, \sigma_r(\bLambda)) \in \R^{r\times r}$ containing the singular values in non-increasing order.    Define the deterministic sequence
	\begin{equation}\label{def_omega_n}
		\omega_n := \sqrt{r \log(n) \over  n}+\sqrt{\epsilon^2 p \log (n) \over  n}.
	\end{equation}
	The following theorem states a decomposition of $\bU_{(r)}$.
	\begin{theorem}\label{thm_PCs}
		Under Assumptions \ref{ass_Z}, \ref{ass_A_general} and \ref{ass_E_general}, assume there exists some sufficiently small constant $c>0$ such that $\epsilon^2 \le c$, $r\log (n) \le c n$ and $\epsilon^2 p \log(n) \le c n$. Then there exists some  $\bR\in\bbO_{r\times r}$ such that,  with probability at least $1-n^{-1}$, one has 
		\begin{equation}\label{decom_Ur}
			\bU_{(r)} = 
			\bR^\T  \left(\bA  \bZ / \sigma + \bN + \bOmega\right)
		\end{equation}
		where  $\bN  =  \bS^{-1}\bL^\T \bE / \sigma
		$ and $\bOmega =  \bDelta( \bZ / \sigma + \bN)$  with
		$
		\|\bDelta\|_\op    \lesssim    \omega_n +\epsilon^2.
		$
	\end{theorem}
	\begin{proof}
		Its proof appears in \cref{app_proof_thm_PCs}.
	\end{proof}

	\cref{thm_PCs} ensures that 
	$
	\bU_{(r)}$ approximates $\bR^\T \bA \bZ/\sigma$,
    up to an additive error $\bN$ and an approximation error $\bOmega$.
	Condition  
	$r\log(n) \ll n$ requires that the number of latent factors is much smaller than the sample size. Condition $\epsilon^2p\log(n) \ll n$ holds in the low-dimensional case when $p \ll n$. In the high-dimensional case, it is commonly assumed in the existing literature of factor analysis \citep{SW2012,BaiLi2012,lam2012,fan2013large,bing2020prediction} that $\epsilon^2  =\cO(1/p)$ whence $\epsilon^2 p \log(n)/n =\cO(\log(n)/ n)$. For instance,  if a fixed proportion of rows in $\bLambda$ are i.i.d. from a $r$-dimensional distribution with its covariance matrix having bounded eigenvalues, then the largest $r$ eigenvalues of $\bLambda\bLambda^\T$ are all of order $p$. In this case, after transferring this scale to $\bZ$ under $\sigma_1(\bLambda) = 1$ and provided that $\bSigma_E$ has bounded operator norm, we have $\epsilon^2 =\cO(1/p)$. 
	Regardless, the condition on $\epsilon^2$ we require in \cref{thm_PCs} is much weaker. 
	
	In view of the decomposition in \eqref{decom_Ur}, both $\bN$ and $\bOmega$ affect the error of using $\bU_{(r)}$ to approximate  the span of $\bZ$.  Specifically, the term $\bN$ is the intrinsic bias of PCA (see, also, the discussion before display \eqref{cond_snr}), and is still additive due to the independence between $\bZ$ and $\bE$. On the other hand, the term $\bOmega$ contains two sources of estimation error: the error of the eigenvectors $\bV_{(r)}$ for estimating $\bL$ and the error of the eigenvalues $\bD_{(r)}$ for estimating $\sigma^2 \bS^2$. In cases when these two errors are expected to be large, one should seek more appropriate estimators of $\bL$ and $\sigma^2 \bS^2$. For instance, in \cref{sec_ext} we consider an improved estimator of $\sigma^2 \bS^2$ when $\bSigma_E$ can be  consistently estimated. If sparsity of the column space of $\bLambda$ is a reasonable assumption, one could use the sparse PCA to estimate  $\bL$ instead. 
	Nevertheless, our analysis of \cref{thm_PCs} remains  valid once  the same type of error bounds of new estimators can be derived.

	To investigate which quantity the subsequent deflation varimax rotation
 	estimates, one needs to analyze the optimization problem in \eqref{obj_L4}.  The decomposition  in \eqref{decom_Ur} provides a key starting point for this. In the next section, we  show that the deflation varimax in fact estimates the matrix
	$\bR^\T\bA$  in \eqref{decom_Ur} up to a signed permutation matrix.
	

	\subsection{Theoretical guarantees of the deflation varimax in \cref{alg_rotate}}\label{sec_theory_alg}

	To provide complete analysis of the  deflation varimax rotation $\wc\bQ$ obtained in \eqref{def_wc_Q}, we start by analyzing the solution to \eqref{obj_L4}. To this end, we introduce the oracle counterpart of \eqref{obj_L4}
	\begin{equation}\label{obj_L4_tilde_main}
		\min_{q\in\Sp^r}  F\left(q;  \bR \bU_{(r)}\right).
	\end{equation}
	We call \eqref{obj_L4_tilde_main} oracle as $\bR$ is not known in advance. Nevertheless, \cref{lem_connection} below provides the connection between \eqref{obj_L4} and \eqref{obj_L4_tilde_main} in terms of their PGD iterates in \eqref{iter_PGD}. 
	
	\begin{lemma}\label{lem_connection}
		Let $\{\wh\bq^{(\ell)}\}_{\ell\ge 0}$ and $\{\wt\bq^{(\ell)}\}_{\ell\ge 0}$  be any iterates obtained via the projected gradient descent  in \eqref{iter_PGD} for solving \eqref{obj_L4} and \eqref{obj_L4_tilde_main}, respectively. Then whenever
		$
		\wh\bq^{(0)} =  \bR^\T \wt\bq^{(0)}
		$
		holds, one has 
		\[
		\wh\bq^{(\ell)} =  \bR^\T \wt\bq^{(\ell)},\qquad \text{for all }\ell \ge 1.
		\]
	\end{lemma}
	\begin{proof}
		Its proof appears in \cref{app_proof_lem_connection}.
	\end{proof}
	
	In view of \cref{lem_connection}, we only need to analyze the matrix $\wt\bQ$ obtained from solving  \eqref{obj_L4_tilde_main} $r$ times. This brings great convenience to our analysis as  
	\begin{equation}\label{decom_RUr}
		\bR \bU_{(r)}   = \bA\bZ / \sigma+ \bN + \bOmega 
	\end{equation}
	has an additive decomposition in its right hand side, as opposed to $\bU_{(r)}$ in \eqref{decom_Ur}. 

	We proceed to analyze the solution obtained from solving \eqref{obj_L4_tilde_main} once (for $k=1$). 
	Specifically, we provide in the following theorem an optimization landscape analysis of its stationary points.  Such results are generic and independent of any particular optimization algorithm one might use to solve \eqref{obj_L4_tilde_main}. 
	For simplicity and also without loss of generality, we assume $\sigma^2=1$ from now on. Denote by $\bSigma_N$  the covariance matrix of columns in $\bN$. For future reference, we note that under Assumptions \ref{ass_A_general} \& \ref{ass_E_general},
	\begin{equation}\label{eq_Sigma_N}
		\bSigma_N = \epsilon^2 \bS^{-1}\bL^\T \bSigma_E \bL \bS^{-1}\quad \text{and} \quad \|\bSigma_N\|_\op \le \epsilon^2/c_\Lambda^2.
	\end{equation}
	Also define $\sigma_N^2 := \sup_{q \in \Sp^r} \|\bP_{q}^{\perp} \bSigma_N q\|_2$ which quantifies the extent to which $\bSigma_N$ departs from a matrix that is proportional to $\bI_r$. Recall from  \cref{thm_PCs} that $\bOmega = \bDelta (\bZ + \bN)$.

	\begin{theorem}\label{thm_critical}
		Under Assumptions \ref{ass_Z}, \ref{ass_A_general} and \ref{ass_E_general},  assume $\kappa\ge c$ for some constant $c>0$ and 
		\begin{equation}\label{cond_n_r}
			\|\bDelta\|_\op +     \sigma_N^2 + \sqrt{r^2\log(n) \over n} ~ \le  ~ c'  
		\end{equation} 
		for some sufficiently small constant $c'>0$.
		Let $\bq$ be  any stationary point to \eqref{obj_L4_tilde_main} such that 
		\begin{equation}\label{cond_critical} 
			( A_i^\T\bq)^2 
			  \ge  1/3, ~ \quad   | A_i^\T\bq| - \max_{j\ne i}| A_j^\T \bq|   \ge C' \left(\|\bDelta\|_\op +    \sigma_N^2 + \sqrt{r^2\log(n) \over n} \right)
		\end{equation}
		holds for some arbitrary $i\in [r]$ and some positive constant $C'$. Then with probability at least $1-n^{-1}$, one has
		\begin{equation}\label{rate_critical}
			 \min\Bigl\{\|\bq -  A_i\|_2,~ \|\bq +  A_i\|_2\Bigr\}   ~  \lesssim  ~  \|\bDelta\|_\op   + \sigma_N^2 + \sqrt{r\log(n)\over n}.
		\end{equation}
	\end{theorem}
	\begin{proof}
		Its proof appears in \cref{app_proof_thm_critical}.
	\end{proof}

	
	Surprisingly, despite the non-convexity of \eqref{obj_L4_tilde_main}, \cref{thm_critical} shows that any stationary point $\bq$ satisfying  \eqref{cond_critical} consistently estimates one column of $\bA$ up to its sign. In particular, {\em any} first-order descent algorithm, including the projected gradient descent (PGD) in \eqref{iter_PGD},  can be used to estimate $\bA$ as long as the resulting stationary point satisfies \eqref{cond_critical}.  
	The first condition in \eqref{cond_critical} requires  $\bq$ to align well with one column of $\bA$  while the second one puts a lower bound on the gap between the largest two entries (in absolute value) of $\bA^\T \bq$. It is worth explaining the latter. Our proof of \cref{thm_critical} reveals that, in addition to the stationary points that are close to columns of $\bA$, there exist several ``bad'' stationary points, say $\bar\bq\in\Sp^r$, which  at the population level satisfy 
	\begin{equation}\label{bad_critical}
		\bA^\T \bar\bq = {1\over \sqrt k}\begin{bmatrix} 
			 1_{k} \\   0_{r-k}
		\end{bmatrix},\qquad\text{for any $2\le k\le r$,}
	\end{equation}
	up to sign and permutation. The second condition in \eqref{cond_critical}  excludes such bad stationary points.

		A key step in our proof is to establish concentration inequalities between the objective function $F(\bq; \bR \bU_{(r)})$, its Riemannian gradient $\grad F(\bq; \bR  \bU_{(r)})$ and their population-level counterparts (see, 
		Lemmas \ref{lem_bd_grad_F_tilde}, \ref{lem_bd_grad_pointwise},  \ref{lem_HH} \& \ref{lem_grad_lip}). In addition to $\|\bDelta\|_\op$ and $\sigma_N^2$, both terms $r\log(n)/n$ and $r^2\log(n)/n$ in displays \eqref{cond_n_r} -- \eqref{rate_critical}  originate from these concentration inequalities, depending on whether the result needs to hold  uniformly over $\bq\in \Sp^r$ (see, displays \eqref{def_event_1}, \eqref{def_event_2} and \eqref{def_event_grad_a}, for explicit expressions).
		Condition  $\kappa > c$ is made for simplicity, and our analysis can be easily extended to allow  $\kappa = \kappa(n) \to 0$ as $n\to \infty$.  
		Finally, the constant $1/3$ in \eqref{cond_critical} is chosen for simplicity and we did not attempt to optimize such explicit constant.

	Recall the decomposition of $\bR\bU_{(r)}$ in \eqref{decom_Ur}.
	Our results in \cref{thm_critical} characterize both the effect of the approximation error $\bOmega$ and that of the additive error $\bN$ on estimating $\bA$ via  solving \eqref{obj_L4_tilde_main}. Specifically, the approximation error in $\bOmega$ affects the estimation error in terms of $\|\bDelta\|_\op$ which according to \cref{thm_PCs} is bounded from above by $\cO_\P(\omega_n+ \epsilon^2)$. On the other hand, the additive error  in $\bN$  affects the estimation error via the covariance matrix of its columns. When $\bSigma_N$ is proportional to the identity matrix, one has $\sigma_N^2 = 0$, leading to  a weaker condition  in \eqref{cond_critical} as well as a faster rate in \eqref{rate_critical}. For a general $\bSigma_N$ given by \eqref{eq_Sigma_N}, one always has $\sigma_N^2 \le \epsilon^2 / c_\Lambda^2$.  We opt to state \cref{thm_critical} in its current form to reveal both effects of $\bN$ and $\bOmega$ as they are informative for understanding the improvements discussed later in \cref{sec_ext}.  For now,  we substitute $(\omega_n + \epsilon^2)$ for $(\|\bDelta\|_\op + \sigma_N^2)$ to simplify the presentation.\\

	To make use of \cref{thm_critical} for the PGD iteration in \eqref{iter_PGD}, we need to verify \eqref{cond_critical} for the corresponding stationary point. Due to non-convexity of \eqref{obj_L4_tilde_main}, this naturally depends on the initialization in \eqref{iter_PGD}. The following theorem  provides  sufficient conditions under which   {\em any generic} initialization provably ensures condition \eqref{cond_critical}, hence recovers one column of $\bA$.   To formally state the result, let $\{\bq^{(\ell)}\}_{\ell\ge 0}$ be a generic sequence  of the PGD iterates in \eqref{iter_PGD}  from solving \eqref{obj_L4_tilde_main}. We introduce the following event on the initial point $q^{(0)}$: for any $i\in [r]$ and any numbers $0< \nu < \mu < 1$, 
	\begin{equation}\label{def_event_init}
		\cE^{(i)} _{\init}(\mu,\nu)= \left\{
				| A_i^\T \bq^{(0)}| \ge \mu,~   | A_i^\T \bq^{(0)}| - \max_{j\ne i}| A_j^\T \bq^{(0)}| \ge  \nu
		\right\}.
	\end{equation}  

	\begin{theorem}\label{thm_one_col}
		Under Assumptions \ref{ass_Z}, \ref{ass_A_general} and \ref{ass_E_general},  assume $\kappa\ge c$ for some constant $c>0$. For any $i\in [r]$,  assume $q^{(0)}$ satisfies $\cE^{(i)} _{\init}(\mu,\nu)$ with some  $0< \nu < \mu < 1$ such that 
		\begin{equation}\label{cond_omega_n}
			      \sqrt{r^2\log(n)\over n}  + \sqrt{\epsilon^2p\log(n)\over n}+  \epsilon^2  ~ \le  c' \mu \nu^2
		\end{equation}
		holds for some sufficiently small constant $c'>0$. Then by choosing the step size $\gamma \le c''/\kappa$ for some small constant $c''>0$, 
		with probability at least  $1-  2/n$,   any stationary point  $q$ of the PGD iterates $\{\bq^{(\ell)}\}_{\ell\ge 0}$ to  \eqref{obj_L4_tilde_main} satisfies  
		\begin{equation}\label{rate_one_col}
			 \min\Bigl\{\|\bq -  A_i\|_2,~ \|\bq +  A_i\|_2\Bigr\}   ~ \lesssim  ~   \omega_n + \epsilon^2.
		\end{equation}
		Moreover, with the same probability, \eqref{rate_one_col} is guaranteed to hold for any  $q^{(\ell)}$ as long as $\ell \ge  C \log(n/\max\{ r, \epsilon^2p\})$.
	\end{theorem}

	The proof  of \cref{thm_one_col} appears in \cref{app_proof_thm_one_col}.  The key to prove  \cref{thm_one_col} is to verify condition \eqref{cond_critical} for any obtained stationary point corresponding to the initialization $q^{(0)}$. To this end,  in Lemmas \ref{lem_iter_gap} \& \ref{lem_iter_sup} of \cref{app_sec_PGD}, we prove that the event $	\cE^{(i)} _{\init}(\mu,\nu)$, condition \eqref{cond_omega_n} and  $\gamma \le c''/\kappa$ ensure that, for any $\ell\ge 0$, 
	\begin{align}\label{lb_gap}
		&| A_i^\T \bq^{(\ell+1)}| - \max_{j\ne i}| A_j^\T \bq^{(\ell+1)}| \ge   \min\bigl\{\nu, ~ 1/4\bigr\}
	\end{align} 
	as well as 
	\begin{equation}\label{lb_sup}
		\begin{split}
				& | A_i^\T \bq^{(\ell+1)}| \ge	| A_i^\T \bq^{(\ell)}|  \left(1 + \gamma ~  \mu \nu \right),\hspace{3cm}  \text{ if } \mu \le | A_i^\T \bq^{(\ell)}|  \le \sqrt{3}/2;\\
				& | A_i^\T \bq^{(\ell+\ell')}|	\ge 1 -  (\omega_n+\epsilon^2)^2- \left(1 -  \gamma \right)^{\ell'},~    \forall~ \ell'\ge 1,\quad \text{ if } \sqrt{3}/2<| A_i^\T \bq^{(\ell)}| \le 1.
		\end{split} 
	\end{equation}
	Consequently, displays \eqref{lb_gap}  and  \eqref{lb_sup} together with condition \eqref{cond_omega_n}  ensure that any stationary point $q$  satisfies \eqref{cond_critical}, hence consistently estimates $A_i$ up to its sign. Moreover, by using the identity $\|A_i -v \|_2^2 = 2(1-A_i^\T v)$ for any $v\in \Sp^r$, it is easy to see that the contraction property in \eqref{lb_sup} guarantees \eqref{rate_one_col} after $\cO(\log n)$ iterations.

	\begin{remark}[Choice of the step size]\label{rem_stepsize}
			To ensure both the local contraction in \eqref{lb_sup} and algorithmic convergence, 
			 \cref{thm_one_col} requires the step size $\gamma$ can not be chosen too large. On the other hand,  although 
			 a too small choice of $\gamma$ will not affect the statistical accuracy ultimately, it  slows down the speed of algorithmic convergence, as seen from  \eqref{lb_sup}. In our extensive simulation, we found that $\gamma = 10^{-5}$ yields overall satisfactory results.  
	\end{remark}

	\begin{remark}[Effect of the initialization]\label{rem_effect_init}
		Thanks to \cref{thm_critical}, we are able to provide sufficient conditions in terms of $\mu$ and $\nu$  under which {\em any generic} initialization provably recovers one column of $\bA$. Furthermore, condition \eqref{cond_omega_n} in \cref{thm_one_col}  reveals that a better initialization, in the sense of larger $\mu$ and $\nu$, leads to a weaker requirement on both the sample size $n$ and the noise level $\epsilon^2$. 
		In \cref{sec_init}, we discuss two initialization schemes  and  establish their corresponding orders of $\mu$ and $\nu$ for which $\cE^{(i)}_{\init}(\mu,\nu)$  holds.
	\end{remark}

	Following \cref{thm_one_col}, we immediately have theoretical guarantees for the whole matrix $\wt \bQ = (\wt Q_1,\ldots,\wt Q_r)$ obtained from solving \eqref{obj_L4_tilde_main} $r$ times via PGD. A key observation for analyzing $\wt \bQ$ is that  the solutions $\wt Q_k$, for each $k\in [r]$, are obtained via the same PGD iterates but different initializations. Therefore, results in \eqref{lb_gap} and \eqref{lb_sup}
	as well as \cref{thm_one_col} are still applicable to analyze each $\wt Q_k$. This is summarized in the following corollary. 

	\begin{corollary}\label{cor_A}
		Grant the conditions in \cref{thm_one_col}. 
		Assume there exists some permutation $\pi: [r] \to [r]$ such that  the initialization $\wt  Q_{\pi(k)}^{(0)}$ satisfies $\cE^{(k)}_{\init}(\mu,\nu)$ for all $k\in [r]$.  Then there exists some signed permutation matrix $\bP$ such that with probability at least $1-2/n$, 
		\begin{align*}
			  \|\wt \bQ - \bA \bP \|_{F}~  \lesssim ~   (\omega_n + \epsilon^2)\sqrt{r} . 
		\end{align*}
	\end{corollary}  
	
	Note that $\wt Q_1^{(0)}, \ldots, \wt Q_r^{(0)}$  in \cref{cor_A} correspond to the initial points of the oracle problem in \eqref{obj_L4_tilde_main} whereas in practice  we only have access to initializations 
	to \eqref{obj_L4}. Nevertheless,   \cref{lem_connection} ensures that we only need to establish the order of $\mu$ and $\nu$ such that $\cE^{(k)}_{\init}(\mu, \nu)$ holds for  $\wt Q_k^{(0)} := \bR \wh Q_k^{(0)}$ for each $k\in [r]$, with   $\wh Q_1^{(0)}, \ldots, \wh Q_r^{(0)}$ being any initialization of the PGD iteration to solve \eqref{obj_L4}. The following theorem states the theoretical guarantees of the corresponding solution $\wc\bQ$ obtained from applying \cref{alg_rotate} to \eqref{obj_L4}. 
	
	\begin{theorem}\label{thm_RA}
		Grant the conditions in \cref{thm_one_col}. 
		Assume there exists some permutation $\pi: [r] \to [r]$ such that $\bR \wh  Q_{\pi(k)}^{(0)}$ satisfies $\cE^{(k)}_{\init}(\mu,\nu)$ for all $k\in [r]$.  Then there exists some signed permutation matrix $\bP$ such that with probability at least $1-2/n$, 
		\begin{align*}
			\|\wc \bQ - \bR^\T\bA \bP \|_{F}~  \lesssim ~    (\omega_n + \epsilon^2)\sqrt{r} .
		\end{align*}
	\end{theorem}
	\begin{proof}
		Its proof appears in \cref{app_proof_thm_RA}.
	\end{proof}

	\begin{remark}[Rotated principal components from the deflation varimax]\label{rem_Z}
		On the one hand,  \cref{thm_RA} shows that the proposed deflation varimax rotation makes valid statistical inference, namely, it consistently estimates $\bR^\T \bA$ up to a signed permutation matrix. Similar results are proved in \cite{rohe2020vintage}  for the global solution to the varimax criterion in \eqref{obj_varimax} under a different model.  On the other hand, by recalling the decomposition in \eqref{decom_Ur} of \cref{thm_PCs}, we see that the rotated PCs of using the deflation varimax satisfy
		\[
		\wc \bQ^\T \bU_{(r)} \approx   \wc \bQ^\T 
		\bR^\T (\bA  \bZ  + \bN) \approx \bP^\T  (\bZ +  \bA^\T \bN).
		\]
		In words, the rotated PCs predict  the factor matrix $\bZ$, up to a scaling ($\sigma = 1$) and signed permutation of the rows as well as an additive error matrix $\bA^\T \bN$. Furthermore, by the structure of $\bSigma_N$ in \eqref{eq_Sigma_N},  columns of $\bA^\T   \bN$  are i.i.d.  Gaussian under \cref{ass_E_general} with zero mean and the covariance matrix bounded by $\epsilon^2/c_\Lambda^2$ in the operator norm.
	\end{remark}

		\begin{remark}[Analysis of the deflation varimax in a broader context]\label{rem_analysis}
			Although theoretical guarantees of the proposed deflation varimax  so far  are stated for the PCs under the decomposition  in \eqref{decom_Ur}, our analysis in fact is carried out under a more general setting. Concretely, suppose one has access to  an $r\times n$ matrix $\wh\bY$ that satisfies 
			\begin{equation}\label{model_Y_hat_main}
				\wh\bY = \bR^\T \left(\bA\bZ +  \bN  + \bOmega\right)
			\end{equation}
			where $\bA\in \bbO_{r\times r}$ is a deterministic  matrix while $\bZ$ and $\bN$ have  i.i.d. columns satisfying Assumptions \ref{ass_Z} \& \ref{ass_E_general}, respectively, with $\sigma^2=1$ and some general covariance matrix $\bSigma_N$.
			In \eqref{model_Y_hat_main} both the rotation matrix
			$\bR\in \bbO_{r\times r}$ and the approximation matrix $\bOmega \in \R^{r\times n}$ could be random, and are further allowed to correlate with each other, as well as with $\bZ$ and $\bN$. In \cref{app_sec_theory_varimax} we give detailed statements of our theoretical guarantees of the proposed deflation varimax under this general model. Clearly the decomposition of $\bU_{(r)}$ in \eqref{decom_Ur} is one instance of model \eqref{model_Y_hat_main}. Below we discuss another two important examples that are subcases of model \eqref{model_Y_hat_main}, and  to which both our procedure and our theory are applicable. 
			
			\begin{enumerate}
				\item[(a)] {\em Independent Component Analysis (ICA).} ICA	 often assumes the model $\bY = \bL \bZ$ for a deterministic, non-singular matrix $\bL \in \R^{r\times r}$ and a random matrix $\bZ \in \R^{r\times n}$ with independent, non-Gaussian entries that have zero mean and equal variance   $\sigma^2>0$. A common first step of various ICA approaches is to whiten the data by using the square root matrix of the sample covariance matrix $\wh \bSigma_Y = n^{-1}\bY\bY^\T$ \citep{hyvarinen2000independent}. Consequently, by writing $\bSigma_Y = \EE[\wh \bSigma_Y]$ and the SVD of $\bL = \bV_L \bD_L \bU_L^\T$, we have 
				\[
				\wh \bSigma_Y^{-1/2} \bY = \bV_L\bU_L^\T  \bZ / \sigma + \bigl(\wh \bSigma_Y^{-1/2} - \bSigma_Y^{-1/2}\bigr)\bY.
				\]
				This is a sub-model of \eqref{model_Y_hat_main} with $\bR = \bI_r$, $\bA = \bV_L\bU_L^\T$, $\bN = \0$ and $\bOmega = (\wh \bSigma_Y^{-1/2} - \bSigma_Y^{-1/2})\bY$. When entries of $\bZ$ have leptokurtic distributions,   our proposed deflation varimax can be applied to estimate $\bV_L\bU_L^\T$ and to further estimate the matrix $\bL$. Moreover, our theory developed in \cref{sec_theory_alg} provides provable guarantees for the resulting estimator. 
				
				We now compare with the existing provable ICA algorithms  from a technical perspective. Under model $\bY = \bL \bZ$ with an orthonormal $\bL \in \bbO_{r\times r}$, 	\cite{anandkumar2014tensor} proposes an algorithm for estimating $\bL$ via orthogonal tensor decomposition and power iterations in a deflation manner. This estimator, together with the subsequent analysis in \cite{anandkumar2015learning},  has the estimation error under the Frobenius norm of order $\cO_\P(\sqrt{r^2/n} + {r^{5/2}/n})$, up to some multiplicative polynomial factor of $\log(n)$, provided that $n\gg r^{5/2}$. In \cite{auddy2023}, an improved estimator which is obtained via fixed-point iterations and also in a deflation manner  is  shown to achieve the minimax optimal rate $\cO_\P(\sqrt{r^2/n})$ under the Frobenius norm. The required sample complexity is also reduced to $n \gg r^2$. To estimate a general non-singular matrix $\bL$,   \cite{anandkumar2014guaranteed} shows that a refined yet complicated procedure based on  tensor decomposition and clustering provably estimates $\bL$ in the same rate as $\cO_\P(\sqrt{r^2/n} + {r^{5/2}/n})$ under the Frobenius norm but with a slower column-wise $\ell_2$-norm. Differently,  \cite{auddy2023} adopts data splitting for whitening  and proves that the same optimal rate $\sqrt{r^2/n}$  can be achieved under the same sample complexity.  
				
				In our case of model \eqref{model_Y_hat_main}, we have to deal with both the additive error $\bN$ and the approximation error $\bOmega$ as well as the random rotation $\bR$. Consequently both the estimation error and the required sample complexity consist of other terms in addition to $\sqrt{r^2/n}$. On the other hand, our analysis in \cref{thm_one_col} is able to characterize the  effect of {\em any generic} initialization whereas both the analysis in \cite{anandkumar2014guaranteed,anandkumar2015learning} and \cite{auddy2023} is only valid for a sufficiently good initialization ($\mu$ in \eqref{def_event_init} needs to be close to $1$ whence $\nu$ is also of constant order).  Finally,  our method does not rely on either data splitting  or complicated tensor decomposition.  
				
				\item[(b)]  {\em Complete Dictionary Learning.}
				As mentioned in Section \ref{sec:intro}, another related problem is the complete dictionary learning where the input data $\wh\bY$ follows model \eqref{model_Y_hat_main} with $\bR = \bI_r$, $\bN = \bOmega = \0$ and entries of $\bZ$ are further assumed to be i.i.d. from a Bernoulli-Gaussian distribution. In this setting, a line of recent works \citep{spielman2013exact,sun2016complete,bai2018subgradient,qu2020geometric,zhai2020complete,zhai2020Understanding,jin2021unique} focus on designing provable algorithms to recover $\bA$. In particular, \cite{sun2016complete,bai2018subgradient,zhang2018structured,jin2021unique} also consider the optimization problem in  \eqref{obj_L4} and analyze its optimization landscape. 
				However, our results developed in this section are applicable in much broader settings.
			\end{enumerate} 
		\end{remark}

	\subsection{Theoretical guarantees on estimating the loading matrix}\label{sec_rotation_loading}

	Armed with the theoretical guarantees of the deflation varimax rotation in Theorems \ref{thm_RA_rand} and \ref{thm_RA_mrs}, we now establish finite sample upper bounds of the estimation error of $\wh \bLambda$ given by  \eqref{def_Lambda_hat} under the Frobenius norm.
	
	\begin{theorem}\label{thm_A_general}
		Grant  conditions in   \cref{thm_RA}  as well as its conclusion.  For the same signed permutation $\bP$ therein,  
		the estimator $\wh \bLambda$ in \eqref{def_Lambda_hat} satisfies 
		\begin{equation}\label{rate_Lambda}
			\|\wh \bLambda - \bLambda\mb P\|_{F}~  = ~  \cO_\P\left(    \omega_n  \sqrt{r}  + \epsilon^2  \sqrt{r}  \right).
		\end{equation} 
	\end{theorem} 
	\begin{proof}
		 Its proof appears in \cref{app_proof_thm_A_general}.
	\end{proof}
	
	 \cref{thm_A_general} also implies  that the loading matrix $\bLambda$ is identifiable up to an $r\times r$ signed permutation matrix. As mentioned after \cref{ass_Z}, if all the coordinates of $Z$ in model \eqref{model_X} have distinct kurtoses, such permutation ambiguity can be removed. 
	
	\begin{remark}[Effects of the ambient dimension and the signal-to-noise ratio]\label{rem_optimality}
		We discuss how the ambient dimension $p$ and the reciprocal of the signal-to-noise ratio (SNR) $\epsilon^2$  affect the estimation of $\bLambda$. After ignoring the $\log(n)$ term, display \eqref{rate_Lambda} in  high-dimensional case ($p \ge n$) reduces to 
		\begin{equation}\label{rate_large_p}
			\|\wh \bLambda - \bLambda\mb P\|_F ~ = \cO_\P\left(
			\sqrt{r^2 \over n} + \sqrt{\epsilon^2 p r \over n}
			\right).
		\end{equation}
		In the low-dimensional case with  $n  \epsilon^2 \le p < n$, we have the same rate as above. \cref{thm_lowerbounds} below shows that the rate in \eqref{rate_large_p} is in fact minimax optimal  under model \eqref{model_X}. On the other hand, if $p  < n\epsilon^2$, then the rate in \eqref{rate_Lambda} becomes 
		\begin{equation}\label{rate_small_p}
			\|\wh \bLambda - \bLambda\mb P\|_F ~ =  
			\cO_\P\left(
			\sqrt{r^2 \over n} +  \epsilon^2 \sqrt{r}
			\right).
		\end{equation}
		If additionally $r \le n\epsilon^4$ holds, the second term becomes dominant. Therefore, the estimator $\wh\bLambda$ is sub-optimal when 
		the SNR is small in the precise sense that 
		$\epsilon^2 \ge  \min\{p/n, \sqrt{r/n}\}$.
		In \cref{sec_ext} we show that this sub-optimality can be remedied if the noise $E$ is structured.
	\end{remark}

	To benchmark the rate obtained in \cref{thm_A_general}, we establish the minimax lower bounds of estimating $\bLambda$ in the following theorem. Consider the following parameter space of $\bLambda$:
	\begin{equation*}
		\cA := \left\{
		\bLambda \in \R^{p\times r}: \bLambda \text{ satisfies \cref{ass_A_general}}
		\right\}.
	\end{equation*}  
	Our minimax lower bounds are established for a fixed distribution of $Z$ under \cref{ass_Z} and $E\sim \cN_p(0, \epsilon^2\sigma^2\bI_p)$. Let $\P_{\bLambda,\epsilon}$ denote the distribution of $\bX$ parametrized by both $\bLambda\in \cA$ and $\epsilon^2$, the reciprocal of the SNR.
	
	\begin{theorem}\label{thm_lowerbounds}
		Under model \eqref{model_X} and Assumptions \ref{ass_Z} \& \ref{ass_E_general} with $\bSigma_E = \bI_p$, assume $r(p-1) \ge 50$. 
		Then there exists some absolute constants $c>0$ and $c'\in (0,1)$ such that 
		\[
		\inf_{\wh\bLambda}\sup_{\bLambda\in \cA, \epsilon^2\ge 0} \P_{\bLambda,\epsilon}\left\{
		\|\wh\bLambda - \bLambda\|_F \ge c\sqrt{r^2\over n}   +c   \sqrt{\epsilon^2 pr \over n}
		\right\} \ge c'. 
		\]
		Here the infimum is taken over all estimators.
	\end{theorem}
	\begin{proof}
		The proof appears in \cref{app_proof_thm_lowerbounds}.
	\end{proof}
	
	\cref{thm_lowerbounds} implies that the rate in \eqref{rate_large_p} is minimax optimal, and cannot be further improved.
	
	Our lower bounds in \cref{thm_lowerbounds} are new in the literature. In the special case $\epsilon^2=0$ whence $\bX = \bLambda \bZ$, \cite{auddy2023} proves that the minimax lower bounds of estimating $\bLambda$ are of order $\sqrt{r^2/n}$ under Assumptions \ref{ass_Z} \& \ref{ass_A_general}. For $\epsilon^2 > 0 $,  minimax lower bounds of estimating $\bLambda$ when $\bLambda$ has orthonormal columns are established in \cite{VuLei2013}. Although $\bLambda \in \bbO_{p\times r}$ satisfies \cref{ass_A_general}, the results in \cite{VuLei2013} are proved under a class of Gaussian distributions of $X$ under model  \eqref{model_X}, hence are not applicable to our setting under \cref{ass_Z}.  For non-Gaussian factor $Z$ and some loading matrix $\bLambda$ with columns having the unit  $\ell_2$-norm, \cite{Jung2016} establishes a local minimax lower bound of estimating $\bLambda$ in the context of ICA. Our proof makes adaptation of their results to our context.

		\section{Initialization}\label{sec_init}
	
	Since we adopt the same PGD in \eqref{iter_PGD} multiple times, it is crucial to have proper initializations. Our analysis in  \cref{thm_one_col} of \cref{sec_theory} shows that the initial point  needs to satisfy $\cE^{(k)}_{\init}(\mu, \nu)$ in \eqref{def_event_init}.
	We discuss below two possible initialization schemes that provably ensure such requirement.

	\subsection{Random Initialization}\label{sec_init_rand}
		The first initialization scheme is to randomly initialize $\wh Q_k^{(0)}$ in \eqref{iter_PGD}  via  standard normal entries. Specifically, for each $1\le k < r$, by writing $\wh  Q_1,\ldots,\wh  Q_k$ as the previously obtained solutions, we define $\wh \bP_0^\perp = \bI_r$ and 
		\begin{equation}\label{def_P_hat_k_comp}
			\wh \bP_{k}^\perp = \bI_r - \sum_{i=1}^{k} \wh  Q_i\wh Q_i^\T. 
		\end{equation}
		Let the matrix $\wh\bV_{(-k)}\in \bbO_{r\times (r-k)}$ contain the leading $(r-k)$ eigenvectors of $\wh \bP_{k}^\perp $.  Further let $g_{(-k)} \in \R^{r-k}$ contain i.i.d. entries of $\cN(0,1)$. For solving \eqref{obj_L4} the $(k+1)$th time, we propose to intialize \eqref{iter_PGD} by 
		\begin{align}\label{def_q_init_all_cols}
			\wh  Q_{k+1}^{(0)}=  
			P_{\Sp^r}\bigl ( \wh\bV_{(-k)}g_{(-k)}\bigr).
		\end{align}

		The following theorem establishes its corresponding order of $\mu$ and $\nu$ for which the event $\cE^{(k+1)}_{\init}(\mu, \nu)$ in \eqref{def_event_init} holds. 
		
		\begin{theorem}[Random Initialization]\label{thm_RA_rand} 
			Grant Assumptions \ref{ass_Z}, \ref{ass_A_general} and \ref{ass_E_general} as well as $\kappa\ge c$ for some constant $c>0$. Choose $\gamma \le c'/\kappa$ for some sufficiently small constant $c'>0$.
			Fix any $\delta \in (0, 1/r^2)$ and assume there exists some sufficiently small constant $c''>0$ such that 
			\begin{equation}\label{cond_omega_n_rand}
				r^{3/2} \left(
				\sqrt{r^2\log(n)\over n} + \sqrt{\epsilon^2 p \log (n) \over n}    +\epsilon^2
				\right) \le  c'' \delta^2
			\end{equation}
			holds. Then there exists a permutation $\pi:[r] \to [r]$ such that with probability at least $1-re^{-2r} - r^2\delta-n^{-1}$, $\bR \wh  Q_{\pi(k)}^{(0)}$ satisfies 
			$\cE^{(k)} _{\init}(\mu, \nu)$  for all $k\in [r]$ with  
			\begin{equation}\label{mu_nu_rand}
				\mu = {1 \over \sqrt{r}} ,\qquad \nu = {1\over 3\sqrt{e}}{\delta \over \sqrt r}.
			\end{equation}
			Consequently, conclusions in \cref{thm_RA} and \cref{thm_A_general} hold.
		\end{theorem}
		\begin{proof}
			The proof appears in \cref{app_proof_thm_RA_rand}. 
		\end{proof} 
		For the  random initialization in \eqref{def_q_init_all_cols}, the matrix  $\wh\bV_{(-k)}$  ensures that $\wh Q_{k+1}^{(0)}$ is not close to any of the previously recovered solutions. This is  crucial to establish the expression of $\mu$ in $\cE_{\init}^{(k)}$ for {\em all} $k\ge 2$.
		Regarding the expression of $\nu$, it originates from the anti-concentration of the standard normal  entries in $g_{(-k)}$.  The union bounds argument over $k'\in [r]\setminus\{k\}$  and $k\in [r]$ of  such anti-concentration  inequality leads to the term $r^2\delta$  in the tail probability of \cref{thm_one_col}.  
		If one adopts   the multiple random initialization scheme discussed in \cref{rem_mrand} below, it is possible to improve $r^2\delta$  to $\delta$ but with the price of employing an additional clustering step. 
		
		Condition  \eqref{cond_omega_n_rand} corresponds to \eqref{cond_omega_n} with $\mu$ and $\nu$ specified in \eqref{mu_nu_rand}. It requires the lower dimension $r$ to be much smaller than the sample size $n$, namely, $r^\alpha \ll n$ for some $\alpha\in [5,13]$ depending on the choice of $\delta$.  It is worth mentioning that \eqref{cond_omega_n_rand} is only needed for the worst case when the initialization gets close to satisfy \eqref{bad_critical}, and mainly for a few initial iterates due to the contraction property in \eqref{lb_sup}.  As mentioned in \cite{anandkumar2014guaranteed}, the random initialization scheme   works very well empirically for learning  ICA model via tensor decomposition and deflation power iterations,  but its analysis is still an open problem even under the ICA model. To the best of our knowledge, \cref{thm_RA_rand} provides the first type of theoretical guarantees for using random initialization scheme in our general context.
		
		\begin{remark}[Multiple random initializations]\label{rem_mrand}
			Instead of using one realization of $g_{(-k)}$, one could draw $g_{(-k)}$ multiple times and compute the corresponding initializations as in \eqref{def_q_init_all_cols}. This will lead to many candidate solutions at each round $k\in [r]$. After all $r$ rounds, a final clustering step  is needed to select $r$ among all obtained candidates as the final estimates. This idea has been proposed in the context of tensor decomposition \citep{anandkumar2014guaranteed}. Though such procedure could improve the tail probability in \cref{thm_RA_rand} from $r^2\delta$ to $\delta$, thereby relaxing \eqref{cond_omega_n_rand}, empirically we do not observe any evident benefit  comparing to using a single random initialization.
		\end{remark}
		
		\subsection{Method of Moments based Initialization}\label{sec_init_momi}
		
		We discuss in this subsection another   initialization scheme  that is based on a Method of Moments  estimator of $\bA$. The empirical moments we use are once again based on randomly sampled standard normal variables. To explain its main idea, we start by stating a population-level quantity from which one can identify one column of $\bA$. 
		
		Let 
		$\bG$ be a random $r\times r$ matrix containing i.i.d. standard normal variables. Consider the following $r\times r$ matrix 
		\[
		\overline\bM(\bG)  =  \sum_{i=1}^r A_i A_i^\T \left(\kappa A_i^\T \bG A_i\right).
		\]
		Due to the orthogonality of columns in $\bA$, we have two key observations. \begin{enumerate}
			\item [(1)] If the two largest  singular values of $\overline\bM(\bG)$  have non-zero gap, then its leading left singular vector   identifies one column of $\bA$ up to sign. 
			\item[(2)] The singular values of $\overline\bM(\bG)$ must take values from the set $\{\kappa |A_1^\T \bG A_1|, \ldots, \kappa |A_r^\T \bG A_r|\}$.  
		\end{enumerate}
		The matrix $\overline\bM(\bG)$ can thus be treated as a population level moment from which we can identify one column of $\bA$.   However, when $\overline\bM(\bG)$ gets replaced by its sample estimate, the latter may not  have  large enough singular value gap to ensure consistency of its leading left singular vector.  To address this issue,  \cite{anandkumar2014guaranteed,anandkumar2015learning}  in the context of learning overcomplete latent models via tensor methods and 
		\cite{auddy2023} under the ICA model, consider a procedure called multiple random slicings which is based on computing $\overline\bM(\bG)$ multiple times from independent copies of $\bG$.
		
		Concretely, for some prespecified integer $N\in \mathbb{N}$, let $\bG_{(1)},\ldots, \bG_{(N)}$ be independent copies of $\bG$. For simplicity, denote 
		\[
		W_{ij} := A_j^\T \bG_{(i)} A_j,\quad \text{for all }i\in [N], j\in [r].
		\]
		By the orthogonality of columns in $\bA$, one can verify   that $\{W_{ij}\}_{i\in [N], j\in [r]}$ are independent standard normal variables (see, the proof of \cref{lem_M_op}). Therefore, for any given $j\in [r]$, concentration inequalities of the absolute value of the maximum of $N$ independent standard normal variables (see \cref{lem_one_sided_Gaussian}) gives  
		that for all $\delta \ge 0$,
		\begin{equation*}
			\P\left\{
			\max_{i\in [N]} |W_{ij}| \ge  \sqrt{2\log(N)} - {\log\log(N) + c \over 2\sqrt{2\log(N)}} - \sqrt{2\log(1/\delta)}
			\right\} \ge 1-\delta.
		\end{equation*}
		On the other hand, let $\bar i := \bar i(j):= \argmax_{i\in [N]} |W_{ij}|$. Since  $\bar i$ is only defined through $\{W_{1j}, \ldots, W_{Nj}\}$, it is independent of $W_{i\ell}$ for all $i\in [N]$ and $\ell \ne j$. This implies that  $\{W_{\bar i \ell }\}_{\ell\in [r]\setminus \{j\}}$  are $(r-1)$ independent standard normal variables, and using \cref{lem_one_sided_Gaussian} again yields that for all $\delta\ge 0$, 
		\begin{equation*} 
			\P\left\{
			\max_{\ell \in [r]\setminus \{j\}}  |W_{\bar i \ell }| ~ \le \sqrt{2\log(2r)}+ \sqrt{2\log(1/\delta)}
			\right\} \ge 1-\delta.
		\end{equation*}
		Fix any $\delta \ge 0$. 
		If we choose $N \ge 4r^2 /\delta^2$, combing the previous two displays together with the observation (2) ensures that with probability at least $1-\delta/r$,
		\begin{equation*}
			\sigma_1\bigl(\overline\bM(\bG_{(\bar i)})\bigr) - \sigma_2\bigl(\overline\bM(\bG_{(\bar i)})\bigr) =  |W_{\bar i j}| - \max_{\ell \in [r]\setminus \{j\}} |W_{\bar i \ell } | ~  \ge ~ c'\sqrt{\log (N)}.
		\end{equation*}
		Since $j\in [r]$ in the definition of $\bar i$ is chosen arbitrarily, we can choose the index 
		\[
			i_* := \argmax_{i\in [N]} 	\left[ \sigma_1\left(\overline\bM(\bG_{(i)})\right) - \sigma_2\left(\overline\bM(\bG_{(i)})\right)\right]
		\] 
		such that its corresponding matrix $\overline\bM(\bG_{(i_*)})$ has a gap between its largest two singular values no smaller than  $c'\sqrt{\log(N)}$. Thus, by observation (1),  the left leading singular vector of $\overline\bM(\bG_{(i_*)})$ could be used to estimate one column of $\bA$ if the index $i_*$ were known. 
		
		However, in practice we need to find both a suitable estimator of the index $i_*$ and an estimator of $\overline\bM(\bG_{(i_*)})$, which turns out to be a quite challenging task.  In \cite{anandkumar2014guaranteed,anandkumar2015learning} the authors did not consider to estimate such $i_*$ but instead employed an additional clustering step in the end to screen and group by all solutions obtained from multiple random slicings (see, also, \cref{rem_mrand}). 
		In the following we describe a simple procedure that meets our goals, and does not require either computing more than $r$ candidate solutions or any additional clustering step. 
		
		For each $i\in [N]$, let the sample level counterpart of $\overline\bM(\bG_{(i)})$  be 
		\begin{equation}\label{def_M_hat_main}
			\wh \bM(\bG_{(i)}) = {1\over 3n}\sum_{t=1}^n U_t U_t^\T \left( U_t^\T \bG_{(i)} U_t \right) - \bG_{(i)} - \bG_{(i)}^\T,  
		\end{equation} 
		with its singular values $\wh \sigma_1^{(i)} \ge \wh \sigma_2^{(i)} \ge \cdots \ge  \wh  \sigma_{r}^{(i)}$ and the left leading singular vector $\wh u^{(i)}_1$. On the one hand, we  show in \cref{lem_M_op} that $\wh \bM(\bG_{(i)})$ consistently estimates $\overline\bM(\bG_{(i)})$ uniformly over $i\in [N]$.  On the other hand, we also show in the proof of \cref{thm_RA_mrs} that the sample analogue of $i_*$, defined as 
		\begin{equation*} 
			\wh i_* = \argmax_{i\in [N]}  \left( 
			\wh \sigma_{1}^{(i)}  -\wh  \sigma_{2}^{(i)} 
			\right),
		\end{equation*}
		ensures that $\sigma_1^{(\wh i_*)} - \sigma_2^{(\wh i_*)} \ge c''\sqrt{\log(N)}$ with high probability.  
		By the observation (1) above, we  thus propose to use the left leading singular vector of $\wh \bM(\bG_{(\wh i_*)})$ for initialization in \eqref{iter_PGD}, namely, 
		$
			\wh  Q_1^{(0)} = \wh u^{( \wh i_*)}_1.
		$
		\cref{thm_RA_mrs} below shows that in fact $\wh  Q_1^{(0)}$ is already a consistent estimator of one column of $\bA$, and could be further used as initialization to obtain faster rates. 
		
		For  solving \eqref{obj_L4} the $k$th round with each $2\le k\le r$, we can repeat the whole procedure above except for using  
		\begin{equation}\label{def_M_hat_k}
			\wh \bM_{(k)}(\bG_{(i)}) =   \wh\bP_{k-1}^\perp  \wh \bM(\bG_{(i)})  \wh\bP_{k-1}^\perp 
		\end{equation}
		in place of $\wh \bM(\bG_{(i)})$. Here the matrix $\wh \bP_{k-1}^\perp$ given by \eqref{def_P_hat_k_comp} prevents from estimating duplicate columns in $\bA$. More concretely, for each $i\in [N]$, let $ \wh \sigma^{(i,k)}_1 \ge  \wh \sigma^{(i,k)}_2 \ge \cdots \ge \wh \sigma^{(i,k)}_r$ and  $\wh u_{1}^{(i,k)}$ be the  singular values and  the leading left singular vector, respectively, of \eqref{def_M_hat_k}. With
		\begin{equation}\label{def_arg_index}
			\wh i_*  := 	\wh i_*(k):= \argmax_{i\in [N]}  \left( 
			\wh \sigma_{1}^{(i,k)}  -\wh  \sigma_{2}^{(i,k)} 
			\right),
		\end{equation}
		we propose to initialize \eqref{iter_PGD} by 
		\begin{equation}\label{def_init_msr}
				\wh  Q_k^{(0)} = \wh u^{( \wh i_*,k)}_1.
		\end{equation}
		
		In the following theorem  we show that the above method of moments based initialization scheme provably ensures  $\cE^{(k)}_{\init}$ for all $k\in [r]$, and  the established orders of $\mu$ and $\nu$   improve upon those for the random initialization.  Recall $\omega_n$ from \eqref{def_omega_n}.
		
		\begin{theorem}[Method of Moments based Initialization]\label{thm_RA_mrs} 
			Grant Assumptions \ref{ass_Z}, \ref{ass_A_general} and \ref{ass_E_general} as well as $\kappa\ge c$ for some absolute constant $c>0$. Choose $\gamma \le c'/\kappa$  for some sufficiently small constant $c'>0$. Further assume   
			\begin{equation}\label{cond_omega_n_msr}
				r \left(\omega_n   + \epsilon^2 \right)  \le  c''
			\end{equation}
			for some small constant $c''>0$. For any $\delta \ge 0$, there exists a permutation $\pi:[r] \to [r]$ such that with probability at least $1-n^{-C}-\delta$, the initialization $\wh  Q_{\pi(k)}^{(0)}$ in \eqref{def_init_msr} with $N \ge 4r^2/\delta^2$  satisfies 
			$\cE^{(k)} _{\init}(\mu, \nu)$  for all $k\in [r]$
			with 
			\begin{equation}\label{rate_mu_nv_mrs}
				\mu = 1 - C' ~ r \left(\omega_n + \epsilon^2 \right)  ,\qquad \nu =  \mu - \sqrt{1-\mu^2}.
			\end{equation}
			Consequently, conclusions in \cref{thm_RA} and \cref{thm_A_general} hold.
		\end{theorem}
		\begin{proof}
			The proof appears in \cref{app_proof_thm_RA_mrs}. 
		\end{proof}
		
		Note that condition \eqref{cond_omega_n_msr} ensures $\mu \asymp \nu \asymp 1$.
		Comparing to the random initialization in \cref{thm_RA_rand}, due to the improved orders of $\mu$ and $\nu$, condition \eqref{cond_omega_n_msr} is weaker than \eqref{cond_omega_n_rand} in terms of dependence on the lower dimension $r$. Condition \eqref{cond_omega_n_msr} could also be  relaxed to $(\omega_n + \epsilon^2)\sqrt{r} \le c''$ by using a data splitting technique in \cite{auddy2023}. However, since splitting the data potentially reduces finite sample performance and the lower dimension $r$ in our context is typically small or moderate, we did not opt for such data splitting strategy. 
		
		\begin{remark}[Rates of the method of moments estimator]
			As a byproduct, our proof of \cref{thm_RA_mrs} also gives the convergence rate of the estimation error of the method of moments estimator. Specifically, let  $\wh \bA_{MoM} = (\wh Q_1^{(0)}, \ldots, \wh Q_r^{(0)})$ with $\wh Q_k^{(0)}$ given by \eqref{def_init_msr}. We prove that, up to a signed permutation matrix,
			\[
			\|\wh \bA_{MoM} - \bR^\T \bA\|_F = \cO_\P\left(
			\left(\omega_n   + \epsilon^2 \right)    r^{3/2}
			\right).
			\]
			Comparing to the estimator in \cref{thm_RA}, this rate is slower by a factor of $r$.
		\end{remark}

	\section{Improvement in the setting of structured noise}\label{sec_ext}

	In this section, we offer two steps of improvement over the procedure described in \cref{sec_method} when the covariance matrix of the additive noise $E$ under model \eqref{model_X} is structured. We will focus on the case when $E$ satisfies \cref{ass_E_general} with $\bSigma_E=  \bI_p$, as this is an important case in various applications \citep{lawley1962factor}. 
	Extension to other structures  is discussed in \cref{rem_ext}.   
	The first step of improvement advocates a different low-dimensional representation from the principal components, while the second step modifies the PGD iterate in \eqref{iter_PGD} when performing deflation varimax rotation.  Essentially these two improvements aim to reduce $\|\bDelta\|_\op$ and $\sigma_N^2$ in \cref{thm_critical}, respectively. 

	\subsection{Improvement over using the principal components} 
	Recall the eigen-decomposition of $n^{-1}\bX\bX^\T$ in \eqref{eq_eigen_Y} and the SVD of $\bLambda = \bL \bS \bA$ in \eqref{svd_A}. Our proof of \cref{thm_PCs} reveals that the bias of $\bU_{(r)}$ for approximating $\bR^\T \bA \bZ/\sigma$ comes from using the eigenvalues $\bD_{(r)}=\diag(d_1,\ldots,d_r)$ to estimate $\sigma^2 \bS^2$. To reduce such bias, we propose to use 
	\begin{equation}\label{def_PCs_hat}
		\wh\bU_{(r)} = \wh \bD_{(r)}^{-1/2}\bV_{(r)}^\T \bX
	\end{equation}
	in place of $\bU_{(r)}$, where
	\begin{equation}\label{def_D_Lambda}
		\wh\bD_{(r)} =  \bD_{(r)} - \bV_{(r)}^\T \left(\wh{\epsilon^2\sigma^2} ~ \bI_p\right) \bV_{(r)} = \bD_{(r)} - \  \wh{\epsilon^2\sigma^2} ~ \bI_r 
	\end{equation}
	is a better estimator of $\sigma^2 \bS^2$ than $\bD_{(r)}$ with
	\begin{equation}\label{def_epssigma_hat}
		\wh{\epsilon^2\sigma^2} = {1\over p-r}\sum_{j>r}^p d_j.
	\end{equation}
	The quantity $\wh{\epsilon^2\sigma^2}\bI_p$ can be treated as an estimator of  $\Cov(E) = \epsilon^2\sigma^2\bI_p$. Indeed, we show in \cref{app_proof_thm_PCs_hat} that $|\wh{\epsilon^2\sigma^2} -\epsilon^2\sigma^2| = \cO_\P(\omega_n)$.  
	
	The following theorem states a decomposition of $\wh\bU_{(r)}$ similar to that of $\bU_{(r)}$ in \cref{thm_PCs} but with a smaller approximation error.  
	
	\begin{theorem}\label{thm_PCs_hat}
		Under Assumptions \ref{ass_Z}, \ref{ass_A_general} and \ref{ass_E_general}, assume $r\log(n) \le cn$ and $\epsilon^2p\log(n)\le cn$   for some sufficiently small constant $c>0$ and $\epsilon^2 \le 1$. Then there exists some  $\bR\in\bbO_{r\times r}$ such that,  with probability at least $1-n^{-1}$,  one has
		\begin{equation}\label{decom_Ur_hat}
			\wh \bU_{(r)} = 
			\bR^\T \left(\bA  \bZ / \sigma + \bN + \bOmega\right)
		\end{equation}
		where  $\bN =  \bS^{-1}\bL^\T \bE / \sigma
		$ and $\bOmega = \bDelta (\bZ/\sigma + \bN)$ 
		with 
		$
		\|\bDelta\|_\op    \lesssim \omega_n.
		$
	\end{theorem}
	\begin{proof}
		Its proof appears in \cref{app_proof_thm_PCs_hat}.
	\end{proof}
	
	Comparing to \cref{thm_PCs}, the bound of the approximation error $\|\bDelta\|_\op$ in  $\wh\bU_{(r)}$ is smaller due to using $\wh \bD_{(r)}$.  From \cref{thm_critical}, a smaller $\|\bDelta\|_\op$ also reduces the estimation error of the subsequent deflation varimax.  Deploying $\wh\bU_{(r)}$ in \eqref{def_PCs_hat} thus can be treated as the first step of improvement, after which one could continue solving \eqref{obj_L4} by using $\wh\bU_{(r)}$ in lieu of $\bU_{(r)}$, and estimating the loading matrix as  in \eqref{def_Lambda_hat} with $\bD_{(r)}$ replaced by $\wh\bD_{(r)}$. 
	In the following, we discuss a second step of improvement in solving \eqref{obj_L4} for performing deflation varimax.

	\subsection{Improvement of the deflation varimax rotation}\label{sec_theory_alg_ext}

	Recall from \cref{thm_critical} that the addtive error matrix $\bN$ affects the error of estimating $\bA$ through $\sigma_N^2 = \sup_{q \in \Sp^r}\|\bP_{q}^{\perp} \bSigma_N q\|_2$. In the discussion after  \cref{thm_critical}  we mentioned that $\sigma_N^2=0$ when $\bSigma_N$ is proportional to the identity matrix. In this section we  show that as long as $\bSigma_N$ can be  estimated consistently, we can remove   $\sigma_N^2$ in \cref{thm_critical} by the proposed modification.

	To explain the main idea,  our proof of \cref{thm_critical} reveals that $\sigma_N^2 $ originates from one term in the difference between the Riemannian gradient, $\grad F(\bq; \bR \bU_{(r)})$, and its population-level counterpart  (see, \cref{lem_bd_grad_F_tilde}, for details). This term has an explicit expression  and in fact can be viewed as the bias of estimating $\bA$ caused by the additive noise $\bN$ in \eqref{decom_Ur}.  Therefore, the idea is to subtract this term  by using its estimate during PGD.  
	
	Concretely, let $\wh\bSigma_N$ be an estimator of $\bR^\T  \bSigma_N \bR$ with  $\bR$ defined in \eqref{decom_Ur}. For solving \eqref{obj_L4} the $k$th time with each $ k\in [r]$,  we propose to iterate as follows until convergence:
	\begin{equation}\label{iter_PGD_corrected}
		\wh  Q_k^{(\ell+1)}= P_{\Sp^r}\left(\wh  Q_k^{(\ell)} -\gamma ~ \wh\grad   F\bigl(\wh  Q_k^{(\ell)}; \bU_{(r)} \bigr)\right),\quad \text{for }\ \ell = 0,1,2,\ldots
	\end{equation}
	where for any $\bq\in\Sp^r$,
	\begin{equation}\label{def_grad_F_hat}
		\wh\grad   F\bigl(\bq; \bU_{(r)} \bigr) := \grad  F\bigl(\bq; \bU_{(r)} \bigr) + \left(
		1 + \bq^\T  \wh\bSigma_N  \bq 
		\right)\bP_{\bq}^\perp \wh\bSigma_N \bq, 
	\end{equation} 
	Comparing to the PGD in \eqref{iter_PGD}, the second term in \eqref{def_grad_F_hat} is the estimator of the bias term mentioned above. Regarding initialization in \eqref{iter_PGD_corrected}, we can adopt the same initialization schemes discussed in \cref{sec_init}.
	
	Regarding the estimation of $\bSigma_N$, using its structure in \eqref{eq_Sigma_N} together with $\bSigma_E = \bI_p$ gives 
	\[
	\bSigma_N	= \epsilon^2 \bS^{-1}\bL^\T \bSigma_E \bL\bS^{-1} = \epsilon^2 \bS^{-2}. 
	\]  
	Recall that $\wh\bD_{(r)}$ in \eqref{def_D_Lambda}  consistently estimates $\sigma^2 \bS^2$ and $\wh{\epsilon^2\sigma^2}$ in  \eqref{def_epssigma_hat} consistently estimates $\epsilon^2 \sigma^2$. We thus propose to estimate $\bR^\T \bSigma_N \bR$ by  
	\[
	\wh \bSigma_N := \wh{\epsilon^2\sigma^2} ~ \wh\bD_{(r)}^{-1}.
	\] 
	
	The following theorem provides rate of convergence of the estimation error of  the resulting rotation matrix $\wc\bQ$  obtained by applying \cref{alg_rotate} to $\wh\bU_{(r)}$ with \eqref{iter_PGD_corrected} in place of the original projected gradient descent iterations in \eqref{iter_PGD}.
	
	\begin{theorem}\label{thm_A_corr}
		Grant the conditions in \cref{thm_one_col} with condition \eqref{cond_omega_n} replaced by 
		\begin{equation}\label{cond_omega_n_general_hat}
			\sqrt{r^2\log(n)\over n} + \sqrt{\epsilon^2p\log(n)\over n} ~ \le  c' \mu \nu^2.
		\end{equation}
		Then there exists a signed permutation matrix $\bP$ such that with probability at least $1 - 3/n$, 
		\begin{align*} 
		  \|\wc \bQ - \bR^\T \bA\mb P\|_{F}~  \lesssim ~ \omega_n \sqrt{r},
		\end{align*} 
		and the estimator $\wh\bLambda$ in \eqref{def_Lambda_hat} with $\bD_{(r)}^{1/2}$ replaced by $\wh \bD_{(r)}^{1/2}$
		satisfies 
		\[ 
		\|\wh \bLambda - \bLambda\mb P\|_{F}~  \lesssim ~  \omega_n \sqrt{r}.
		\]
	\end{theorem}
	\begin{proof}
		Its proof appears in \cref{app_proof_thm_A_corr}.
	\end{proof}
	
	In \cref{app_sec_improv} we also state a general version of \cref{thm_A_corr} which is valid for any consistent estimator of $\bSigma_N$. Such result is useful for extension to other structures of $E$.
	
	Comparing to  \cref{thm_RA} and \cref{thm_A_general}, the rate in \cref{thm_A_corr} does not contain the term $\epsilon^2\sqrt{r}$, hence is faster when $\epsilon^2  \ge  \min\{p/n, \sqrt{r/n}\}$.  For the same reason, condition \eqref{cond_omega_n_general_hat} is also weaker than \eqref{cond_omega_n} in \cref{thm_one_col}.  
	In view of  \cref{thm_lowerbounds}, we conclude that the improved  rate in \cref{thm_A_corr} is minimax optimal across all regimes of the signal-to-noise ratio.
	
	\begin{remark}[Comparison with identification and estimation of the column space of the loading matrix]
		It is worth comparing our results of estimating the loading matrix $\bLambda$ with the guarantees of using PCA to estimate the column space of $\bLambda$. To make a fair comparison, we consider the case $\bLambda \in \bbO_{p\times r}$ and $\bSigma_E =  \bI_p$. In such case PCA is known to estimate $\bLambda$, up to $r\times r$ {\em orthogonal matrices}, in the minimax optimal rate $\cO_\P(\sqrt{\epsilon^2 p r / n})$ provided that $p\ge \max\{3r, r+4\}$ \citep[Theorem 3.1]{VuLei2013}. Since  comparison in the case  $\epsilon^2  \ll  r/p$ is less interesting as $\epsilon^2=0$ leads to perfect estimation for PCA, we focus our discussion on the case  $\epsilon^2 p \ge r$. It is interesting to note that the optimal rate of PCA  is the same (ignoring the logarithmic factor) as  that   obtained in \cref{thm_A_corr}  for estimating $\bLambda$, notably, only up to $r\times r$ {\em signed permutation matrices}.  It implies that from the statistical perspective, we do not lose any statistical accuracy for identifying and estimating $\bLambda$, up to  this strict subset of orthogonal matrices. The extra price we pay is the assumption of the factor $Z$ being  independent and leptokurtic.  From the computational perspective, the proposed minimax optimal estimator of $\bLambda$ requires to additionally compute the deflation varimax rotation. Since the latter can be  computed in low order polynomial running time (see, Theorems \ref{thm_one_col}, \ref{thm_RA_rand} and \ref{thm_RA_mrs}), we do not encounter any computational barrier either.
	\end{remark}

	\begin{remark}[Improvement on initialization]\label{rem_improve_init}
		Similar to \cref{thm_RA_rand} and \cref{thm_RA_mrs}, condition \eqref{cond_omega_n_general_hat} can be instantiated by  both   initialization  schemes in \cref{sec_init}.  Concretely, for the random initialization, the same expressions of $\mu$ and $\nu$ given by \eqref{mu_nu_rand} can be derived and condition \eqref{cond_omega_n_general_hat} becomes 
		\[
		r^{3/2} \left(
		\sqrt{r^2\log(n)/ n} + \sqrt{\epsilon^2 p \log (n) / n}     
		\right) \le  c'' \delta^2,
		\]
		which is  weaker than \eqref{cond_omega_n_rand} in \cref{thm_RA_rand}. 
		For the method of moments based initialization, by writing $\wh \bSigma_U := \bI_r + \wh\bSigma_N$ and by replacing \eqref{def_M_hat_main} with 
		\begin{align}\label{def_M_hat_impro}
		\wh \bM(\bG_{(i)}) = {1\over 3n}\sum_{t=1}^n U_t U_t^\T  U_t^\T \bG_{(i)} U_t -\wh \bSigma_U \left(\bG_{(i)} +\bG_{(i)}^\T\right)\wh \bSigma_U  - \tr\left(
		\bG_{(i)} \wh \bSigma_U 
		\right)\wh \bSigma_U,
		\end{align}
		one can prove similar results of \cref{thm_RA_mrs} with condition  \eqref{cond_omega_n_msr} weakened to 
		$
		r \omega_n  \le  c''
		$
		and the expressions of $\mu$ and $\nu$ improved to
		\[
		\mu = 1 - C'   r \omega_n,\quad \nu =  \mu - \sqrt{1-\mu^2}.
		\] 
	\end{remark} 
	
	\begin{remark}[Extension to other structured noise]\label{rem_ext}
		The correction in \eqref{def_grad_F_hat} uses the Gaussianity of the additive error $\bN$ through 
		$\EE[N_{ij}^4] = 3\EE[N_{ij}^2]$ for each $i\in [r]$ and $j\in [n]$. For non-Gaussian case, \eqref{def_grad_F_hat} could be modified provided that the fourth moment of $N_{ij}$ can be estimated consistently. 
		Although we focus on the case $\bSigma_E = \bI_p$, both steps of  improvement mentioned above can be applied to other structures of $\bSigma_E$. For example, if $\bSigma_E$ is diagonal  but with distinct diagonal elements, then instead of using PCA, one should resort to either the classical factor analysis (such as \cite{BaiLi2012}) or the heteroscedastic PCA in \cite{zhang2022heteroskedastic} to obtain estimators of $\bL$ and $\sigma^2\bS^2$ in lieu of $\bV_{(r)}$ and $\bD_{(r)}$, and to estimate the covariance matrix $\bSigma_E$. Once this is done, the second step of improvement could be applied analogously. For non-diagonal but sparse $\bSigma_E$, one might resort to the approaches of recovering an additive decomposition between a low-rank matrix and a sparse matrix \citep{Hsu2011,chandrasekaran2012latent,candes2011robust} to compute matrices similar to $\bV_{(r)}$ and $\bD_{(r)}$, and to estimate  $\bSigma_E$. We leave thorough analysis of these extensions to future research.
	\end{remark}

	 \section{Real data analysis}\label{sec_real_data}
	
	  In this section, we discuss  selection of the number of factors in \cref{sec_select_r} and demonstrate the efficacy of our proposed approach through two real data analyses in Sections \ref{sec_data_digits} and \ref{sec_data_imgs}.
	
	 \subsection{Selection of the number of factors}\label{sec_select_r}
	
	     Choosing the number of factors under model \eqref{model_X} has been one of the most fundamental problems in factor analysis. From the theoretical perspective, many criteria have been proposed and shown to be consistent under certain regimes. This includes information based criteria in \cite{BaiNg2002}, eigen-gap based criteria in \cite{Ahn-2013,lam2012,bing2020prediction} and parallel analysis \citep{PA}, just to name a few. However, on the practical side, the screeplot is still widely used  to determine the number of factors despite its subjectivity. 
	
	     For analyzing real data sets, a common empirical phenomenon is  that different choices of the number of factors all yield meaningful results. 
	     As commented in \cite{rohe2020vintage}, ``many times the factors have something resembling a hierarchical structure ... The {\em Cheshire Cat Rule} says that there is not a single correct answer for the choice of $r$, that answer depends upon where you want to go.'' In our real data analysis, we chose the number of factors that yields meaningful results. We also had similar findings across a range of different choices.

	 \subsection{Learning meaningful basis of handwritten digits}\label{sec_data_digits}
	
	 In digital image processing, an important task, also known as sparse coding, is to find a meaningful set of basis so that the original images can be represented as approximately sparse linear combinations of this basis  \citep{field1994goal,olshausen1996emergence}.  In this section, we demonstrate the efficacy of  our proposed algorithm for this purpose 
	 on a dataset consisting of $8000$ handwritten digit images obtained from the MNIST dataset \citep{Lecun1998Minst}.   
	 Each original image consists of $28\times28$ pixels, vectorized as a $784$-dimensional vector. By concatenating these vector-level pixels across all images, we obtain a  $p\times n$ data matrix $\bX$ with $p = 784$ and $n = 8000$.  Our primary objective is to apply the proposed PCA-dVarimax to learn a $p\times r$ matrix $\bLambda$ consisting of basis of all images, as well as an approximately sparse, low-dimensional representation $\bZ\in\R^{r\times n}$ of $\bX$ with $r =49$. 
	
	 For the proposed PCA with deflation varimax (PCA-dVarimax), in both real data analyses and simulations in \cref{sec_sim}, we chose the step size in \eqref{iter_PGD} 
    as $\gamma = 10^{-5}$, and adopt the stopping criterion when either $\|\grad F(\wh Q_k^{(\ell)}; \bU_{(r)})\|_2 \le 10^{-6}$ or we reach $5000$ iterations. 
	
	 \cref{fig:learnedBasis} depicts the learned 49 basis of PCA-dVarimax, with each small panel corresponding to one column of $\wh\bLambda$, given by \eqref{def_Lambda_hat}. For comparison, we also include the basis learned from applying PCA only. 
	 It is clear that most of the learned basis from PCA-dVarimax represent certain meaningful patterns of handwritten digits whereas the basis of PCA provide little understanding of the underlying structure. 
	 In \cref{fig:app_spca} we further plot four pairs of the rotated Principal Components (PCs), that is, the rows of $\wc \bQ^\T\bU_{(r)}$ from PCA-dVarimax, and contrast them with the original, unrotated PCs. Evidently, the PCs after rotation tend to be  approximately more sparse, a phenomenon called {\em radial streaks} when the axes are aligned with the streaks, as introduced by \cite{Thurstone1947} (see, also, \cite{rohe2020vintage}). This finding  is in line of \cref{rem_Z}.

	 \begin{figure}[ht]
		 \centering
		 \includegraphics[width=.47\textwidth, height=.3\textheight]{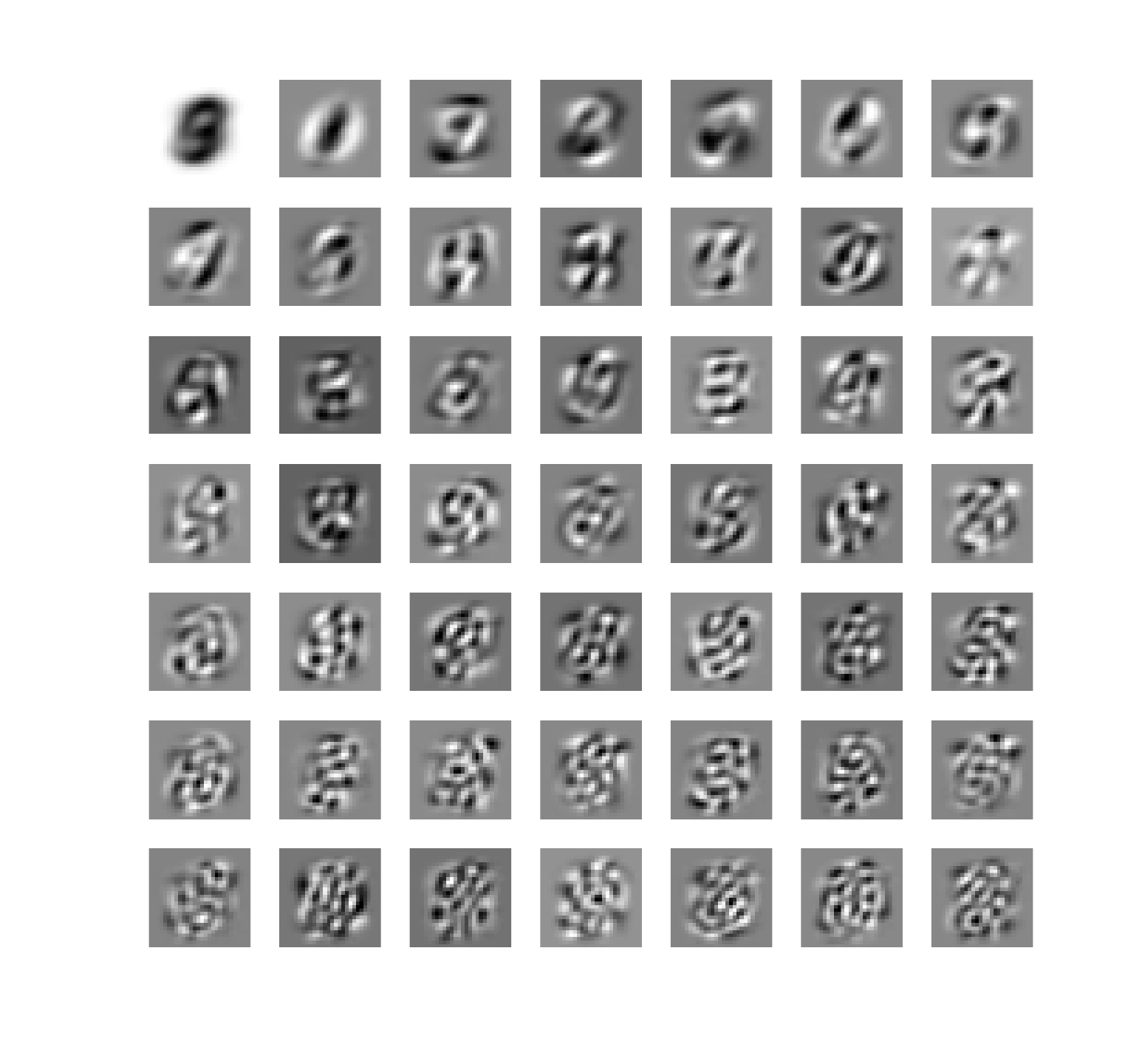}
		 \hspace{-5mm}
		 \includegraphics[width=.47\textwidth, height=.3\textheight]{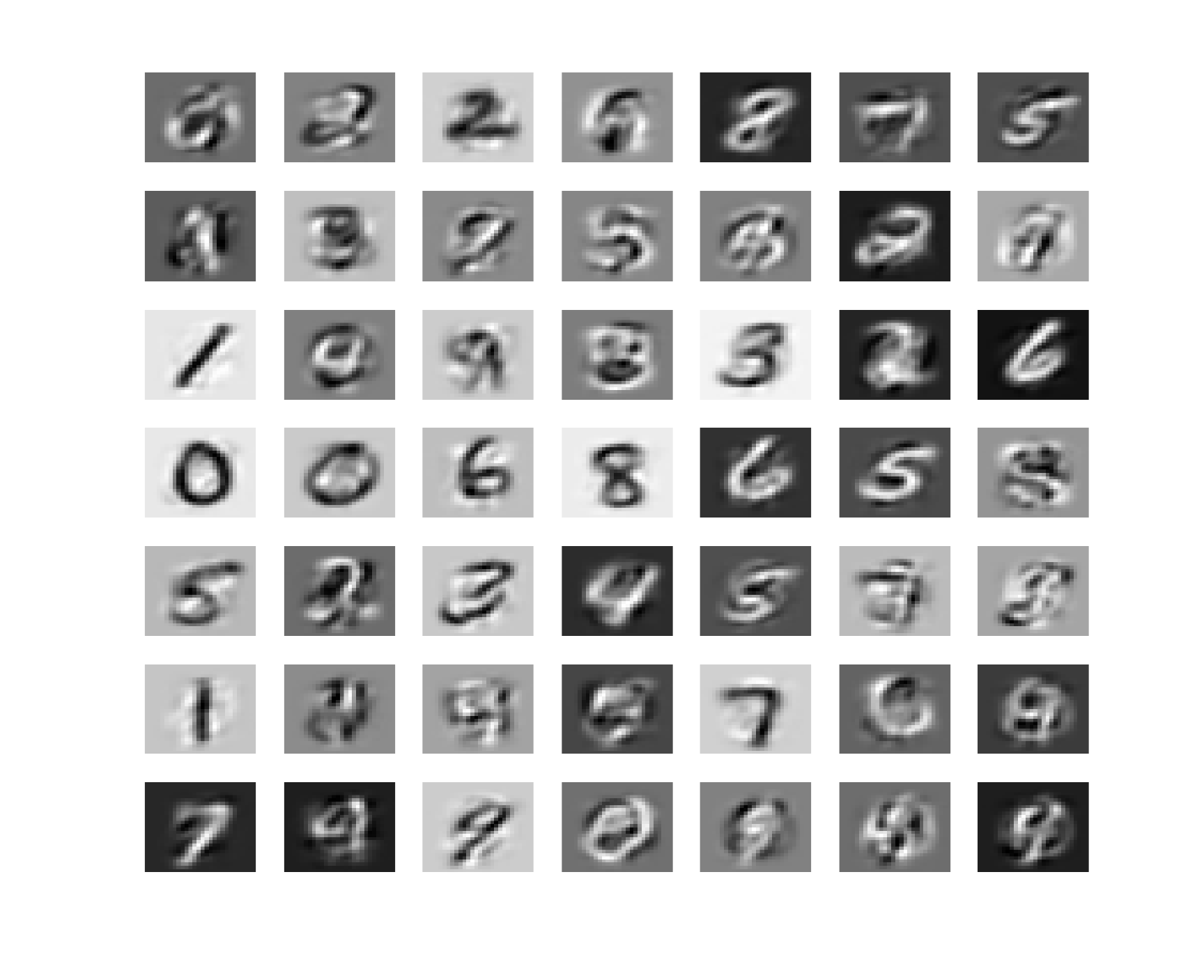}
		 \vspace{-5mm}
		     \caption{Learned 49 basis by PCA (left) and PCA-dVarimax (right).}
		     \label{fig:learnedBasis}
		 \end{figure}
	
	 \begin{figure}
		   \begin{subfigure}[h!]{0.5\textwidth}
			          \centering
			          \includegraphics[width=\textwidth]{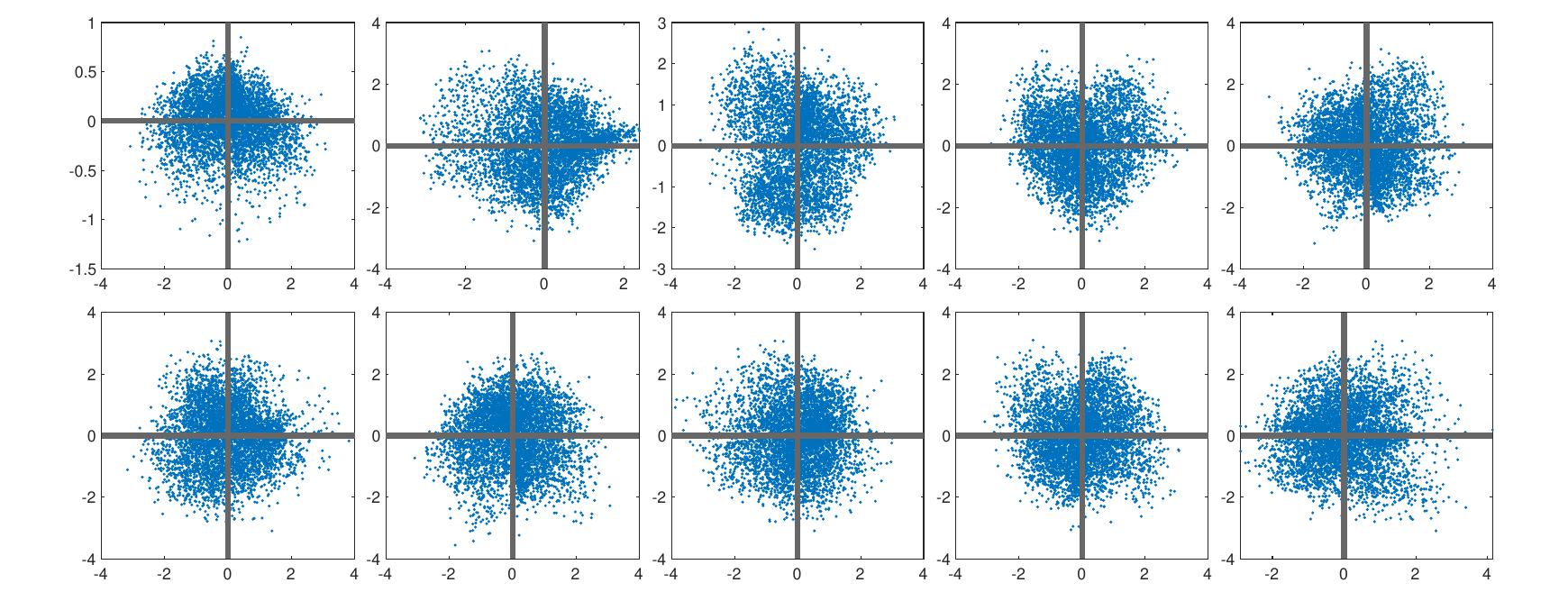}
			      \end{subfigure}
		      \hspace{-3mm}
		      \begin{subfigure}[h!]{0.5\textwidth}
			          \centering
			          \includegraphics[width=\textwidth]{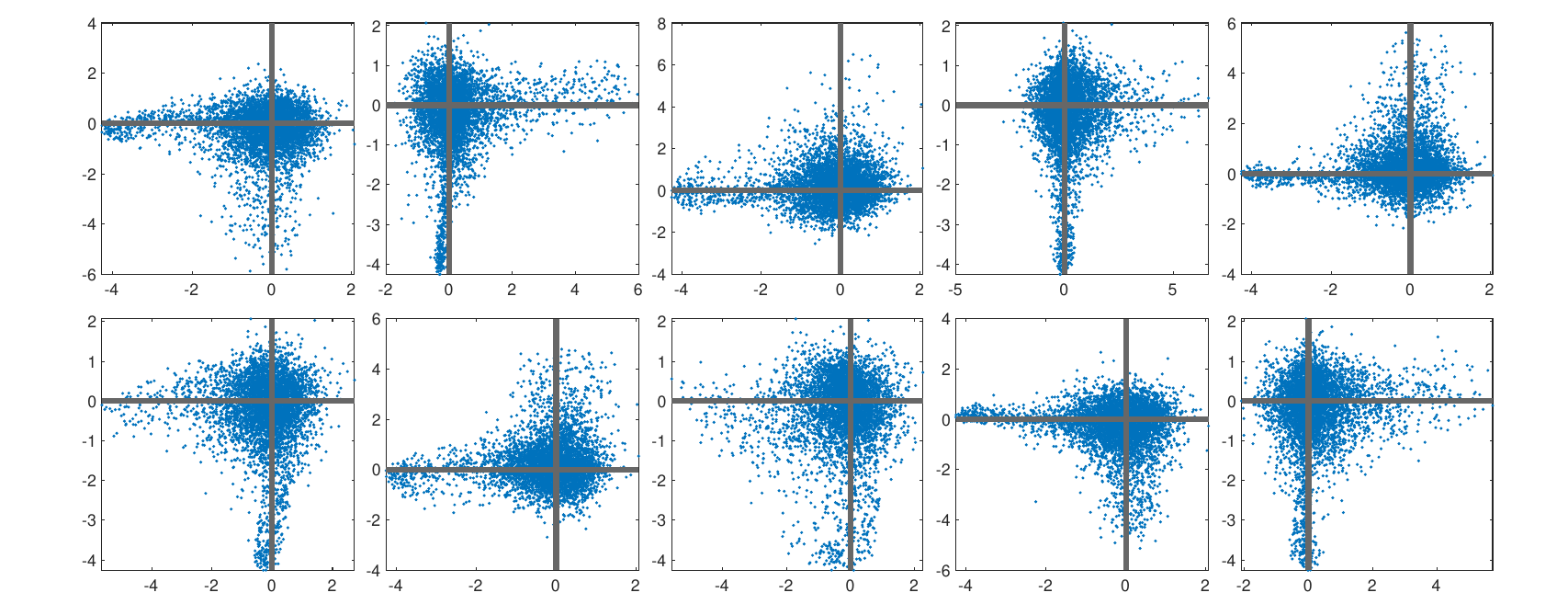}
			      \end{subfigure}
		 \vspace{-.1in}
		     \caption{Ten selected pairs of unrotated PCs (left) and rotated PCs (right).}
		     \label{fig:app_spca}
		 \end{figure}

	 \subsection{Recovering signals from mixed images}\label{sec_data_imgs}
	
	 In this section, we demonstrate the effectiveness of the proposed PCA-dVarimax for recovering signals from mixed images, a classical application of Independent Component Analysis (ICA) for feature extraction \citep{hyvarinen2000independent}. We start by selecting two grayscale images, and by generating three noise images using i.i.d  samples from $\cN(0, 80^2)$. The first row of  \cref{fig:RealBSS_image} shows these five original images. Each of these images is originally represented as a $171\times 171$ matrix, and then reshaped to a vector of length $n = 29241$. Stacking these vectors together yields the $r\times n$ factor matrix $\bZ$ (with $r = 5$)  in which the two rows corresponding to grayscale images are the signals that we would like to extract. 
	 Next we generate $p = 100$ linear combinations of these five images and mix them with additive Gaussian noise. Specifically, we generate the data matrix as $\bX=\bLambda\bZ+ \bE$  where $\bLambda\in \mathbb R^{p\times r}$ has i.i.d standard normal entries with normalized operator norm, and the additive noise matrix $\bE\in \mathbb R^{p\times n}$ is generated as in  \cref{sec_sim} with $\bSigma_E = \bI_p$ and $\eps=0.5$. The second row in \cref{fig:RealBSS_image} showcases five randomly selected rows of $\bX$.
	
	 To recover the two rows of $\bZ$ corresponding to the signal images, in view of \eqref{def_Lambda_hat} and \cref{rem_Z}, it is natural to predict $\bZ$ by the scaled rotated Principal Components (PCs),
	 \begin{equation}\label{def_Z_hat}
		    \wh \bZ = \wc\bQ^\T \bU_{(r)} ~ \|\bV_{(r)} \bD_{(r)}^{1/2}\wc\bQ\|_\op = \wc\bQ^\T \bU_{(r)} \|\bD_{(r)}^{1/2}\|_\op.
		 \end{equation}
	 The matrices $\bD_{(r)}$, $\bU_{(r)}$ and $\wc\bQ$ are defined in \eqref{eq_eigen_Y}, \eqref{def_PCs} and \eqref{def_wc_Q}, respectively.  \cref{fig:RealBSS_result} plots the resulting five rows of $\wh\bZ$ from PCA-dVarimax. For comparison, we also include the predicted rows of $\bZ$ from applying:  (1) PCA-dVarimax-2, the two-step improved PCA-dVarimax introduced in \cref{sec_ext}, also see details in \cref{subsec_exp_improv}, (2) Fast-ICA, the popular ICA algorithm   \citep{hyvarinen1997fast,hyvarinen2000independent}, (3)   Fast-ICA-Tensor, the same Fast-ICA algorithm except for using  initialization in \cite{auddy2023},
	 (4) PCA,  corresponding to $\wc\bQ = \bI_r$ in \eqref{def_Z_hat},  and (5) Varimax, with $\wc\bQ$ in \eqref{def_Z_hat} replaced by the rotation matrix obtained from solving the varimax criterion in \eqref{obj_varimax}. We use the function \textsf{rotatefactors} in MATLAB2023b  to compute the varimax rotation.\footnote{Code to reproduce numerical results is available in \url{https://github.com/Jindiande/optimal_deflation_varimax}.}
	
	 Notably, PCA-dVarimax and PCA-dVarimax-2 are the only two algorithms that effectively recover  both signal images (the first and the fifth rotated PCs). Moreover, PCA-dVarimax-2 recovers the second signal  better as seen from the fifth rotated PCs. This is as expected according to  our theoretical findings in \cref{sec_ext}. By contrast,  Fast-ICA, Fast-ICA-Tensor and Varimax only captured one signal image while PCA  failed to separate the original two signals from each other.  
	 
	 \begin{figure}[!htbp] 
		 \begin{subfigure}{\textwidth}
			         \centering
			 \includegraphics[width=\textwidth,height=5cm]{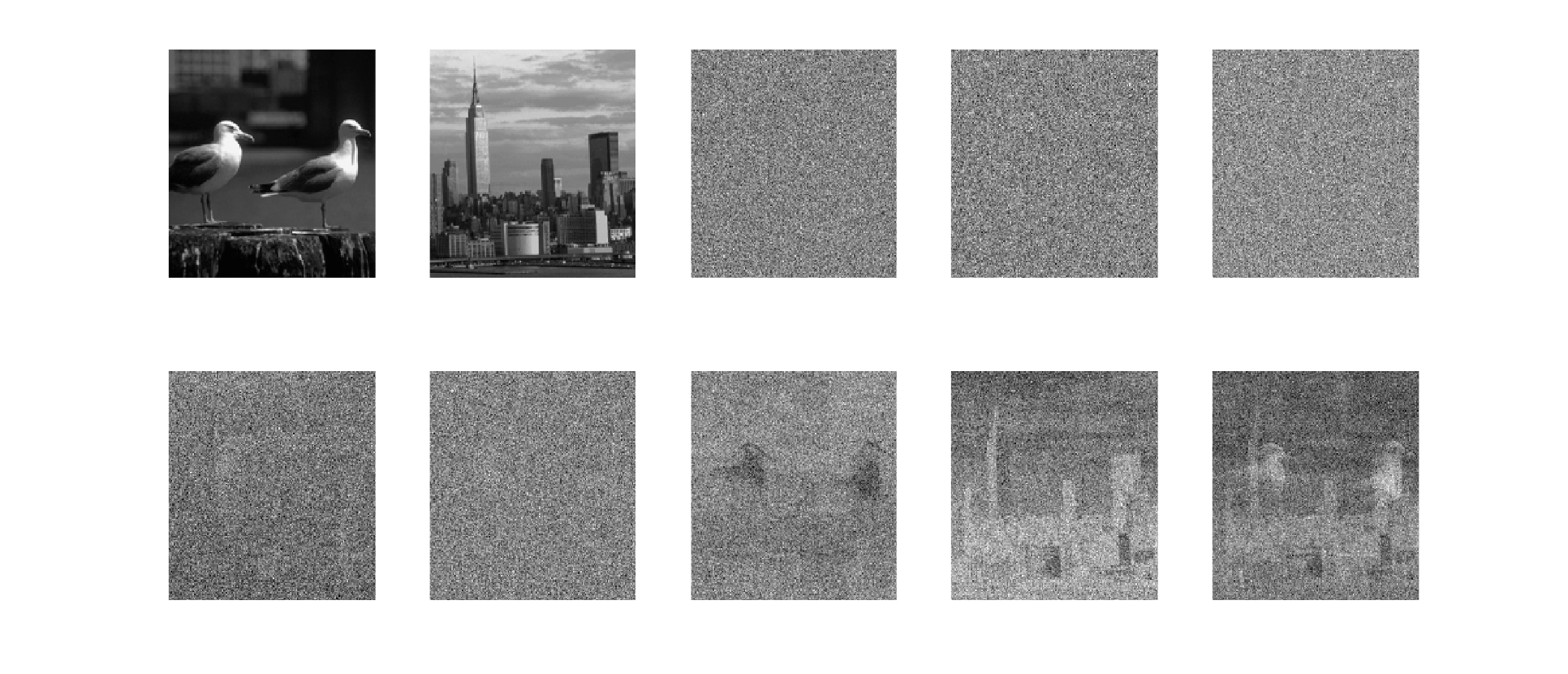}
			 \vspace{-.3in}
			 \caption{Input images. The first row contains five original images in $\bZ$ (2 signal images and 3 Gaussian-noise images). The second row contains five randomly selected mixed images in $\bX$.}
			 \label{fig:RealBSS_image}
			 \end{subfigure}
		 \vfill
		 \begin{subfigure}{\textwidth}
			         \centering
			 \includegraphics[width=\textwidth,height=15cm]{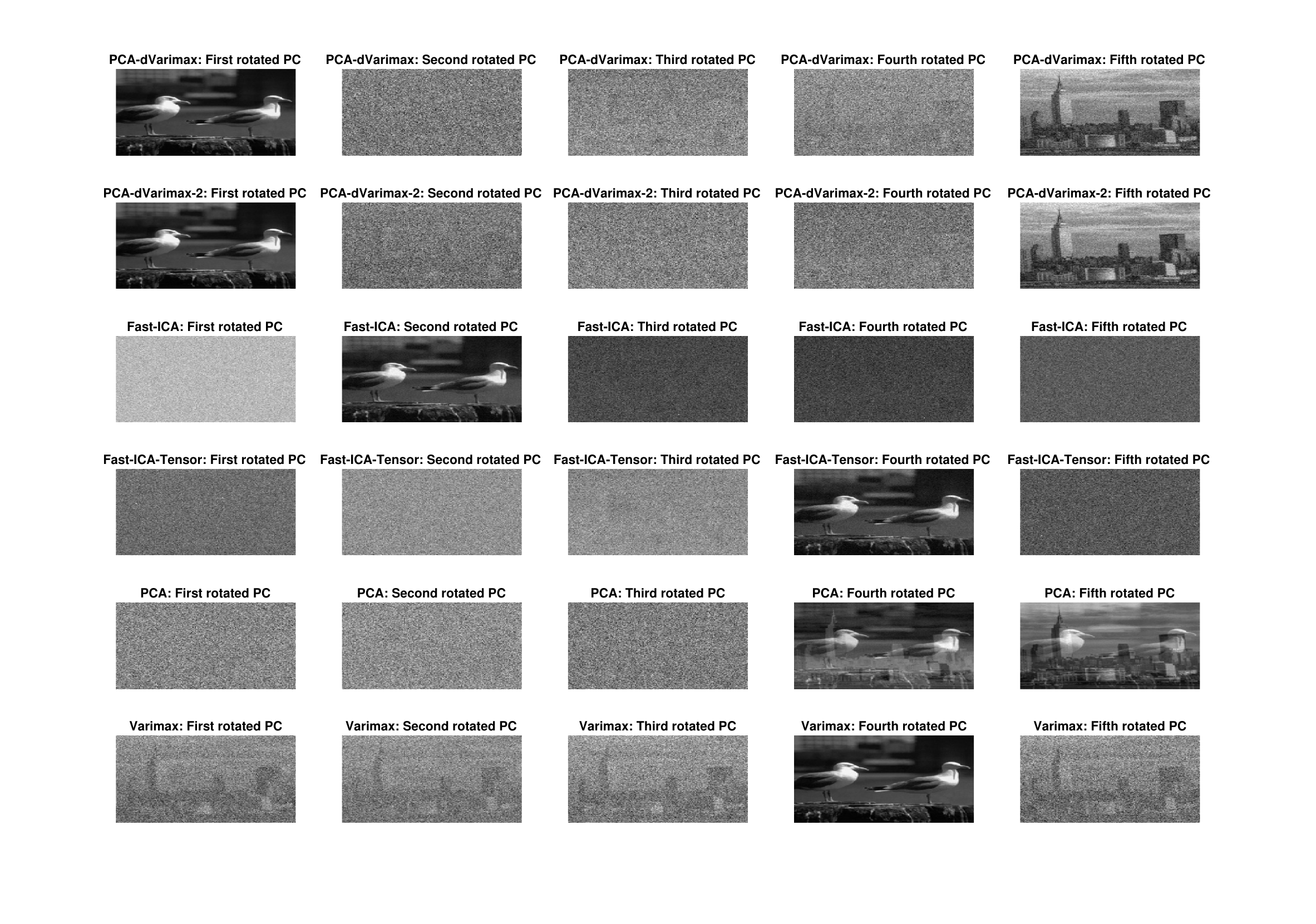}
			 \vspace{-.7in}
			 \caption{Recovered images from  PCA-dVarimax, PCA-dVarimax-2, Fast-ICA, Fast-ICA-Tensor, PCA, and Varimax.}
			 \label{fig:RealBSS_result}
			 \end{subfigure}
		 \end{figure}

{\small
    \setlength{\bibsep}{0.85pt}{
    \bibliographystyle{plainnat}
    \bibliography{ncvx}
    }
}
	
	\newpage

	\begin{appendix}

        \section{Simulation}\label{sec_sim}
	
	In this section, we conduct additional simulation studies  to verify our theoretical findings. In \cref{sec:loading_compare} we examine the performance of  our proposed PCA-dVarimax for estimating the loading matrix and compare its performance  with other approaches. In \cref{subsec_exp_improv} we verify the benefit of both steps of improvements as discussed in \cref{sec_ext}.
	 
	We start by describing our data generating mechanism. For the loading matrix $\bLambda$, we choose $\bLambda = \bar\bLambda \diag(d_1,\ldots,d_r) $ with entries of $\bar\bLambda\in\R^{p\times r}$ i.i.d.  from $\cN(0, 1)$ and $d_1, \ldots, d_r$ i.i.d. from Uniform$[0.5,1.5]$. The matrix $\bLambda$ is then normalized to the unit operator norm. 
	Entries of $\bZ$ are generated i.i.d. from the Bernoulli-Gaussian distribution, that is, $Z_{ij} = B_{ij}W_{ij}$ with $B_{ij} \sim \text{Bernoulli}(\theta=0.1)$  and $W_{ij}\sim \cN(0,1)$ for all $i\in[r]$ and $j\in [n]$. Finally, columns of the noise matrix $\bE$ are generated i.i.d. from $\cN_p(0, \eps^2 \bSigma_E/p)$ so that the reciprocal of signal-to-noise ratio (SNR), $\epsilon^2$, is proportional to $\eps^2/p$. Regarding the choice of $\bSigma_E$, we consider: (a) $\bSigma_E = \bI_p$;  (b) $\bSigma_E =   \diag(\gamma_1^2,\ldots, \gamma_p^2)$ with 
		$
		\gamma_j^2 = {p v_j^\alpha  /(\sum_{j=1}^p v_j^\alpha)} 
		$
		for all $ j\in[p]$,
		where $v_1,\ldots, v_p$ are i.i.d. from Uniform $[0, 1]$. The magnitude of $\alpha \ge 0$ controls the degree of heteroscedasticity, and is set to 0.1 in our simulation; and (c) $[\bSigma_E]_{ij} = (0.5)^{|i-j|}$ for all $i,j\in [p]$. 
	Since the results for all three choices are nearly identical, we only report the one of  $\bSigma_E=\mb I_p$. 
	
 For a given estimator $\wh\bLambda$ of $\bLambda$, we use the following metric to evaluate its estimation error:
	\begin{align}\label{exp:error_def}
		\min_{\mb P} ~  & \|\wh\bLambda -\bLambda \mb P\|_F ,\quad \text{subject to} \quad \mb P \textrm{ is an $r\times r$ signed permutation matrix.}
	\end{align}
	We will compute this metric for different estimators in various situations  by varying $n$, $p$, $r$ and $\epsilon^2$ one at a time. For each individual setting, we report the averaged estimation error of each procedure over $100$ repetitions.

	\subsection{Performance on estimating the loading matrix}\label{sec:loading_compare}

In this section, we examine the performance of our proposed PCA-dVarimax for estimating the loading matrix $\bLambda$. In \cref{sec:sim_ini_compare} we first investigate the effect of  initialization scheme used in PCA-dVarimax for estimating $\bLambda$.  In \cref{sec:sim_methods_compare}, we then compare   PCA-dVarimax  with other approaches. 

\subsubsection{Effect  of different initialization schemes}\label{sec:sim_ini_compare}

As studied in \cref{sec_init}, PCA-dVarimax is suitable for a wide range of initialization schemes. We now verify this numerically   by considering  the following  initialization schemes coupled with the PCA-dVarimax approach. 

\begin{itemize}
    \item Random Initialization (RI): This is as discussed in Section~\ref{sec_init_rand}.
    
    \item Method of Moments-based Initialization (MoMI): This is as discussed in Section~\ref{sec_init_momi}.
    
    \item Multiple Random Initializations (MRI):  This initialization  scheme  is mentioned in Remark~\ref{rem_mrand} which uses multiple random draws of $g_{(-k)}$.  Specifically, for recovering the $(k+1)$th column with $0\le k < r$, let $g_{(-k)}(\ell)$, for $\ell=1,2,\ldots,L$, be $L$ independent draws as $g_{(-k)}$, and their  corresponding initialization
    $Q_{k+1}^{(0)}(\ell) $ given by \eqref{def_q_init_all_cols} with $g_{(-k)}$ replaced by $ g_{(-k)}(\ell)$.    A reasonable criterion of selecting one from these initializations is based on the loss function $F(\cdot; \bU_{(r)})$ given by \eqref{obj_L4}: 
    \begin{align*}
    	\ell^*=\argmin_{\ell \in [L]} F\Bigl( Q_{k+1}^{(0)}(\ell); \bU_{(r)} \Bigr).
    \end{align*}
	MRI then uses $Q_{k+1}^{(0)}(\ell^*)$ as the initialization. We set  $L=r^2$  in our simulation study. 

\end{itemize}


	\cref{fig_all_errors_ini} shows the averaged estimation errors of PCA-dVarimax using different initialization schemes in various settings. We vary 
	$n\in \{50, 100, 300, 500, 700, 900\}$,  $p \in\{10,20,30,50,70,100,300,700\}$, $r \in \{3, 5, 10, 15, 20\}$ and $\eps^2 \in \{0.1, 0.3, 0.5, 0.7, 0.9\}$ one at a time, and   fix $n = p=300$, $r=5$ and $\eps^2=0.1$ whenever they are not varied. 
	It is clear that all initialization methods lead to similar estimation errors in all settings.   This  indicates that each of the initialization schemes could be used for finding the global solution in  \eqref{obj_L4}, hence could be paired with our proposed PCA-dVarimax approach. This is consistent with our theoretical guarantees in \cref{sec_init} which are established for any generic initialization. 
    
    In \cref{fig_all_errors_converg_ini}, we further examine the speed of convergence of different initialization schemes by comparing their estimation errors after certain number of iterations. We consider the case $n=900$, $p=300$, $r=5$ and $\eps^2=0.1$ with the iteration number varying within $\{1000, 1500, 2000, 2500, 3000, 3500, 4000, 4500, 5000\}$. It is evident to see that PCA-dVarimax-MoMI enjoys the fastest convergence while PCA-dVarimax-RI has the slowest convergence.  PCA-dVarimax-MRI improves upon PCA-dVarimax-RI is comparable to PCA-dVarimax-MoMI. 
    In view of these findings, we only consider the PCA-dVarimax estimator  using MoMI in the rest of our simulation studies. 
	
	From \cref{fig_all_errors_ini}, we observe that PCA-dVarimax has better performance for larger values of $n$ or smaller values of $r$ and $\eps^2$.  As the ambient dimension $p$ increases, the performance of PCA-dVarimax first gets improved and then becomes stabilized. This is because the noise level $\epsilon^2$ in our simulation setting is of order $\eps^2/p$ whence the rate in \eqref{rate_Lambda} becomes  $\sqrt{r^2/n} + \sqrt{\eps^2 r / n} + (\eps^2/p)\sqrt{r}$. This rate gets faster as $p$ initially increases from a small value and stops changing once $p$ becomes large enough. All these findings are in line with our theory in \cref{sec_theory}.

 \begin{figure}
     \centering
     \includegraphics[width=0.8\textwidth]{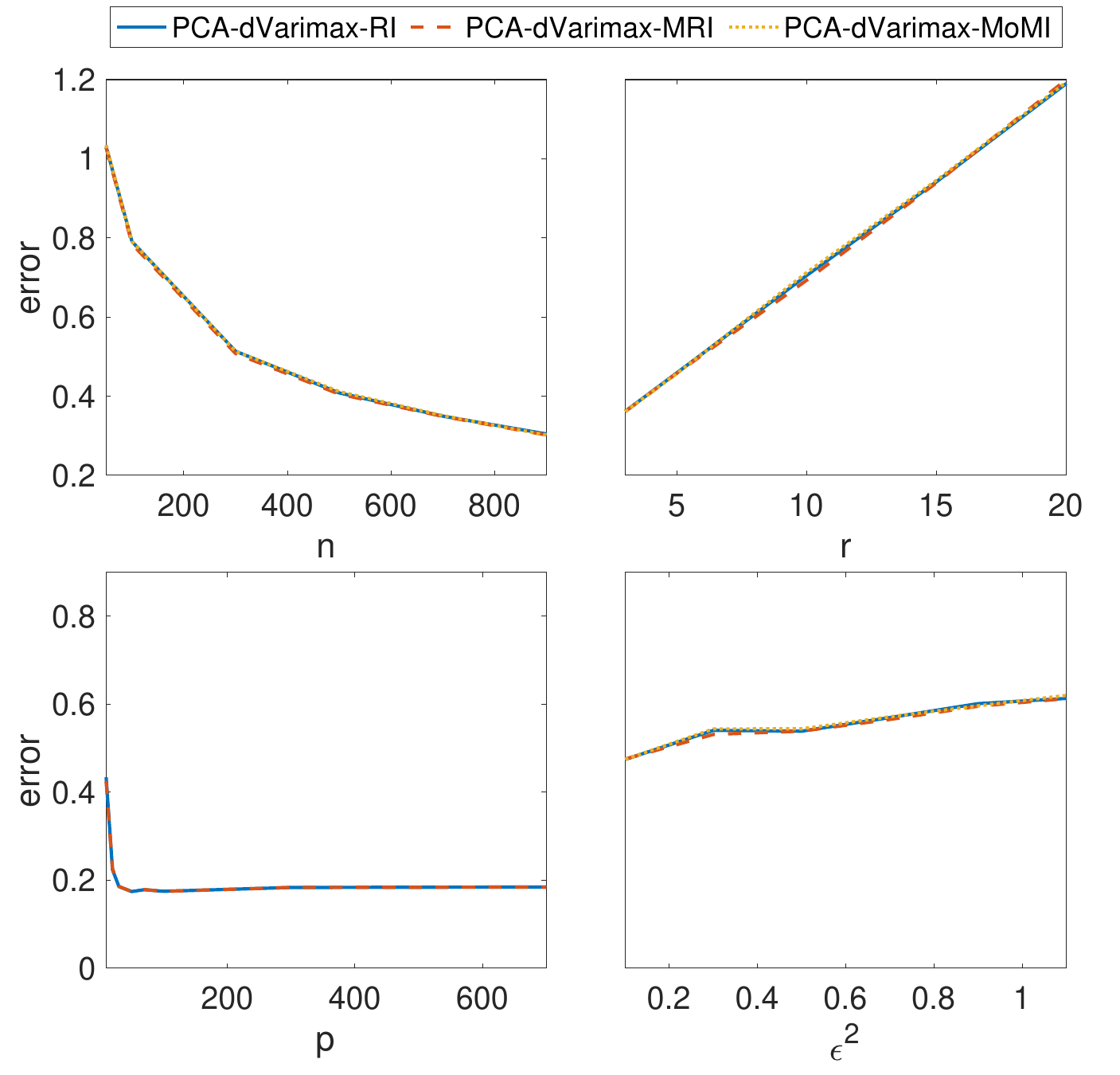}
     \caption{The averaged estimation errors of PCA-dVarimax coupled with different initialization schemes.}
\label{fig_all_errors_ini}
 \end{figure}
 \begin{figure}
     \centering
     \includegraphics[width=0.65\linewidth]{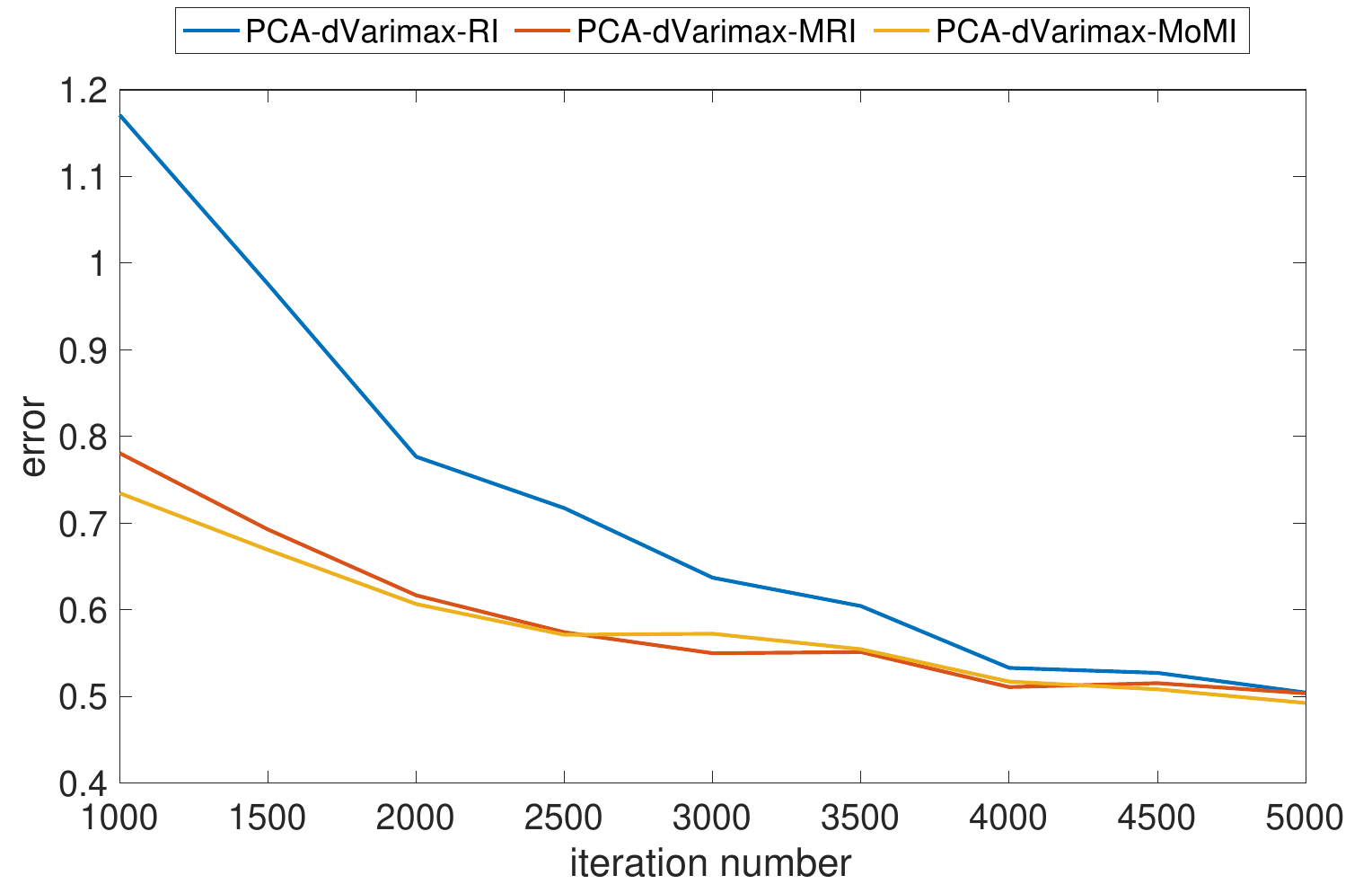}
     \caption{The averaged estimation errors of PCA-dVarimax with different initialization schemes at various numbers of iterations}
\label{fig_all_errors_converg_ini}
 \end{figure}
	
		\subsubsection{Comparison of PCA-dVarimax with other approaches}\label{sec:sim_methods_compare}
		
	Regarding on estimating  the loading matrix $\bLambda$, we compare PCA-dVarimax with  other approaches, including PCA, Fast-ICA, Fast-ICA-Tensor, Varimax and S-PCA. These methods except S-PCA are already introduced in \cref{sec_data_imgs}. 
		The S-PCA estimator refers to the sparse PCA method  \citep{Zou06spca,ShenHuangspca} defined as
		\begin{align}\label{obj_SPCA_est}
			\wh \bLambda_{\textrm{S-PCA}} = {\wh\bV \diag(\wh d_1, \ldots, \wh d_r) \over \max_{1\le i\le r}\wh d_i},\quad \text{with } ~ \wh d_i =  \|\wh U_{i\cdot}\|_2 \text{ for all }i\in [r],
		\end{align}
		where 
		\begin{align}\label{obj_SPCA}
			&(\wh\bV, \wh\bU) = \argmin_{\bV\in \bbO_{p\times r},\bU\in \mathbb{R}^{r\times n}} \left\|\bX-\bV\bU\right\|_F^2 +\lambda\sum_{i,j}|U_{ij}|.
		\end{align}
		In \eqref{obj_SPCA_est}, we modify the original sparse PCA procedure for  estimating the loading matrix $\bLambda$ in our context. 
		For solving \eqref{obj_SPCA}, we adopt the popular alternating direction algorithm  \citep{bertsekas2015parallel,tseng2001convergence} by updating either $\bV$ or $\bU$  alternatively until convergence.

	 \cref{fig_all_errors} depicts the averaged estimation errors of all approaches in various scenarios. We vary 
	 $n\in \{50, 100, 300, 500, 700, 900\}$,  $p \in\{10,20,30,50,70,100,300,700\}$, $r \in \{3, 5, 10, 15, 20\}$ and $\eps^2 \in \{0.1, 0.2, 0.3, 0.4, 0.5, 0.6, 0.7\}$ one at a time, and   fix $n =900$, $p=300$, $r=5$ and $\eps^2=0.1$ whenever they are not varied.  It is clear that  the proposed PCA-dVarimax has comparable performance with Varimax while outperforms all other estimators. Despite the competitive performance in our simulation settings, Varimax does not have any theoretical guarantee, hence might yield unsatisfactory result in practice, as shown in \cref{sec_data_imgs}.  
 It is also worth noting the two Fast-ICA algorithms suffer from moderate and high dimensionality and their performance deteriorate as $p$ increases. 

 
	\begin{figure}[ht]
	    \centering
	    \includegraphics[width=0.95\textwidth]{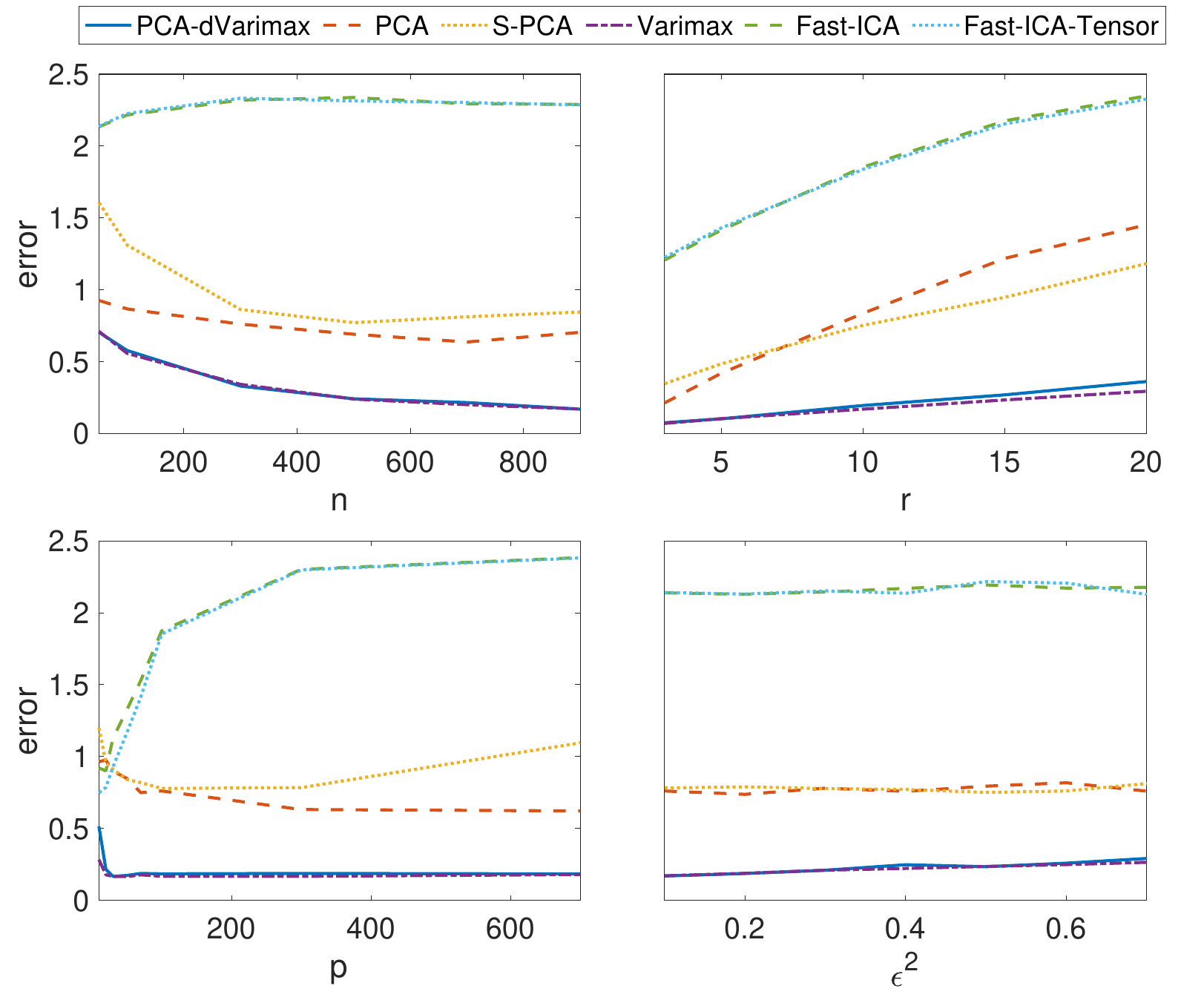}
	    \caption{The averaged estimation errors of each procedure.}
	    \label{fig_all_errors}
	\end{figure}
	
	\subsection{Improvement in the setting of structured noise}\label{subsec_exp_improv}

    In this section we verify the benefit of both steps of improvement proposed in \cref{sec_ext} when  $\bSigma_E = \bI_p$.  
	To separate the effect of each step, we consider two variants of PCA-dVarimax: (1)  PCA-dVarimax-1,  the estimator defined in \eqref{def_Lambda_hat} with $\bD_{(r)}$ replaced by $\wh\bD_{(r)}$ in \eqref{def_D_Lambda},
	and (2) PCA-dVarimax-2, the estimator defined in \eqref{def_Lambda_hat} with $\bD_{(r)}$ replaced by $\wh\bD_{(r)}$ and $\wc\bQ$ obtained from applying \cref{alg_rotate} to $\wh\bU_{(r)}$ but with \eqref{iter_PGD_corrected} in place of \eqref{iter_PGD}.
	

	We set  $n = 500$, $r = 5$ and  vary $p \in\{ 20, 30, 50, 100, 200, 300\}$, $\eps^2 \in \{0, 0.2, 0.4, 0.6, 0.8, 1\}$ one at a time. \cref{fig:improvement} depicts the averaged estimation errors of PCA-dVarimax and its two variants. From the left panel, we see that the improvement over PCA-dVarimax becomes visible only when $p$ is small or moderate. This is in line with \cref{thm_A_corr} and \cref{rem_optimality} from which we expect improvement only when $rp^2=o(n\eps^4)$ and $p^2 = o(n \eps^2)$ by using $\epsilon^2 \asymp \eps^2/p$ in our simulation setting. In this regime, the right panel further corroborates that the improvement gets more visible as $\eps^2$ increases. In \cref{fig:improvement}, we also see the benefit of using two steps of improvement over that of using one step only. These findings are in line with our Theorems \ref{thm_PCs_hat} \& \ref{thm_A_corr}. 
	
	\begin{figure}[ht]
		\centering
		\includegraphics[width=\textwidth]{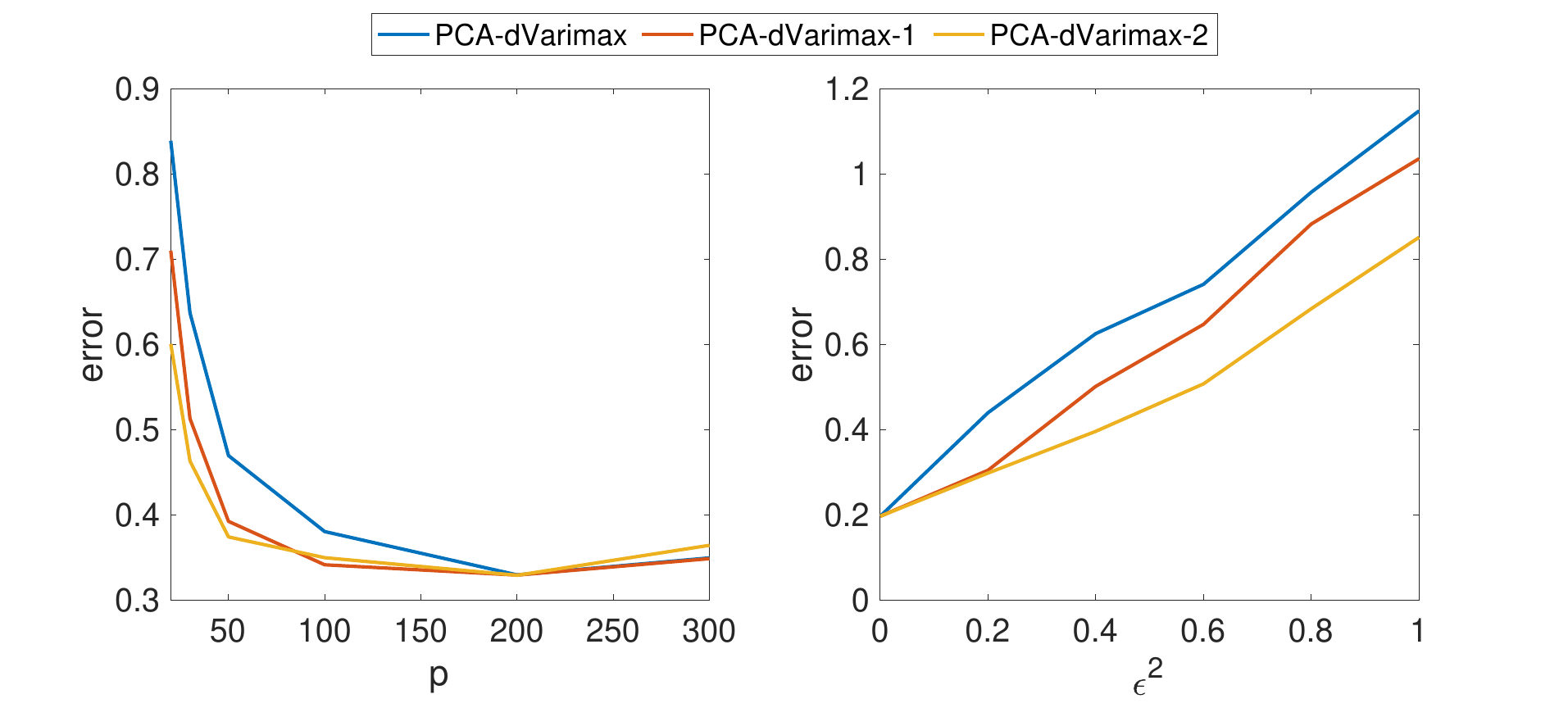}
		\caption{The averaged estimation errors of PCA-dVarimax, PCA-dVarimax-1, and PCA-dVarimax-2 with $p \in\{ 20, 30, 50, 100, 200, 300\}$ and  $\eps^2=0.6$  in the left panel, and $p=20$ and $\eps^2 \in \{0, 0.2, 0.4, 0.6, 0.8, 1\}$ in the right panel.}
		\label{fig:improvement}
	\end{figure}

		\section{Analysis of the deflation varimax procedure under a general setting}\label{app_sec_theory_varimax}
		
		In this section we analyze the proposed deflation varimax procedure under the following model: suppose one has access to the $r\times n$ matrix  $\wh \bY$ satisfies 
		\begin{equation}\label{model_Y_hat}
			\wh\bY = \bR^\T \left(\bY+ \bOmega\right)\qquad \text{with }\quad \bY  := \bA\bZ + \bN.
		\end{equation}
		where 
		\begin{itemize}
			\item $\bA$ is an $r\times r$ deterministic matrix with orthonormal columns;
			\item columns of $\bZ$ are i.i.d. and satisfy \cref{ass_Z} with $\sigma^2=1$;
			\item columns of $\bN$ are i.i.d. $\cN_r(0, \bSigma_N)$ with  $\|\bSigma_N\|_\op \le 1$;
			\item $\bR \in \bbO_{r\times r}$ and $\bOmega\in\R^{r\times n}$ are some generic matrices that could depend on $\bZ$ and $\bN$.
		\end{itemize}
		 As commented in \cref{rem_analysis}, the above model of $\wh\bY$  subsumes the decomposition of $\bU_{(r)}$ in \eqref{decom_Ur} as a special case.  Our goal in this section is to analyze the $r\times r$ rotation matrix  $\wc \bQ$, obtained by applying the deflation varimax procedure in \cref{alg_rotate} to $\wh\bY$ in lieu of $\bU_{(r)}$. 
		 
		 Formally,  we consider the following optimization problem analogous to \eqref{obj_L4},
		\begin{equation}\label{obj_L4_hat}
			\min_{\bq\in\Sp^r}  F(\bq; \wh \bY),\quad \text{with }\quad  F(\bq; \wh \bY) := -{1 \over 12 n}\|\wh\bY^\T \bq\|_4^4. 
		\end{equation} 
		Its oracle counterpart, similar to \eqref{obj_L4_tilde_main},  is
		\begin{equation}\label{obj_L4_tilde}
			\min_{q\in\Sp^r} F(\bq; \wt \bY)
		\end{equation} 
		where we write
		\begin{equation}\label{model_Y_tilde}
			\wt \bY := \bR \wh \bY =  \bY+\bOmega \in \R^{r\times n}.
		\end{equation}
		In view of  \eqref{model_Y_hat}, our results naturally depend on the approximation matrix $\bOmega$. To characterize its effect, let $\eta_n$ be any positive deterministic sequence and define the event 
		\begin{equation}\label{def_event_Omega}
			\cE_\Omega(\eta_n) := \left\{
			\sup_{\bq\in\Sp^r} {1\over n}\sum_{t=1}^n (\bq^\T\Omega_t)^4  \le ~ \eta_n^4
			\right\}.
		\end{equation}
		Define
		\begin{equation}\label{def_eta_n_prime}
			\eta_n' = C\left(\eta_n  +  \sigma_N^2+ \sqrt{r^2\log(n)\over n}\right),\quad 
            \bar\eta_n = C\left(\eta_n  +  \sigma_N^2+ \sqrt{r\log(n)\over n}\right)
		\end{equation} 
		for some large constant $C>0$, where 
		\begin{equation}\label{def_sigma2_N}
			\sigma_N^2	:=  \sup_{q \in \Sp^r} \|\bP_{q}^{\perp} \bSigma_N q\|_2.
		\end{equation}
		In this section we prove the following theorems.

		\begin{theorem}[\cref{thm_critical} under model \eqref{model_Y_hat}]\label{app_thm_critical}
					Grant model \eqref{model_Y_hat} and the event $\cE_\Omega(\eta_n)$. Assume $\kappa\ge c$ for some constant $c>0$ and  
					\begin{equation}\label{cond_n_r_app}
					\eta_n'  \le c'  
				\end{equation} 
				for some sufficiently small constant $c'>0$.
				Let $\bar\bq$ be  any stationary point to \eqref{obj_L4_tilde} such that 
				\begin{equation}\label{cond_critical_app} 
					(A_i^\T\bar\bq)^2 = \|\bA^\T \bar\bq\|_\infty^2  ~ \ge  1/3, \qquad | A_i^\T\bar\bq| - \max_{j\ne i}| A_j^\T \bar\bq| ~ \ge C'  \eta_n'
				\end{equation}
				holds for an arbitrary $i\in [r]$ and for some constant $C'>0$. Then with probability at least $1-n^{-1}$, one has
				\begin{equation}\label{rate_critical_app}
					 \min\Bigl\{\|\bar\bq -  A_i\|_2,\|\bar\bq +  A_i\|_2\Bigr\} ~   \lesssim  ~  \bar\eta_n.
				\end{equation}
			\end{theorem} 
		
		\begin{theorem}[Other results in \cref{sec_theory_alg}  under model \eqref{model_Y_hat}]\label{app_thm_deflation}
			Grant model \eqref{model_Y_hat} and the event $\cE_\Omega(\eta_n)$.
			With condition \eqref{cond_omega_n} replaced by 
			\[
				\eta_n' \le c'\mu\nu^2,
			\]
			all the statements in \cref{thm_one_col}, \cref{cor_A} and \cref{thm_RA} of \cref{sec_theory_alg} continue to hold for the corresponding quantities of \cref{alg_rotate} when applied to $\wh \bY$ in lieu of $\bU_{(r)}$.
		\end{theorem}

		In particular, when $\wh\bY =  \bU_{(r)}$,  \cref{thm_PCs} shows that 
		  the approximation matrix satisfies
		$\bOmega = \bDelta (\bA\bZ+\bN) = \bDelta \bY$ and $\bSigma_N$ satisfies \eqref{eq_Sigma_N}.
		Moreover,  by invoking the deviation inequalities in \eqref{bd_obj_L4} and \eqref{def_event_1} of the objective function, we find that with probability at least $1-n^{-1}$, 
		\[
		\sup_{\bq\in\Sp^r} {1\over 3n}\sum_{t=1}^n (\bq^\T \Omega_t)^4  \le \|\bDelta\|_\op^4 \sup_{\bq\in\Sp^p} {1\over 3 n} \sum_{t=1}^n (\bq^\T   Y_t)^4 \lesssim \|\bDelta\|_\op^4.
		\]
		On the event $\cE_{\Omega}(\eta_n)$, the following thus holds with probability at least $1-n^{-1}$,
		$$ 
		\eta_n \lesssim \|\bDelta\|_\op,\quad 
		\eta_n' \lesssim  \|\bDelta\|_\op + \sigma_N^2  + \sqrt{r^2\log(n)\over n},\quad \bar\eta_n \lesssim  \|\bDelta\|_\op + \sigma_N^2  + \sqrt{r\log(n)\over n}.$$
	Recall that $\|\bDelta\|_\op \lesssim \omega_n +\epsilon^2$  from \cref{thm_PCs}  and $\sigma_N^2 \le \epsilon^2/c_\Lambda^2$ from \eqref{eq_Sigma_N}. The results in  \cref{sec_theory_alg} follow from \cref{app_thm_critical} and \cref{app_thm_deflation}.

	 	\subsection{Improvement of the deflation varimax when $\Sigma_N$ can be estimated consistently}\label{app_sec_improv}
		
		Let $\wc\bQ$ be the corresponding estimator obtained from applying \cref{alg_rotate} to $\wh\bY$ with \eqref{iter_PGD_corrected} in place of the original PGD in \eqref{iter_PGD} (of course, we also substitute $\wh \bY$ for $\bU_{(r)}$).  The following theorem states its rate of convergence for estimating $\bR^\T \bA$ in the Frobenius norm. The rate naturally depends on how well $\wh\bSigma_N$ estimates $\bR^\T \bSigma_N \bR$. 
		For some deterministic sequence $\alpha_n > 0$, we define the event
		\[
		\cE_{\bSigma_N}(\alpha_n) := \left\{
		 \|
		\wh\bSigma_N - \bR^\T  \bSigma_N \bR
		 \|_\op \le \alpha_n
		\right\}.
		\]

		\begin{theorem}\label{thm_RA_hat}
				Grant model \eqref{model_Y_hat} and the event $\cE_\Omega(\eta_n)\cap \cE_{\bSigma_N}(\alpha_n)$. Assume $\kappa\ge c$ for some constant $c>0$. 	Fix any $0< \nu < \mu < 1$ such that 
				\begin{equation}\label{cond_eta_n_alpha_n}
					\eta_n + \alpha_n   + \sqrt{r^2\log(n)/ n}  ~ \le  c' \mu \nu^2
				\end{equation}
				holds for some sufficiently small constant $c'>0$. Choose $\gamma \le c''/\kappa$ for some sufficiently small constant $c''>0$.     Assume there exists some permutation $\pi: [r] \to [r]$ such that  the initialization $\bR \wh  Q_{\pi(i)}^{(0)}$ satisfies $\cE^{(i)}_{\init}(\mu,\nu)$ for all $i\in [r]$.  Then there exists some signed permutation matrix $\bP$ such that   with probability at least $1-n^{-1}$, 
			\begin{align*}
				\|\wc \bQ - \bR^\T\bA\mb P\|_{F}~  \lesssim ~  \left(
				\eta_n + \alpha_n + \sqrt{r\log(n)\over n}
				\right)\sqrt{r} .
			\end{align*}
		\end{theorem}
		\begin{proof}
			Its proof appears in \cref{app_proof_thm_RA_hat}.
		\end{proof}

		Essentially, we have the same result as in \cref{app_thm_critical} with $\sigma_N^2$ replaced by $\alpha_n$. Therefore, we expect improvement when $\alpha_n \ll \sigma_N^2$, reflecting the benefit of consistently estimating $\bSigma_N$.

		\subsection{Preliminaries to the proof of Theorems \ref{app_thm_critical} and \ref{app_thm_deflation}}\label{app_pre}
		
		\subsubsection{Expressions of the constrained $\ell_4$ optimization and its Riemannian gradient}
		
		For the following optimization problem 
		\begin{align}\label{obj_L4_Y}
			\min_{ q\in \mathbb{S}^{r}}~ F(\bq; \bY)
		\end{align}
		its Riemannian gradient over $\bq\in \Sp^r$  is
		\begin{align*}
			&\grad F(\bq; \bY) := \grad_{ q\in \Sp^r} F( q; \bY)= \bP_{\bq}^{\perp }\nabla F(\bq; \bY) 
		\end{align*}
		where $\nabla F(\bq; \bY)$ is the classical gradient of $F(\cdot; \bY)$ at any $\bq$. Further calculation gives that  
		\begin{align}\label{eq_grad_F}
			&\grad F(\bq; \bY) = -{1\over 3  n} \bP_{\bq}^\perp \sum_{t=1}^n \left(\bq^\T   Y_t\right)^3    Y_t.
		\end{align}
		To define the population level counterpart of $F(\cdot; \bY)$ and $\grad F(\cdot; \bY)$ under model \eqref{model_Y_hat}, using \cref{lem:obj_exp} gives that any $\bq\in \Sp^r$,
		\begin{align}\label{eq_f} \nonumber
				f(\bq)  := \EE[F(\bq; \bY)] &=   - {1\over 12}\EE\left[
				(\bq^\T Y_t)^4
				\right]\\
				&= -\frac14\left( \kappa  \|\bA^\T \bq \|_4^4 +   1  + 2   \norm{\bq}{\bSigma_N}^2 +   \norm{\bq}{\bSigma_N}^4\right)  
		\end{align}
		where we write  
		$
		\norm{\bq}{\bSigma_N}^2 := \bq^\T  \bSigma_N \bq.
		$
		The Riemannian gradient of $f( q)$ over $\bq\in \Sp^r$ is
		\begin{align}\label{eq_grad_f}
			\grad f( q)   &=  {1\over 3}\bP_{\bq}^\perp \EE[(\bq^\T  Y_t)^3   Y_t  ] = 
			\grad h(\bq) - \left( 1 +   \norm{\bq}{\bSigma_N}^2\right)    ~ \bP_{\bq}^\perp \bSigma_N  \bq 
		\end{align}
		with 
		\begin{align} \label{eq_grad_h}
			\grad h(\bq) &:= - \kappa ~   \bP_{\bq}^\perp  
			\bA (\bA ^\T \bq)^{\circ 3}.
		\end{align}
		Here for any $v\in \R^d$, $v^{\circ 3} = (v_1^3, \ldots, v_d^3)^\T$ represents the element-wise multiplication.
		As a result of \eqref{eq_grad_h}, we have the following identity that will be used later, 
		\begin{equation}\label{eq_grad_ai} 
			\innerprod{-\grad h( q)}{ A_i} =   \kappa   ( A_i^\T \bq) \left[( A_i^\T \bq)^2 -\|\bA^\T \bq\|_4^4\right],\quad \text{for any $i\in [r]$ and $q\in \Sp^r$}.
		\end{equation} 
		
		\subsubsection{Concentration inequalities of the objective function in \eqref{obj_L4_Y} and its Riemannian gradient}
		
		Our analysis will frequently use the following concentration inequalities between $F(\cdot ; \bY)$ and $f(\cdot )$ as well as those between $\grad F(\cdot ; \bY)$ and $\grad f(\cdot)$. 
		Consider the events
		\begin{align}\label{def_event_1}
			&\cE_F = \left\{
			\sup_{\bq\in \Sp^r} |F(\bq; \bY) - f(\bq)| ~ \lesssim \sqrt{r^2\log(n)\over n}
			\right\}\\\label{def_event_2}
			&\cE_{\grad} = \left\{
			\sup_{\bq\in \Sp^r} \|\grad F(\bq; \bY) - \grad f(\bq)\|_2 ~ \lesssim \sqrt{r^2\log(n)\over n}
			\right\}\\\label{def_event_grad_a}
			& \cE_{\grad,A} = \left\{\max_{k\in [r]}\|\grad F( A_k; \bY) - \grad f( A_k)\|_2 ~ \lesssim \sqrt{r\log(n)\over n}\right\}.
		\end{align}  
		By \eqref{eq_grad_F} and \eqref{eq_grad_f}, for any fixed $\bq \in \Sp^r$,   
		\begin{align*}
			\norm{\grad F(\bq; \bY) - \grad f(\bq)}2 &\le  {1\over   3n}\left\|\sum_{t=1}^n\left[ (\bq^\T  Y_t)^3   Y_t  - \EE[(\bq^\T  Y_t)^3   Y_t  ]\right] \right\|_2.
		\end{align*} 
		Invoking \cref{lem_grad_lip} (for $\cE_{\grad})$, \cref{lem_HH} (for $\cE_F$) and \cref{lem_bd_grad_pointwise} with $\bq = A_k$ (for $\cE_{\grad,A}$)  together with the union bounds over $k\in [r]$ ensures that there exists some large constant $C\ge 1$ such that 
		\[
			\P\left(
			\cE_F \cap \cE_{\grad}\cap \cE_{\grad,A}
			\right) \ge 1- n^{-C}
		\]
		holds under model \eqref{model_Y_hat} and $n\ge  r(r + \log n)$,

		\subsection{Proof of \cref{lem_connection}}\label{app_proof_lem_connection}
		
		We prove the results for  $\{\wh\bq^{(\ell)}\}_{\ell\ge 0}$ and  $\{\wt\bq^{(\ell)}\}_{\ell\ge 0}$ which are any PGD iterates  in \eqref{iter_PGD} for solving \eqref{obj_L4_hat} and \eqref{obj_L4_tilde}, respectively.
		
		\begin{proof}
			For any $\ell \ge 0$, note that
			\begin{align*}
				\wh\bq^{(\ell+1)} &= P_{\Sp^r}\left(
				\wh\bq^{(\ell)}  - \gamma\grad  F(\wh\bq^{(\ell)}; \wh\bY)
				\right)  =  {\wh\bq^{(\ell)} - \gamma\grad  F(\wh\bq^{(\ell)};\wh\bY)\over \sqrt{1 + \gamma^2 \|\grad  F(\wh\bq^{(\ell)};\wh\bY)\|_2^2}}  
			\end{align*}
			where the second equality uses 
			$$
				\|\wh\bq^{(\ell)} - \gamma\grad  F(\wh\bq^{(\ell)};\wh\bY)\|_2^2   = 1 + \gamma^2\|\grad   F(\wh\bq^{(\ell)};\wh\bY)\|_2^2.
			$$
			Similarly, 
			\begin{align*}
				\wt\bq^{(\ell+1)} 
				&= {\wt\bq^{(\ell)} - \gamma\grad   F(\wt\bq^{(\ell)};\wt\bY)\over \sqrt{1 + \gamma^2 \|\grad   F(\wt\bq^{(\ell)};\wt\bY)\|_2^2}}.
			\end{align*}
			Suppose $\wh\bq^{(\ell)} =  \bR^\T \wt\bq^{(\ell)}$ for any $\ell \ge 0$. It remains to prove 
			\[
			\grad  F(\wh\bq^{(\ell)};\wh\bY)  = \bR^\T \grad   F(\bq^{(\ell)};\wt\bY). 
			\]
			To this end, recalling  \eqref{model_Y_tilde} and \eqref{eq_grad_F}, we have that 
			\begin{align*}
				\bR^\T \grad   F(\bq^{(\ell)};\wt\bY) &=  {1\over 3n} \sum_{t=1}^n \bR^\T \bP_{\wt\bq^{(\ell)}}^\perp  \bR \wh Y_t \left(\wt\bq^{(\ell) \T}  \bR \wh Y_t\right)^3  \\
				&=  {1\over 3n} \sum_{t=1}^n   \bP_{\wh\bq^{(\ell)}}^\perp   \wh Y_t \left(\wh\bq^{(\ell)\T}  \wh Y_t\right)^3\\
				&  = \grad  F(\wh\bq^{(\ell)};\wh\bY)
			\end{align*}
			where in the second line we used $\wh\bq^{(\ell)} =  \bR^\T \wt\bq^{(\ell)}$ and 
			\begin{align}\label{ident_P_perp}
				\bR^\T \bP_{\wt\bq^{(\ell)}}^\perp  \bR = \bI_r -  \bR^\T\wt\bq^{(\ell)} \wt\bq^{(\ell)\T} \bR = \bI_r - \wh\bq^{(\ell)}\wh\bq^{(\ell)\T} = \bP_{\wh\bq^{(\ell)}}^\perp.
			\end{align}
			The proof is complete.
		\end{proof}

		\subsection{Proof of \cref{thm_critical} \& \cref{app_thm_critical}: the optimization landscape analysis}\label{app_proof_thm_critical}

		Recall that we only need to prove \cref{app_thm_critical}  stated in \cref{app_sec_theory_varimax} which is more general than \cref{thm_critical}.

		\begin{proof}
			Recall  $\grad h(\cdot)$ and $\grad F(\cdot)$ from \eqref{eq_grad_h} and \eqref{eq_grad_F}, respectively.   For any stationary point $\bq$ to \eqref{obj_L4_tilde}, we must have  
			\begin{align*}
				0 = -\grad F(\bq; \wt \bY) = \grad h( q)-\grad F(\bq; \wt \bY)-\grad h( q).
			\end{align*} 
			Let $\bzeta = \bA^\T \bq$ and recall \eqref{eq_grad_ai}. Pick any  $j\in[r]$.  Multiplying both sides of the above identity by $ A_j$ yields
			\begin{align*}
				0 & ~ =~   \kappa ~ \zeta_j   \left(\zeta_j^2-\|\bzeta\|_4^4\right)      +  A_j^\T\left(\grad h( q) - \grad F(\bq; \wt \bY)\right)
			\end{align*}
			so that by rearranging terms we can write 
			\[
			\zeta_j^3- \alpha \zeta_j+ \beta_j=0
			\]
			where 			\begin{equation}\label{eq_alpha_beta_general}
				\alpha = \norm{\bzeta}4^4, \qquad \beta_j  = {1\over \kappa} A_j^\T\left(\grad h( q)-\grad F(\bq; \wt \bY)\right).
			\end{equation}  
			In the rest of the proof we work on the event 
			\begin{equation}\label{def_event_bar}
				\bar \cE = \cE_{\Omega}(\eta_n) \cap  \cE_F\cap \cE_{\grad}\cap \cE_{\grad,A}
			\end{equation}
			with $\cE_{\Omega}(\eta_n)$ given by \eqref{def_event_Omega}. Let $i\in [r]$ be the index given by \eqref{cond_critical}. Without loss of generality, we assume $\zeta_i > 0$ as we may replace $ A_i$ by $- A_i$ otherwise. For future reference, under the condition $r^2\log(n) \le n$, we have that for any $\bq\in \Sp^r$,
			\begin{align}\nonumber
				&\|\grad  F(\bq; \wt\bY) - \grad h(\bq)\|_2\\\nonumber   
                &~\lesssim   \eta_n  +     \sigma_N^2 + \|\grad F(\bq; \bY) - \grad f(\bq)\|_2 &&\text{by \cref{lem_bd_grad_F_tilde} and $\cE_F$}\\\label{def_event_tilde}
				&~\le  \bar \eta_n + C \|\bq-A_i\|_2 \sqrt{r^2\log(n)\over n}  &&\text{by \cref{lem_grad_lip}  and \eqref{def_eta_n_prime}}\\ \label{def_event_tilde_new} 
				&~ \le \eta_n'.
			\end{align}  
			Furthermore, the event \eqref{event_grad} in \cref{lem_betai} holds under  condition $\eta_n' \le c \kappa$ for some sufficiently small constant $c>0$.

			The first inequality in \eqref{cond_critical} 
			ensures that  
			\begin{equation}\label{lb_alpha_general}
				\alpha  \ge \|\bzeta\|_\infty^4 = \zeta_i^4 \ge {1\over 9} 
			\end{equation}
            with $i$ defined in \eqref{cond_critical}. We also assume without loss of generality that $\zeta_i>0$. On the other hand, \cref{lem_betai} gives
			\begin{equation}\label{ub_beta_general}
				4|\beta_j|   <  \|\bzeta\|_4^6 =  \alpha^{3/2}
			\end{equation}
			as well as 
			\begin{align}\nonumber
				|\beta_j| &\le {1\over \kappa} \|\bq -  A_j\|_2 \|\grad F(\bq; \wt \bY) - \grad h(\bq)\|_2\\\label{ub_beta}
				&\lesssim {1\over \kappa}\|\bq -  A_j\|_2 \left( \bar\eta_n  +  \|\bq-A_i\|_2 \sqrt{r^2\log(n)\over n} \right) &&\text{by \eqref{def_event_tilde}}\\\label{ub_beta_1}
				&\le {2\over \kappa}\eta_n' &&\text{by \eqref{def_event_tilde_new}}.
			\end{align}
			Further note that the preceding displays and the condition 
			$\eta_n' \le   c\kappa$
			imply
			\begin{equation}\label{bd_beta_alpha_ratio}
			{2|\beta_j| \over  \alpha} \lesssim  {\eta_n' \over \kappa \alpha} < {1\over 3} \le  \sqrt{\alpha}.
			\end{equation}
			In conjunction with \eqref{lb_alpha_general} and by invoking \cref{lem:cubic_function}, one can deduce that the stationary point $q$ and its corresponding $\zeta$ must belong to one of the following two cases.
			\begin{itemize}
				\item \textbf{Case \rom{1}}: For the $i\in [r]$ defined in \eqref{cond_critical}, 
				\[
				    |\zeta_i| \in \left[\sqrt{\alpha} - {2|\beta_i|\over \alpha},\; \sqrt{\alpha} + {2|\beta_i|\over \alpha}\right] ,~ \quad |\zeta_j| 
				<  {2|\beta_j|\over \alpha},\quad \forall j\in [r]\setminus \{i\};
				\]
				\item \textbf{Case \rom{2}}: There exists at least one $j\in [r]\setminus\{i\}$ such that \label{cas:finite_case_3}
				\begin{equation*}
					|\zeta_i| \in \left[\sqrt{\alpha} - {2|\beta_i|\over \alpha},\; \sqrt{\alpha} + {2|\beta_i|\over \alpha}\right] ,~ \quad|\zeta_j| \in \left[\sqrt{\alpha} - {2|\beta_j|\over \alpha},\; \sqrt{\alpha} + {2|\beta_j|\over \alpha}\right].
				\end{equation*}
			\end{itemize}
			In fact we argue $\bq$ can only belong to \textbf{Case \rom{1}}. Indeed, suppose that $\bq$ belongs to \textbf{Case \rom{2}}, then using the definition of   \textbf{Case \rom{2}} and \eqref{bd_beta_alpha_ratio} yields
			\begin{align*}
				|\zeta_i|-|\zeta_j|  &\le {2|\beta_i| \over \alpha} + {2|\beta_j| \over \alpha} \le {C''\over  \kappa} \eta_n',
			\end{align*}
			which contradicts with the second display in \eqref{cond_critical} whenever $C'> C''/\kappa$.  We proceed to derive the upper bound of $\|\bq- A_i\|_2$ for \textbf{Case \rom{1}}.
			
			On the one hand, we bound $\zeta_i^4$ from below as 
			\begin{align}\label{ieq:zeta_inf_lower}
				\zeta_i^4= \norm{\mb \zeta}4^4-\sum_{j\neq i}\zeta_j^4
				&\geq\norm{\mb \zeta}4^4-\sum_{j\neq i}\zeta_j^2 \max_{j\neq i}\zeta_j^2\nonumber\\
				&\geq  \alpha -  \max_{j\ne i} {4\beta_j^2 \over \alpha^2} &&\text{by $\norm{\mb \zeta}2^2 = 1$ and definition of {\bf Case \rom{1}}}.
			\end{align}
			On the other hand, 
			\begin{align}\label{ieq:zeta_inf_upper}
				\zeta_i^2\leq \paren{\sqrt{\alpha}+\frac{2|\beta_i|}{\alpha}}^2
				= \alpha +\frac{4\beta_i^2}{\alpha^2} + {4|\beta_i| \over \sqrt \alpha} =  \alpha\left(
				1 + \frac{4\beta_i^2}{  \alpha^3} + {4|\beta_i| \over    \alpha^{3/2}}
				\right).
			\end{align}
			Combine (\ref{ieq:zeta_inf_lower}) and (\ref{ieq:zeta_inf_upper}) to obtain
			\begin{align*}
				\zeta_i^2=\frac{\zeta_i^4}{\zeta_i^2} &\geq {1-\max_{j\ne i}\frac{4\beta_j^2}{ \alpha^3}  \over 1 +\frac{4\beta_i^2}{  \alpha^3}  +{4|\beta_i| \over   \alpha^{3/2}}} =  1 - {  \max_{j\ne i}\frac{4\beta_j^2}{  \alpha^3}+\frac{4\beta_i^2}{  \alpha^3} + {4|\beta_i| \over    \alpha^{3/2}}  \over 1 +\frac{4\beta_i^2}{  \alpha^3}  +{4|\beta_i| \over    \alpha^{3/2}}} 
			\end{align*}
			which, together with    \eqref{lb_alpha_general}, 
			\eqref{ub_beta} and \eqref{ub_beta_1}, implies
			\begin{align*}
				\|q-A_i\|_2^2   &=2(1-\zeta_i)\\ &\leq   \frac{8\beta_i^2}{  \alpha^3} + \max_{j\ne i} \frac{8\beta_j^2}{  \alpha^3}  + {8|\beta_i| \over    \alpha^{3/2}}\\
				& \le \|q-A_i\|_2^2 ~   {C\over \kappa^2}\left(\eta_n' 
				\right)^2+ {C\over \kappa^2}\left( \bar \eta_n^2 + \|q-A_i\|_2^2 ~ {r^2\log(n)\over n}\right)\\
                &\quad +\|q-A_i\|_2 ~ {C\over \kappa}\bar\eta_n  + \|q-A_i\|_2^2 ~  {C\over \kappa}  \sqrt{r^2\log(n)\over n}.
			\end{align*}
		    Rearranging terms and using $\eta_n' \le c\kappa$ and $r^2\log(n) \le c n$ yield $\|q - A_i\|_2 \lesssim \bar \eta_n$, thereby completing  the proof. 
		\end{proof}

		\subsubsection{Lemmas used in the proof of \cref{app_thm_critical}}

		For any $\bq\in\Sp^r$, the following lemma bounds from above $\|\grad F(\bq; \wt \bY) - \grad h(\bq)\|_2$ by   quantities related with $|F(\bq;\bY)-f(\bq)|$ and $\|\grad F(\bq;\bY) - \grad f(\bq)\|_2$.

		\begin{lemma}\label{lem_bd_grad_F_tilde}
			Grant conditions in \cref{app_thm_critical}. On the event $\cE_{\Omega}(\eta_n)$ with some $\eta_n \le 1$,  for any $\bq\in\Sp^r$, one has
			\begin{align*}
				\|\grad F(\bq; \wt \bY) - \grad h(\bq)\|_2  ~ &\lesssim ~  \eta_n\left( 1 +  \sup_{\bq'\in \Sp^r} \left|F(\bq';\bY) - f(\bq')\right |\right) +   \sigma_N^2 \\
				&\qquad + \|\grad F(\bq;\bY) - \grad f(\bq)\|_2.
			\end{align*}
		\end{lemma}
		\begin{proof} 
			Pick any $\bq\in\Sp^r$. 
			By triangle inequality
			\begin{align*}
				\|\grad F(\bq; \wt \bY) - \grad h(\bq)\|_2   &\le  \|\grad F(\bq; \wt \bY) - \grad F(\bq;\bY)\|_2 \\
				&\quad  +  \|\grad   F(\bq;\bY) - \grad f(\bq)\|_2 +     \left(1 +        \|\bq\|_{\bSigma_N}^2   \right) \|\bP_{\bq}^{\perp}  \bSigma_N \bq\|_2.
			\end{align*}
			Since  $\|\bSigma_N\|_\op \le  1$, by recalling  \eqref{def_sigma2_N}, we have
			\[
			 \left(1 +      \|\bq\|_{\bSigma_N}^2   \right) \|\bP_{\bq}^{\perp}  \bSigma_N \bq\|_2 \le  2\sigma_N^2.
			\]
			It remains to bound from above $\|\grad F(\bq; \wt \bY) - \grad F(\bq;\bY)\|_2$ on the event $\cE_\Omega(\eta_n)$. Recalling \eqref{eq_grad_F} and 
			$
			\wt\bY = \bR \wh\bY  = \bY + \bOmega,
			$
			we find that
			\begin{align*}
				&\grad F(\bq; \wt \bY) - \grad F(\bq;\bY)\\ &=   -{1\over 3  n} \bP_{\bq}^\perp \sum_{t=1}^n (\bq^\T \wt Y_t)^3 \wt Y_t + {1\over 3 n} \bP_{\bq}^\perp \sum_{t=1}^n (\bq^\T  Y_t)^3  Y_t\\
				&=  -{1\over 3 n} \bP_{\bq}^\perp  \sum_{t=1}^n 
				\left\{
				 Y_t(\bq^\T\Omega_t)^3 +  3 Y_t(\bq^\T \Omega_t)^2 \bq^\T Y_t  +  3 Y_t(\bq^\T\Omega_t) (\bq^\T Y_t)^2 + \Omega_t (\bq^\T \wt Y_t)^3
				\right\}.
			\end{align*}
			It suffices to bound from above each term below:
			\begin{align*}
				\Gamma_1 &= {1\over 3 n} \sup_{\bq\in\Sp^r} \left\|
				\bP_{\bq}^\perp \sum_{t=1}^n   Y_t(\bq^\T\Omega_t)^3 \right\|_2,\\
				\Gamma_2 &= {1\over 3 n} \sup_{\bq\in\Sp^r} \left\|
				\bP_{\bq}^\perp \sum_{t=1}^n  Y_t(\bq^\T\Omega_t) (\bq^\T Y_t)^2 \right\|_2,\\
				\Gamma_3 &= {1\over 3 n} \sup_{\bq\in\Sp^r} \left\|
				\bP_{\bq}^\perp \sum_{t=1}^n  \Omega_t (\bq^\T \wt Y_t)^3\right\|_2.
			\end{align*} 
			Regarding $\Gamma_1$, we have that 
			\begin{align*}
				\Gamma_1 &= {1\over 3 n}\sup_{\bar\bq,\bq\in\Sp^r}  
				\sum_{t=1}^n  \bar\bq^\T \bP_{\bq}^\perp   Y_t(\bq^\T\Omega_t)^3 \\ 
				&\le  {1\over 3 n}\sup_{\bar\bq,\bq\in\Sp^r}  
				\sum_{t=1}^n  \bar\bq^\T  Y_t(\bq^\T\Omega_t)^3
				+ {1\over 3 n}\sup_{\bar\bq,\bq\in\Sp^r}  
				\sum_{t=1}^n  \bar\bq^\T \bq \bq^\T Y_t(\bq^\T\Omega_t)^3\\
				&\le {1\over 3 n}\sup_{\bar\bq,\bq\in\Sp^r}  
				\sum_{t=1}^n  \bar\bq^\T  Y_t(\bq^\T\Omega_t)^3
				+ {1\over 3 n}\sup_{\bar\bq,\bq\in\Sp^r}  
				\sum_{t=1}^n   \bq^\T Y_t(\bq^\T\Omega_t)^3\\
				&\le  {2\over 3 n}  \sup_{\bq\in\Sp^r} 
				\left(\sum_{t=1}^n    (\bq^\T  Y_t)^4\right)^{1/4} \sup_{\bq\in\Sp^r} \left(\sum_{t=1}^n (\bq^\T \Omega_t)^4\right)^{3/4} &&\text{by H\"{o}lder's inequality}\\
				&\le  2  \sup_{\bq\in\Sp^r} 
				\left({1\over 3 n}\sum_{t=1}^n    (\bq^\T  Y_t)^4\right)^{1/4}  \eta_n^3 &&\text{by $\cE_{\Omega}(\eta_n)$}.
			\end{align*}
			Regarding $\Gamma_2$, similar arguments yield
			\begin{align*}
				\Gamma_2 & = {1\over 3 n} \sup_{\bar\bq,\bq\in\Sp^r}  
				\sum_{t=1}^n \bar\bq^\T\bP_{\bq}^\perp   Y_t(\bq^\T\Omega_t) (\bq^\T Y_t)^2\\
				& \le {1\over 3 n}\left\{ \sup_{\bar\bq,\bq\in\Sp^r}  
				\sum_{t=1}^n \bar\bq^\T  Y_t(\bq^\T\Omega_t) (\bq^\T Y_t)^2 + \sup_{\bq\in\Sp^r}  
				\sum_{t=1}^n  (\bq^\T\Omega_t) (\bq^\T Y_t)^3\right\}\\
				&\le {1\over 3 n}\left\{\sup_{\bar\bq,\bq\in\Sp^r}  
				\left(\sum_{t=1}^n (\bar\bq^\T  Y_t)^{4/3}(\bq^\T Y_t)^{8/3}\right)^{3/4} + \sup_{\bq\in\Sp^r}  \left(\sum_{t=1}^n    (\bq^\T  Y_t)^4\right)^{3/4}\right\}\sup_{\bq\in\Sp^r} \left(\sum_{t=1}^n (\bq^\T \Omega_t)^4\right)^{1/4} \\
				&\le 2  \sup_{\bq\in\Sp^r} 
				\left({1\over 3 n}\sum_{t=1}^n    (\bq^\T  Y_t)^4\right)^{3/4} \eta_n.
			\end{align*}
			Moreover, we can bound from above $\Gamma_3$ by 
			\begin{align*}
				\Gamma_3 &= {1\over 3 n} \sup_{\bar\bq,\bq\in\Sp^r}  
				\sum_{t=1}^n \bar\bq^\T \bP_{\bq}^\perp  \Omega_t \left[\bq^\T ( Y_t + \Omega_t)\right]^3\\
				&\le {1\over 3 n} \left\{\sup_{\bar\bq,\bq\in\Sp^r}  
				\sum_{t=1}^n \bar\bq^\T \Omega_t \left[\bq^\T  ( Y_t + \Omega_t)\right]^3 +  \sup_{\bq\in\Sp^r} \sum_{t=1}^n  \bq^\T \Omega_t \left[\bq^\T  ( Y_t + \Omega_t)\right]^3\right\}\\
				&\lesssim {1\over 3 n} \left\{\sup_{\bar\bq,\bq\in\Sp^r}  
				\sum_{t=1}^n \bar\bq^\T \Omega_t \left[(\bq^\T   Y_t)^3 + (\bq^\T  \Omega_t)^3\right] +  \sup_{\bq\in\Sp^r} \sum_{t=1}^n  \bq^\T \Omega_t \left[(\bq^\T  Y_t)^3+ (\bq^\T\Omega_t)^3\right]\right\}\\
				&\lesssim {1\over 3 n} \left\{
				\sup_{\bq\in\Sp^r} \sum_{t=1}^n (\bq^\T \Omega_t)^4 + \sup_{\bq\in\Sp^r} \left(\sum_{t=1}^n (\bq^\T  Y_t)^4\right)^{3/4} \sup_{\bq\in\Sp^r} 
				\left(\sum_{t=1}^n    (\bq^\T \Omega_t)^4\right)^{1/4}
				\right\}\\
				&\lesssim  \eta_n^4 +   \sup_{\bq\in\Sp^r} 
				\left({1\over 3 n}\sum_{t=1}^n    (\bq^\T  Y_t)^4\right)^{3/4} \eta_n
			\end{align*}
			Finally, recall from \eqref{eq_f} that, for all $\bq\in \Sp^r$,
			\begin{align*}
				4|f(\bq) | = {1\over 3}\EE\left[  (\bq^\T  Y_t)^4\right] &= 
				\kappa \|\bA^\T\bq\|_4^4 + 1  + 2   \norm{\bq}{\bSigma_N}^2  +   \norm{\bq}{\bSigma_N}^4 \le \kappa + 4.
			\end{align*} 
			It then follows by using $\eta_n \le 1$ and noting that  
			\begin{equation}\label{bd_obj_L4}
				\begin{split}
					{1\over 3 n} \sup_{\bq\in\Sp^r}  
					\sum_{t=1}^n    (\bq^\T  Y_t)^4 &\le \sup_{\bq\in\Sp^r} f(\bq)  + {1\over 3 n}  \sup_{\bq\in\Sp^r}  \abs{ \sum_{t=1}^n  \left[  (\bq^\T  Y_t)^4 - \EE[  (\bq^\T  Y_t)^4]\right]}\\
					&\le {\kappa\over 4}+1  + 4\sup_{\bq\in\Sp^r} |F(\bq;\bY) - f(\bq)|. 
				\end{split} 
			\end{equation} 
			This completes the proof.
		\end{proof}

		The following  lemma  states properties of  $\beta_j$ defined in  \eqref{eq_alpha_beta_general}.

		\begin{lemma}\label{lem_betai}
			Grant conditions in \cref{app_thm_critical}. 
			For any  $\bq\in \Sp^r$ satisfying $\|\bzeta\|_\infty^2 \ge 1/3$ with $\bzeta = \bA^\T \bq$, on the event
			\begin{equation}\label{event_grad}
				\left\{ \|\grad h\paren{ q}-\grad F(q; \wt \bY)\|_2   \le {\kappa\over 216}
				\right\},
			\end{equation}
			the quantity $\beta_j$ defined in \eqref{eq_alpha_beta_general} satisfies: for any $j\in [r]$, 
			\begin{itemize}
				\item[(1)] $4|\beta_j|  <  \|\bzeta\|_4^{6}$;
				\item[(2)] $|\beta_j|  ~ \le    \min\left\{\| A_j+\bq\|_2, \| A_j-\bq\|_2\right\}  {1\over \kappa}\|\grad h\paren{ q}-\grad F(q; \wt \bY)\|_2$.
			\end{itemize}
		\end{lemma}
		\begin{proof}
			Without loss of generality, assume $\zeta_j =  A_j^\T \bq > 0$.
			By the fact that  the gradients of both $\grad h(\bq)$ and $\grad F(\bq; \wt \bY)$ lie in $\mb P_{\bq}^\perp$, we find 
			\begin{equation}\label{disp_beta_upper_bound}
				\begin{split}
					|\beta_j| &= {1\over \kappa}\abs{ A_j^\T (\grad h\paren{ q}-\grad F(q; \wt \bY))}   \\
					&\le 
					{1\over \kappa}\| A_j - \bq\|_2  \|\grad h\paren{ q}-\grad F(q; \wt \bY)\|_2.
				\end{split}
			\end{equation}
			This proves the claim in part (2). For part (1),  by \eqref{lb_alpha_general} and \eqref{event_grad}, we find
			\[
			\frac{|\beta_j|}{\|\bzeta\|_4^{6}}  \leq  {54   \over   \kappa} \|\grad h\paren{ q}-\grad F(q; \wt \bY)\|_2   < {1\over 4},
			\]
			completing the proof. 
		\end{proof}

		\begin{lemma}[Lemma $B.3$, \cite{qu2020geometric}]\label{lem:cubic_function}
			Consider the cubic function
			\begin{align*}
				f(x)=x^3-\alpha x+\beta.
			\end{align*}
			When $\alpha > 0$ and $4|\beta| \leq \alpha^{3/2}$, the roots of $f(x)=0$ belong to the following set
			\begin{align*}
				\Brac{ |x|\leq \frac{2|\beta|}{\alpha}}\bigcup\Brac{ |x-\sqrt{\alpha}|\leq \frac{2|\beta|}{\alpha}}\bigcup\Brac{ |x+\sqrt{\alpha}|\leq \frac{2|\beta|}{\alpha}}.
			\end{align*}
		\end{lemma}

		\subsection{A deterministic analysis of the PGD iterate for solving \eqref{obj_L4_tilde}}\label{app_sec_PGD}
		We work on the event $\bar{\cE}$ given by \eqref{def_event_bar}, and once on this event, all the results in this section are deterministic.  	Recall $\grad F(\cdot; \wt\bY)$ from \eqref{eq_grad_F} and $\grad h(\cdot)$  from  \eqref{eq_grad_h}. 
		From \eqref{def_event_tilde}, we know that on the event $\bar \cE$, for any $\bq \in\Sp^r$
		\begin{equation}\label{bd_gradient_app}
			 \|\grad F(\bq; \wt \bY) - \grad h(\bq)\|_2  ~ \le   ~ \bar \eta_n + \min_{i\in [r]}\|\bq - A_i\|_2\sqrt{r^2\log(n)\over n} \le \eta_n' 
		\end{equation}
		with $\eta_n'$ and $\bar\eta_n$ given by \eqref{def_eta_n_prime}. 
		Let $\{\bq^{(\ell)}\}_{\ell \ge 0}$ be any PGD iterates in \eqref{iter_PGD} from solving \eqref{obj_L4_tilde}.

		\begin{lemma}\label{lem_iter_gap}
			Grant model \eqref{model_Y_hat}. For any $0<\nu < \mu < 1$, assume there exists some sufficiently small constant $c>0$ such that
			\begin{equation}\label{cond_omega_td_epsilon}
				\eta_n'   \le c \kappa ~ \mu \nu^2.
			\end{equation}
			Let $\gamma \le c/\kappa$. Then for any $i\in [r]$ and any $\ell \ge 0$, if 
			\begin{equation}\label{cond_init_app}
				\abs{ A_i^\T \bq^{(\ell)}} \ge  \mu,  \quad \abs{ A_i^\T \bq^{(\ell)}} - \max_{j\ne i}\abs{ A_j^\T \bq^{(\ell)}} \ge \min\left\{\nu, ~ {1\over 4}\right\}
			\end{equation}
			holds,
			then on the event $\bar{\cE}$,   we have 
			$$
			\abs{ A_i^\T \bq^{(\ell+1)}} - \max_{j\ne i}\abs{ A_i^\T \bq^{(\ell+1)}} \ge   	\min\left\{\nu, ~ {1\over 4}\right\}.
			$$
		\end{lemma}
		\begin{proof}
			The proof appears in \cref{app_proof_lem_iter_gap}.
		\end{proof}

		\begin{lemma}\label{lem_iter_sup}
			Grant conditions in \cref{lem_iter_gap}. For any fixed iteration $\ell\ge 0$ and   $i\in [r]$ such that \eqref{cond_init_app} holds. Then  on the event $\bar\cE$, the following statements hold:
			\begin{enumerate}
				\item[(a)] if $\mu \le | A_i^\T\bq^{(\ell)}| \le \sqrt{3}/2$, then, for some constant $C>0$,
				\[
				{| A_i^\T\bq^{(\ell+1)}|  \over | A_i^\T\bq^{(\ell)}|} \ge  1  + C \gamma\kappa ~ \mu \nu.
				\]
				\item[(b)] if $\sqrt{3}/2 \le | A_i^\T\bq^{(\ell)}| \le 1$, 
					 then there exists some constant $C'>0$ such that, for any $\ell' \ge 1$,
					\[
					\min\left\{\bigl\|A_i -   \bq^{(\ell+\ell')}\bigr\|_2, ~ \bigl\|A_i +  \bq^{(\ell+\ell')}\bigr\|_2\right\} ~ \le  ~  C'{\bar\eta_n \over \kappa} + \left(
					1 - C' \kappa\gamma
					\right)^{\ell'/2}
					\] 
					and 
					 \[
						\bigl|A_i^\T q^{(\ell+\ell')}\bigr|  ~\ge ~ 1 -C'\left(\bar\eta_n\over \kappa\right)^2- \left(1 - C'\kappa \gamma \right)^{\ell'}.
					\] 
			\end{enumerate} 
		\end{lemma}
		\begin{proof}
			The proof appears in \cref{app_proof_lem_iter_sup}.
		\end{proof}

		The following lemma bounds from above the $\ell_2$-norm of $\grad F(\bq; \wt \bY)$ uniformly over $\bq\in \Sp^r$.

		\begin{lemma}\label{lem_grad_F_norm}
			Grant model \eqref{model_Y_hat}. On the event $\bar\cE$, the following holds uniformly over $\bq\in \Sp^r$,
			\begin{align*}
			 \|\grad F(\bq; \wt \bY) \|_2  &\le  \kappa \left(\|\bA^\T \bq\|_\infty^2  - \|\bA^\T \bq\|_4^4 \right)  + \bar\eta_n + \min_{i\in [r]}\|\bq - A_i\|_2\sqrt{r^2\log(n)\over n}\\
                &\le  \kappa \left(\|\bA^\T \bq\|_\infty^2  - \|\bA^\T \bq\|_4^4 \right)  + \eta_n'.
			\end{align*} 
		\end{lemma}
		\begin{proof}
			Pick any $\bq\in\Sp^r$ and write $\bzeta = \bA^\T \bq$. On the event $\bar\cE$, 
			\begin{align*}
				&\|\grad F(\bq; \wt \bY)\|_2\\
                &\le \norm{\grad h(\bq)}2 +  \|\grad F(\bq; \wt \bY) - \grad h(\bq)\|_2   \\
				&\le \kappa\| \bP_{\bq}^{\perp} \bA (\bA^\T \bq)^{\circ3}\|_2    +  \bar \eta_n + \min_{i\in [r]}\|\bq - A_i\|_2\sqrt{r^2\log(n)\over n}  &&\text{by \eqref{eq_grad_h} and \eqref{bd_gradient_app}}.
			\end{align*}
			Since 
			\begin{align*}
				\| \bP_{\bq}^{\perp} \bA (\bA^\T \bq)^{\circ3}\|_2 & = \| \bA (\bA^\T \bq)^{\circ3}\|_2 - \| \bq \bq^\T \bA (\bA^\T \bq)^{\circ3}\|_2\\
				&= \| \bzeta^{\circ3}\|_2 -  \|\bzeta\|_4^4\\
				&\le \|\bzeta\|_\infty^2  - \|\bzeta\|_4^4 &&\text{by }\|\bzeta\|_2\le 1,
			\end{align*}		
			we conclude that
			\begin{align*} 
				\|\grad F(\bq; \wt \bY)\|_2& \leq   \kappa \left(\|\bzeta\|_\infty^2  - \|\bzeta\|_4^4 \right)  + \bar\eta_n + \min_{i\in [r]}\|\bq - A_i\|_2\sqrt{r^2\log(n)\over n}.
			\end{align*}
			Since the above holds uniformly over $\bq\in\Sp^r$, the proof is complete.
		\end{proof}

		\subsubsection{Proof of \cref{lem_iter_gap}}\label{app_proof_lem_iter_gap}
		\begin{proof}
			To simplify notation,  we write 
			$$
			\bzeta^{(\ell)} = \bA^\T \bq^{(\ell)},\quad \text{for all }\ell \ge 0.
			$$ 
			Fix any $\ell \ge 0$ and consider two cases depending on the value of $|\zeta_i^{(\ell)}|$.    
			Pick any $j\in [r] \setminus \{i\}$. 
			Without loss of generality, we may assume both $\zeta_i^{(\ell)}$ and $\zeta_j^{(\ell)} $ are positive.   
			
			{\bf Case (a):}  When $\mu^2 \ge 3/4$, it suffices to consider
			\begin{align}\label{case:lem:grad_away:2}
				[\zeta_i^{(\ell)}]^2  ~\in~ \left[{3\over 4}, ~ 1\right].
			\end{align}
			The update in \eqref{iter_PGD} yields
			\begin{align}\label{eq_diff_zeta}\nonumber
				\zeta_i^{(\ell+1)} - \zeta_j^{(\ell+1)}   &=
				{  A_i^\T\bq^{(\ell)}- \gamma A_i^\T\grad  F(\bq^{(\ell)}; \wt \bY)   -   A_j^\T\bq^{(\ell)} + \gamma A_j^\T\grad  F(\bq^{(\ell)}; \wt \bY) 
					\over  \|\bq^{(\ell)} - \gamma\grad  F(\bq^{(\ell)}; \wt \bY)\|_2}\\
				&= {\zeta_i^{(\ell)}-\zeta_j^{(\ell)}- \gamma( A_i- A_j)^\T
					\grad  F(\bq^{(\ell)}; \wt \bY)  
					\over  \|\bq^{(\ell)} - \gamma\grad  F(\bq^{(\ell)}; \wt \bY)\|_2}.
			\end{align} 
			On the event $\bar\cE$,  observe that 
			\begin{align}\label{eq_gradient_update}\nonumber
				\|\bq^{(\ell)} - \gamma\grad  F(\bq^{(\ell)}; \wt \bY)\|_2^2  &= 1 + \gamma^2\|\grad  F(\bq^{(\ell)}; \wt \bY)\|_2^2 \\\nonumber
				& \le 1 + 2\gamma^2(\kappa^2 + (\eta_n')^2)&&\text{by \cref{lem_grad_F_norm}}\\
				& \le 1+c &&\text{by $\gamma < c/\kappa$ and \eqref{cond_omega_td_epsilon}}.
			\end{align}
			We obtain
			\begin{align*}
				\zeta_i^{(\ell+1)}-\zeta_j^{(\ell+1)}
				&\ge  {1\over \sqrt{1+c}} \left(\zeta_i^{(\ell)} -\zeta_j^{(\ell)}-2\gamma  ~ \|\grad  F(\bq^{(\ell)}; \wt \bY)\|_2\right)\nonumber\\
				&\geq  {1\over \sqrt{1+c}}\left({\sqrt{3}-1 \over 2}- 2\gamma(\kappa + \eta_n')  \right) &&\text{by \cref{lem_grad_F_norm}} \nonumber\\
				&\geq  {1\over 4} &&\text{by $\gamma < c/\kappa$ and \eqref{cond_omega_td_epsilon}}.
			\end{align*}
			Here the penultimate step also uses \eqref{case:lem:grad_away:2} and  
			$$
			\zeta_i^{(\ell)} -\zeta_j^{(\ell)}\geq  \zeta_i^{(\ell)} - \sqrt{\|\bzeta^{(\ell)}\|_2^2 - [\zeta_i^{(\ell)}]^2} \ge \zeta_i^{(\ell)} - \sqrt{1 - [\zeta_i^{(\ell)}]^2} \ge  {\sqrt{3}-1 \over 2}.
			$$
			\smallskip

			{\bf Case (b):} When $\mu^2 \le 3/4$, it suffices to consider
			\begin{align}\label{case:lem:grad_away:1}
				[\zeta_i^{(\ell)}]^2 ~\in~\left[   \mu^2 , ~   {3\over 4} \right].
			\end{align} 
			By writing 
			\[
			\bar{\nu} :=  \nu \wedge (1/4),
			\]
			we proceed to show that
			\begin{align}\label{def_beta_bar}
				\zeta_i^{(\ell+1)} - \zeta_j^{(\ell+1)} \geq \zeta_i^{(\ell)} - \zeta_j^{(\ell)} \overset{\eqref{cond_init_app}}{\ge}  \bar{\nu}.
			\end{align}
			From \eqref{eq_gradient_update}, since the right hand side of \eqref{eq_diff_zeta} equals 
			\begin{align*}
				&\zeta_i^{(\ell)}-\zeta_j^{(\ell)} - 
				{\left(\sqrt{1 + \gamma^2 \|\grad  F(\bq^{(\ell)}; \wt \bY)\|_2^2}-1\right)(\zeta_i^{(\ell)}-\zeta_j^{(\ell)})+ \gamma( A_i- A_j)^\T
					\grad  F(\bq^{(\ell)}; \wt \bY)  
					\over  \sqrt{1 + \gamma^2\|\grad  F(\bq^{(\ell)}; \wt \bY)\|_2^2}}\\
				&\ge \zeta_i^{(\ell)}-\zeta_j^{(\ell)} - 
				{(\gamma^2/2) \|\grad  F(\bq^{(\ell)}; \wt \bY)\|_2^2 (\zeta_i^{(\ell)}-\zeta_j^{(\ell)}) + \gamma( A_i- A_j)^\T
					\grad  F(\bq^{(\ell)}; \wt \bY)  
					\over  \sqrt{1 + \gamma^2\|\grad  F(\bq^{(\ell)}; \wt \bY)\|_2^2}}
			\end{align*}
			by the basic inequality $\sqrt{x + 1}-1\le x/2$ for all $x\ge 0$, it suffices to prove 
			\begin{equation}\label{ieq:lem:grad_away:sub1:step_size} 
				\gamma( A_j- A_i)^\T
				\grad  F(\bq^{(\ell)}; \wt \bY)  ~ \ge ~  (\gamma^2/2)\|\grad  F(\bq^{(\ell)}; \wt \bY)\|_2^2 ~  (\zeta_i^{(\ell)}-\zeta_j^{(\ell)}).
			\end{equation}
			Since \cref{lem_grad_F_norm} ensures that the right hand side is no greater than 
			\begin{equation}\label{bd_rhs}
				\gamma^2  (\zeta_i^{(\ell)}-\zeta_j^{(\ell)})\left[
				\kappa^2 \left( [\zeta_i^{(\ell)}]^2  - \|\bzeta^{(\ell)}\|_4^4 \right)^2 +   (\eta_n')^2
				\right].
			\end{equation}
			We proceed to bound from below
			\begin{align*}
    			&( A_j- A_i)^\T
    			\grad  F(\bq^{(\ell)}; \wt \bY)\\ &=  ( A_j- A_i)^\T
    			\grad h(\bq^{(\ell)}) - ( A_j- A_i)^\T
    			\left[\grad  F(\bq^{(\ell)}; \wt \bY) - \grad h(\bq^{(\ell)})\right].
			\end{align*}
			Recalling \eqref{eq_grad_ai}, we find that
			\begin{align}\label{lb_diff_ji_grad}\nonumber
				( A_j- A_i)^\T
				\grad h(\bq^{(\ell)}) & =  
				\kappa\left(
				[\zeta_i^{(\ell)}]^3 - \zeta_i^{(\ell)}\|\bzeta^{(\ell)}\|_4^4 - [\zeta_j^{(\ell)}]^3 + \zeta_j^{(\ell)}\|\bzeta^{(\ell)}\|_4^4\right)\\
				&\ge  \kappa\left(
				\zeta_i^{(\ell)} - \zeta_j^{(\ell)}
				\right)\left(
				[\zeta_i^{(\ell)}]^2 - \|\bzeta^{(\ell)}\|_4^4
				\right).
			\end{align}
			Since \eqref{cond_init_app} implies
			\begin{align}\label{lb_diff_zeta}
				\abs{\zeta_i^{(\ell)}}^2-\abs{\zeta_j^{(\ell)}}^2 = \left(\zeta_i^{(\ell)} -\zeta_j^{(\ell)} \right)\left(\zeta_i^{(\ell)}+ \zeta_j^{(\ell)}\right) \ge \mu \bar\nu,
			\end{align}
            by writing $\zeta_{(-i)}^{(\ell)}$ the sub-vector of $\zeta^{(\ell)}$ with its $i$th coordinate set to zero,
			we find that
			\begin{align}\label{bd_sup_ell_4}\nonumber
				[\zeta_i^{(\ell)}]^2- \|\mb \zeta^{(\ell)}\|_4^4 &  = \abs{\zeta_i^{(\ell)}}^2-\abs{\zeta_i^{(\ell)}}^4-\norm{\mb \zeta_{(-i)}^{(\ell)}}4^4\\\nonumber
				&\ge \abs{\zeta_i^{(\ell)}}^2-\abs{\zeta_i^{(\ell)}}^4 - \norm{\mb \zeta_{(-i)}^{(\ell)}}2^2\norm{\mb \zeta_{(-i)}^{(\ell)}}\infty^2\\\nonumber
				& = \abs{\zeta_i^{(\ell)}}^2-\abs{\zeta_i^{(\ell)}}^4 - \left(
				\norm{\bzeta^{(\ell)}}2^2 - \abs{\zeta_i^{(\ell)}}^2
				\right)\norm{\mb \zeta_{(-i)}^{(\ell)}}\infty^2\\\nonumber
				&=\left(\abs{\zeta_i^{(\ell)}}^2-\norm{\mb \zeta_{(-i)}^{(\ell)}}\infty^2\right)\left(1-\abs{\zeta_i^{(\ell)}}^2\right)\\\nonumber
				&\ge  {1\over 4}  \left( \zeta_i^{(\ell)} - \|\mb \zeta_{(-i)}^{(\ell)}\|_\infty \right) \zeta_i^{(\ell)} &&\text{by \eqref{case:lem:grad_away:1}}\\
				&\ge {\mu \bar\nu\over 4}&&\text{by \eqref{cond_init_app}}.
			\end{align}  
			Therefore,
			\begin{align*}
				&( A_j- A_i)^\T
				\left[\grad  F(\bq^{(\ell)}; \wt \bY) - \grad h(\bq^{(\ell)})\right]\\ &\le 2 \|\grad  F(\bq^{(\ell)}; \wt \bY) - \grad h(\bq^{(\ell)})\|_2\\
				&\le 2\eta_n'  &&\text{by $\bar\cE$ and \eqref{bd_gradient_app}}\\
				&\le  {\kappa\over 8} \mu \bar\nu^2 &&\text{by   \eqref{cond_omega_td_epsilon} and }\nu \le 4\bar \nu\\
				&\le {\kappa\over 2}\left(
				[\zeta_i^{(\ell)}]^2 - \|\bzeta^{(\ell)}\|_4^4
				\right)\left(\zeta_i^{(\ell)}-\zeta_j^{(\ell)}\right)&&\text{by  \eqref{bd_sup_ell_4}}.
			\end{align*}
			Together with \eqref{lb_diff_ji_grad}, we conclude that 
			\[
			( A_j- A_i)^\T
			\grad  F(\bq^{(\ell)}; \wt \bY)  \ge  {\kappa\over 2}\left(
			[\zeta_i^{(\ell)}]^2 - \|\bzeta^{(\ell)}\|_4^4
			\right)\left(\zeta_i^{(\ell)}-\zeta_j^{(\ell)}\right).
			\]
			In view of \eqref{bd_rhs}, the choice $\gamma \le c/\kappa$ ensures that \eqref{ieq:lem:grad_away:sub1:step_size} holds, completing the proof. 
		\end{proof}

		\subsubsection{Proof of \cref{lem_iter_sup}}\label{app_proof_lem_iter_sup}
		\begin{proof}  
			Recall that we write 
			$$
			\bzeta^{(\ell)} = \bA^\T \bq^{(\ell)},\quad \text{for all }\ell \ge 0.
			$$ 
			Without loss of generality, we assume $\mb \zeta_i^{(\ell)}\geq 0$ as otherwise we may replace $ A_i$ by $- A_i$.  By definition and \eqref{eq_gradient_update}, 
			\begin{align*}
				\zeta^{(\ell+1)}_i   =  A_i^\T \bq^{(\ell+1)} 
				& =  \frac{ A_i^\T\bq^{(\ell)}-\gamma A_i^\T\grad  F(\bq^{(\ell)}; \wt \bY)}{\|\bq^{(\ell)}-\gamma\grad  F(\bq^{(\ell)}; \wt \bY)\|_2} \\\nonumber
				&=\frac{1}{\sqrt{1+\gamma^2\|\grad  F(\bq^{(\ell)}; \wt \bY)\|_2^2}}\left(\zeta_i^{(\ell)}- \gamma ~  A_i^\T \grad  F(\bq^{(\ell)}; \wt \bY) \right)\\  
				& = \zeta_i^{(\ell)} - {\left(\sqrt{1 + \gamma^2\|\grad  F(\bq^{(\ell)}; \wt \bY)\|_2^2} - 1\right)\zeta_i^{(\ell)}+ \gamma ~   A_i^\T\grad  F(\bq^{(\ell)}; \wt \bY) \over  \sqrt{1 + \gamma^2\|\grad  F(\bq^{(\ell)}; \wt \bY)\|_2^2}}\\
				& \ge \zeta_i^{(\ell)} - {\gamma^2\|\grad  F(\bq^{(\ell)}; \wt \bY)\|_2^2 ~ \zeta_i^{(\ell)}/2 + \gamma ~  A_i^\T\grad  F(\bq^{(\ell)}; \wt \bY) \over  \sqrt{1 + \gamma^2\|\grad  F(\bq^{(\ell)}; \wt \bY)\|_2^2}}
			\end{align*}
			where the last step uses  $\sqrt{x + 1}-1\le x/2$ for all $x\ge 0$.  We work on the event $\bar\cE$ in the following. 
			
			On the one hand,  \cref{lem_grad_F_norm} ensures that 
			\begin{align*}
				&\|\grad  F(\bq^{(\ell)}; \wt \bY)\|_2^2\\  &\le   2\kappa^2\left([\zeta_i^{(\ell)}]^2 - \|\bzeta^{(\ell)}\|_4^4\right)^2 + 2 \left(\bar\eta_n +  \| A_i - \bq^{(\ell)}\|_2\sqrt{r^2\log(n)\over n}\right)^2\\
    &   \overset{\eqref{eq_gradient_update}}{\le} c^2\kappa^2. 
			\end{align*}
			On the other hand,   \eqref{eq_grad_ai}, \eqref{disp_beta_upper_bound} and $\bar\cE$ ensure that
			\begin{align*}
				 &A_i^\T\grad  F(\bq^{(\ell)}; \wt \bY)\\  & =  A_i^\T\grad h(\bq^{(\ell)}) +   A_i^\T \left(
				\grad  F(\bq^{(\ell)}; \wt \bY) - \grad h(\bq^{(\ell)})
				\right)\\ 
				&\le  - \kappa \zeta_i^{(\ell)} \left([\zeta_i^{(\ell)}]^2 - \|\bzeta^{(\ell)}\|_4^4\right) +  2 \left(\bar \eta_n   + \| A_i - \bq^{(\ell)}\|_2\sqrt{r^2\log(n)\over n} \right)\| A_i - \bq^{(\ell)}\|_2.
			\end{align*}
			By $\gamma \le c/\kappa$, it then follows that
			\begin{align}\label{contraction}\nonumber
				\zeta^{(\ell+1)}_i   & \ge \zeta^{(\ell)}_i
				\left[
				1 -     c \gamma^2    \kappa^2\left([\zeta_i^{(\ell)}]^2 - \|\bzeta^{(\ell)}\|_4^4\right)^2   + c  \gamma  \kappa  \left(
				[\zeta_i^{(\ell)}]^2 - \|\bzeta^{(\ell)}\|_4^4 \right)  
				\right]\\\nonumber
				& \qquad - c \gamma^2   \zeta^{(\ell)}_i  \bar\eta_n^2-  C \gamma  \bar\eta_n \| A_i - \bq^{(\ell)}\|_2 -C \gamma \sqrt{r^2\log(n)\over n} \| A_i - \bq^{(\ell)}\|_2^2\\
				&\ge  \zeta^{(\ell)}_i
				\left[1  +   c   \gamma  \kappa  \left([\zeta_i^{(\ell)}]^2 - \|\bzeta^{(\ell)}\|_4^4\right) - c  \gamma^2   \bar\eta_n^2
				\right]\\\nonumber
        & \qquad -    C \gamma  \bar\eta_n   \| A_i - \bq^{(\ell)}\|_2 -  C \gamma \sqrt{r^2\log(n)\over n} \| A_i - \bq^{(\ell)}\|_2^2.
			\end{align}
			We consider two cases in the sequel. Recall $\bar\eta_n\le \eta_n'$ from \eqref{def_eta_n_prime}.
			
			If $\mu \le  \zeta_i^{(\ell)} \le  \sqrt{3}/2$, then
			\begin{align*}
				\zeta^{(\ell+1)}_i   & \ge  \zeta^{(\ell)}_i
				\left[1  +   c \gamma \kappa \left([\zeta_i^{(\ell)}]^2 - \|\bzeta^{(\ell)}\|_4^4\right) -  c \gamma^2 \bar\eta_n^2
				\right] -   2C \gamma     \eta_n'\\
				& \ge \zeta^{(\ell)}_i
				\left(1  +    {c\over 4} \gamma \kappa  \mu \bar{\nu}  - c\gamma^2 \bar\eta_n^2
				\right) -    2C  \gamma  \eta_n' &&\text{by \eqref{bd_sup_ell_4}}\\
				&\ge \zeta^{(\ell)}_i\left( 1 +c''\gamma \kappa  \mu \bar{\nu}   \right) &&\text{by \eqref{cond_omega_td_epsilon} and $\gamma < c/\kappa$},
			\end{align*}
			proving the first claim. 
			 
			 
				If $\sqrt{3}/2 \le \zeta_i^{(\ell)}  \le 1$, then it follows from \eqref{bd_sup_ell_4} that 
				\begin{align*}
					|\zeta_i^{(\ell)}|^2 - \|\bzeta^{(\ell)}\|_4^4 &\ge \left(\abs{\zeta_i^{(\ell)}}^2-\norm{\mb \zeta_{(-i)}^{(\ell)}}\infty^2\right)\left(1-\abs{\zeta_i^{(\ell)}}^2\right)\\
					&\ge \left(\abs{\zeta_i^{(\ell)}}^2-\norm{\mb \zeta_{(-i)}^{(\ell)}}\infty^2\right)\left(1-\zeta_i^{(\ell)} \right)\\\
					&\ge \left(2\abs{\zeta_i^{(\ell)}}^2-1\right)\left(1-\zeta_i^{(\ell)} \right)\\
					&\ge  {1\over 2}\left(1-\zeta_i^{(\ell)} \right).
				\end{align*}
				Define 
				\[
					H_\ell = \|A_i -  \bq^{(\ell)}\|_2,\qquad \forall ~ \ell \ge 0,
				\]
				and note that 
				\[
				H_\ell^2 = 2 - A_i^\T \bq^{(\ell)} = 2(1-  \zeta_i^{(\ell)}).
				\]
				It suffices to prove the upper bound of $H_{\ell + \ell'}$. 
				From \eqref{contraction}, after a bit algebra, one can deduce that 
				\begin{align}\label{chain_loss}
					H_{\ell+1}^2 \le \left(1  - C_1 \kappa \gamma  \right) H_\ell^2 + C_2  \gamma    \bar\eta_n H_\ell + C_2 \gamma^2  \bar\eta_n^2 .
				\end{align}
				The claim 
				\begin{equation}\label{induction}
					H_{\ell + \ell'} \le {2C_2\over C_1} {\bar\eta_n \over \kappa}+ \left(
					1 - {C_1 \over 2}\kappa \gamma
					\right)^{\ell' / 2} H_\ell
				\end{equation}
				follows by induction. Indeed, the case $\ell' = 0$ follows trivially. Suppose 
				\eqref{induction} holds for some $\ell' \ge 1$. Then 
				\begin{align*}
					 H_{\ell+\ell'+1}^2 
					&\overset{\eqref{chain_loss}}{\le}    H_{\ell+\ell'}\left[\left(1  -   C_1\kappa \gamma  \right) H_{\ell+\ell'}  +  C_2  \gamma    \bar\eta_n \right]  + C_2 \gamma^2      \bar\eta_n^2\\
					&~ \overset{(i)}{\le}  H_{\ell+\ell'} \left[
					\left(1  -   C_1 \kappa\gamma  \right) \left({2C_2 \over C_1}{\bar\eta_n\over \kappa} + \left(
					1 - {C_1 \over 2}\kappa \gamma
					\right)^{\ell'/ 2} H_\ell
					\right) + C_2  \gamma     \bar\eta_n
					\right]+ C_2 \gamma^2     \bar\eta_n^2\\
					&~ \le  H_{\ell+\ell'} 	\left(1  -   {C_1 \over 2}\kappa \gamma  \right) \left[
					 {2C_2 \over C_1}{\bar\eta_n\over \kappa}  + \left(
					1 - {C_1 \over 2}\kappa\gamma
					\right)^{\ell' / 2} H_\ell 
					\right]+ C_2 \gamma^2      \bar\eta_n^2\\
					&~ \overset{(i)}{\le} \left(
					{2C_2 \over C_1}{\bar\eta_n\over \kappa}  + \left(
					1 - {C_1 \over 2}\kappa\gamma
					\right)^{\ell' / 2} H_\ell 
					\right)\left(1  -   {C_1 \over 2}\kappa \gamma  \right)\left[
					 {2C_2 \over C_1} {\bar\eta_n\over \kappa}  + \left(
					1 - {C_1 \over 2}\kappa\gamma
					\right)^{\ell' / 2} H_\ell 
					\right]\\
					&\quad + C_2 \gamma^2     \bar\eta_n^2\\
					&~ \le  
					\left({2C_2 \over C_1} {\bar\eta_n\over \kappa} \right)^2     + \left(
					1 - {C_1 \over 2}\kappa\gamma
					\right)^{\ell'+ 1} H_\ell^2 +  {4C_2\over C_1} {\bar\eta_n\over \kappa}  \left(
					1 - {C_1 \over 2}\kappa\gamma
					\right)^{(\ell'+1)/2} H_\ell\\
					&~ = \left({2C_2 \over C_1} {\bar\eta_n\over \kappa}  +\left(
					1 - {C_1 \over 2}\kappa\gamma
					\right)^{(\ell'+ 1)/2} H_\ell  \right)^2
				\end{align*}
				as desired. Both steps $(i)$ use the induction hypothesis \eqref{induction}. The penultimate step holds provided that $\kappa\gamma \le {2C_2/C_1}$.   
				The proof is complete.  
		\end{proof}

		\subsection{Proof of \cref{thm_one_col}: one column recovery via PGD}\label{app_proof_thm_one_col}

		\begin{proof}
			We will work on the event $\bar\cE$ given by \eqref{def_event_bar}. Let $\{\bq^{(\ell)}\}_{\ell \ge 0}$ be the PGD iterates for solving \eqref{obj_L4_tilde} and let $\bq$ be its final solution.    
			By invoking  Lemmas \ref{lem_iter_gap}  \&  \ref{lem_iter_sup} in conjunction with the condition $\eta_n' \le c \mu \nu^2$ and $\cE_{\init}^{(i)}(\mu, \nu)$, we conclude that the solution $\bq$  must satisfy 
			\begin{equation*}
				\abs{ A_i^\T \bq} \ge  {1\over 2},\qquad \abs{ A_i^\T \bq} - \max_{j\ne i}\abs{ A_j^\T \bq} \ge  \min\left\{
				\nu, {1\over 4}
				\right\}.
			\end{equation*}
			Since condition \eqref{cond_omega_n} also ensures that 
			$
			(C/4) \eta_n' \le  \nu/4 \le \nu \wedge (1/4),
			$
			we conclude that $\bq$ satisfies \eqref{cond_critical} hence satisfies \eqref{rate_critical_app} according to \cref{app_thm_critical}.  
		\end{proof}
		
		\subsection{Proof of \cref{thm_RA}: guarantees of the deflation varimax rotation}\label{app_proof_thm_RA}
		
		\begin{proof}
			Without loss of generaility, we assume $\bP = \bI_r$.
			Recall from \eqref{def_wc_Q} that  
			\begin{align*}
				\wc \bQ=\argmin_{\bQ\in\bbO_{r\times r}}\norm{\wh\bQ - \bQ}{F}
			\end{align*}
			where the columns of $\wh\bQ = (\wh Q_1,\ldots, \wh Q_r)$ are obtained from \eqref{iter_PGD} for solving \eqref{obj_L4_hat} $r$ times.
			By the fact that $\bR^\T \bA \in \bbO_{r\times r}$, we have
			\begin{align*}
				\norm{ \wc\bQ-\bR^\T\bA}{F} &\leq \norm{\wc\bQ-\wh\bQ}{F}+\norm{\wh\bQ-\bR^\T\bA}{F}\\
				&\leq 2\norm{\wh\bQ-\bR^\T\bA }{F}\\
				& = 2\norm{\bR \wh\bQ- \bA }{F}\\
				& =  2\norm{\wt\bQ- \bA }{F}.
			\end{align*} 
			Here we invoke  \cref{lem_connection} for the oracle solution $\wt\bQ$ analyzed in \cref{cor_A}. The proof is completed by further invoking \cref{cor_A}.
		\end{proof}

		\subsection{Proof of the results using the random initialization scheme}\label{app_sec_init_rand}

		Before proving \cref{thm_RA_rand}, we first give two lemmas that are useful for establishing the orders of $\mu$ and $\nu$	in $\cE_{\init}^{(k)}(\mu, \nu)$.  Recall from \eqref{def_event_init} that $\cE_{\init}^{(k)}(\mu, \nu)$ is defined relative to 
		$$
			\wt Q_k^{(0)} :=  \bR \wh Q_k^{(0)}\quad  \text{ for each $k\in [r]$.}
		$$ 
		Here the orthogonal matrix $\bR$ is given in \eqref{decom_Ur}.
		
		For given $k\ge 2$, we define an event that quantifies the error of previously recovered $(k-1)$ columns.  To this end, recall that  $(\wh Q_1, \ldots, \wh Q_r)$ are the solutions obtained from \eqref{iter_PGD} for solving \eqref{obj_L4_hat}.  
		Let $\eps_n\ge 0$ be some generic, deterministic sequence.  Define $\P\{\cE_A(1,\eps_n)\} = 1$ and 
		\begin{equation}\label{def_event_E_k}
			\cE_A(k,\eps_n)  := \left\{
			\max_{1\le i\le k-1}\|
			  \wh Q_i -  \bR^\T A_i
			\|_2 ~ \le ~  \eps_n
			\right\},\quad \text{for all $2\le k \le r$}.
		\end{equation}
		The following lemmas provide two crucial properties  on $\bA^\T \wt Q_k^{(0)}$ with $\wt Q_k^{(0)}$ defined in \eqref{def_q_init_all_cols}. Their proofs are stated in \cref{app_proof_lem_init_rand}.

		\begin{lemma}[Expression of $\mu$]\label{lem_sup_init}
			Grant model \eqref{model_Y_hat}.
			For any $ k\in [r]$ and any $\eps_n\ge 0$ such that 
			\begin{equation}\label{cond_omega_n_bar_tilde}
				\eps_n\sqrt{r(k-1)}  \le c
			\end{equation}
			holds
			for some sufficiently small constant $c>0$, then on the event $\cE_A(k, \eps_n)$, 
			the initialization $\wt Q_k^{(0)}$ in \eqref{def_q_init_all_cols} satisfies 
			\[
			\max_{k\le i \le r}  \abs{ A_i^\T \wt Q_k^{(0)}} =  \|\bA^\T \wt Q_k^{(0)}\|_\infty \ge {1\over \sqrt{r}}.
			\]
		\end{lemma}

		\begin{lemma}[Expression of $\nu$]\label{lem_gap_init}
			Grant model \eqref{model_Y_hat}.
			Fix any $1\le k\le r$, any $\eps_n\ge 0$ and any $\delta>0$. Assume that there exists some sufficiently small constant $c>0$ such that
			\begin{equation}\label{cond_omega_n_bar_tilde_prime}
				\eps_n\sqrt{r(k-1)}  \le c ~ \delta 
			\end{equation}
			holds. Then for any $k \le i\le r$ with
			\begin{equation}\label{def_i_argmax}
				\abs{ A_i^\T\wt  Q_k^{(0)}} = \max_{k\le j\le r}\abs{ A_j^\T\wt  Q_k^{(0)}},
			\end{equation} 
			the initialization $\wt  Q_k^{(0)}$ in \eqref{def_q_init_all_cols} satisfies 
			\begin{align*}
				\P\left\{ \left\{\abs{ A_i^\T\wt  Q_k^{(0)}} -\max_{j\in[r]\setminus \{i\}}\abs{ A_j^\T\wt  Q_k^{(0)}} \ge \sqrt{1\over 9e}{\delta  \over \sqrt r} \right\}\cap  \cE_A(k, \eps_n) \right\}\ge 1-\delta r - e^{-2r}
			\end{align*} 
		\end{lemma}

		\subsubsection{Proof of \cref{thm_RA_rand}: guarantees of the random initialization}\label{app_proof_thm_RA_rand}

		\begin{proof}		
			We work on the event $\bar\cE$ defined in \eqref{def_event_bar} which holds with probability at least $1-n^{-1}$.  Without loss of generality, we assume the identity permutation $\pi$ as well as $\bP = \bI_r$.

			In view of \cref{thm_RA}, it suffices to (i) establish the order of $\mu$ and $\nu$  to which $\cE^{(k)}_{\init}(\mu, \nu)$ holds for $\wt Q_{k}^{(0)}$ for each $k\in [r]$, and (ii) verify condition \eqref{cond_omega_n} for the established order of $\mu$ and $\nu$.  We prove this by induction.
			
			For $k = 1$,  notice that $\cE_A(1, \eps_n)$ defined above \eqref{def_event_E_k} holds with probability one, for any $\eps_n$.  Furthermore, conditions \eqref{cond_omega_n_bar_tilde} and \eqref{cond_omega_n_bar_tilde_prime} of  Lemmas \ref{lem_sup_init} \& \ref{lem_gap_init} hold automatically for $k=1$.  Invoking Lemmas \ref{lem_sup_init} \& \ref{lem_gap_init} thus gives that, with probability at least $1-r\delta -  e^{-2r}$,
			\[
			\abs{ A_1^\T \wt Q_1^{(0)}} =  \left\|\bA^\T \wt Q_1^{(0)}\right\|_\infty \ge  {1\over  \sqrt{r}}, \quad  \abs{ A_1^\T \wt Q_1^{(0)}} - \max_{j\in [r]\setminus\{1\}}\abs{ A_j^\T \wt Q_1^{(0)}} \ge  c_0{\delta \over  \sqrt r} 
			\]
			for $c_0 = \sqrt{1/(9e)}$.  This implies $\cE^{(1)}_{\init}(\mu, \nu)$ holds for $\wt Q_1^{(0)}$ with 
			\[
			\mu = {1\over  \sqrt{r}},\qquad \nu =  c_0{\delta \over  \sqrt r} .
			\]
			For the above $\mu$ and $\nu$,  condition \eqref{cond_omega_n} is ensured by   \eqref{cond_omega_n_rand}. We thus  have verified  (i) and (ii) for $k=1$. 
			
			Suppose (i) and (ii) hold for any $k \in [2, r-1]$.  By 
			invoking \cref{thm_one_col} with $\cE^{(i)}_{\init}(\mu, \nu)$ for all $i\le k$ and the identity permutation matrix, we have 
			\[
				\max_{1\le i\le k}\|\wh  Q_i -  \bR^\T A_i\|_2 \le C( \omega_n + \epsilon^2) := \eps_n,
			\]
			which further implies that the event 
			$\cE_A(k+1, \eps_n)$ in \eqref{def_event_E_k} holds.  
			It remains to prove that  on the event $\bar{\cE}\cap \cE_A(k+1, \eps_n)$, the event  $\cE^{(k+1)}_{\init}(\mu, \nu)$ holds for $\wt Q_{k+1}^{(0)}$ with probability at least $1-r\delta - e^{-2r}$.  To this end, by invoking Lemmas \ref{lem_sup_init} and  \ref{lem_gap_init} for $k+1$ and the identity permutation $\pi$, we have that
			with probability $1-r\delta -  e^{-2r}$, 
			\[
			\abs{A_{k+1}^\T \wt Q_{k+1}^{(0)}} \ge \mu,\quad   \abs{ A_{k+1}^\T \wt Q_{k+1}^{(0)}} - \max_{j\in [r]\setminus\{k+1\}}\abs{ A_j^\T \wt Q_{k+1}^{(0)}} \ge \nu
			\]
			with $\mu$ and $\nu$ specified above. This proves the claim. Finally, we complete the proof by taking the union bounds over $k\in [r]$. 
		\end{proof}

		\subsubsection{Proofs of \cref{lem_sup_init} \& \cref{lem_gap_init}}\label{app_proof_lem_init_rand}

		We first state a key lemma that both proofs of \cref{lem_sup_init} and \cref{lem_gap_init} use. For any $k\in [r]$, recall that  $\wh\bV_{(-k+1)}$ contains the first $(r-k+1)$ eigenvectors of 
		$$
			\wh \bP_{k-1}^\perp = \bI_r - \sum_{i=1}^{k-1} \wh  Q_i\wh Q_i^\T,\qquad \text{with }\quad \wh \bP_{0}^\perp = \bI_r.
		$$  
		Denote by $ \bA_{(-k+1)}$ the  matrix that contains the last $(r-k+1)$ columns of $\bA$. 
		The following lemma states upper bounds  of the operator norm of $\|\wh\bV_{(-k+1)} -  \bR^\T \bA_{(-k+1)} \bO_{(-k+1)}\|_\op$ for some $\bO_{(-k+1)} \in \bbO_{(r-k+1)\times (r-k+1)}$.

		\begin{lemma}\label{lem:bound_V_A}
			Grant model \eqref{model_Y_hat}.
			For any $1\le k\le r$,  on the event $\cE_A(k, \eps_n)$ with 
			$\eps_n  \sqrt{r} \le  1$,
			there exists $\bO_{(-k+1)} \in \bbO_{(r-k+1)\times (r-k+1)}$ such that  
			\begin{align*}
				\|\wh\bV_{(-k+1)} - \bR^\T\bA_{(-k+1)} \bO_{(-k+1)} \|_\op \le  ~ 3\eps_n \sqrt{k-1}.
			\end{align*}
		\end{lemma}
			\begin{proof}
			The result follows trivially for $k=1$. For $k\ge 2$, recall that 
			\begin{equation}\label{eq_P_perp_hat}
			\wh \bP_{k-1}^{\perp} = \bI_p - \sum_{i=1}^{k-1}\wh Q_i\wh Q_i^\T.
			\end{equation}
			Similarly, write 
			\begin{equation}\label{eq_P_perp}
			\bP_{k-1}^{\perp} = \bI_p - \sum_{i=1}^{k-1} \bR^\T A_i A_i^\T \bR. 
			\end{equation}
			Denoting $\wh \bQ_{(k-1)} = (\wh Q_1,\ldots,\wh Q_{k-1})$, we find on the event $\cE_A(k, \eps_n)$ that
			\begin{align}\label{bd_P_k_diff}\nonumber
				\left\|\wh \bP_{k-1}^\perp  - \bP_{k-1}^\perp   
				\right\|_\op &\le \left\|\wh  \bQ_{(k-1)}\wh  \bQ_{(k-1)}^\T  -  \bR^\T\bA_{(k-1)}  \bA_{(k-1)}^\T  \bR
				\right\|_\op\\\nonumber
				&\le 2\left\|\wh  \bQ_{(k-1)} - \bR^\T \bA_{(k-1)}\right\|_\op + \left\|\wh  \bQ_{(k-1)} -  \bR^\T \bA_{(k-1)}\right\|_\op^2\\\nonumber
				&\le 2\sqrt{(k-1)\eps_n^2} + (k-1)\eps_n^2\\
				&\le 3 \eps_n  \sqrt{k-1}.
			\end{align}
			In the last step we used $\eps_n\sqrt{r} \le 1$. 
			The result then follows by an application of the Davis-Kahan theorem in \cref{lem_sin_theta}.
		\end{proof}

		\begin{proof}[Proof of \cref{lem_sup_init}]
			Fix any $k \in [r]$ and we work on the event $\cE_A(k, \eps_n)$. According to \cref{lem:bound_V_A}, there exists  some $\bO_{(-k+1)} \in \bbO_{(r-k+1)\times (r-k+1)}$ such that
			\begin{equation}\label{def_event_A}
				\|\wh\bV_{(-k+1)} - \bR^\T \bA_{(-k+1)} \bO_{(-k+1)} \|_{\op} ~ \le  ~ 3\eps_n\sqrt{k-1}.
			\end{equation}
			Let us write for short
			\begin{equation}\label{def_Delta_k}
				\bDelta_k = \wh\bV_{(-k+1)} - \bR^\T\bA_{(-k+1)} \bO_{(-k+1)},\quad \text{for all }k\in [r].
			\end{equation}
			By definition in \eqref{def_q_init_all_cols},  
			\begin{align*}
				\max_{k\le i \le r}  \abs{ A_i^\T \wt Q_k^{(0)}}  &= \max_{k\le i \le r}  \abs{ A_i^\T \bR \wh Q_k^{(0)}}    \\
				& = \max_{k\le i \le r}\frac{\abs{ A_i^\T \bR \wh\bV_{(-k+1)}~g_{(-k+1)}}}{\norm{g_{(-k+1)}}2}\nonumber\\
				& \ge \max_{k\le i \le r}\left(\frac{\abs{ A_i^\T  \bA_{(-k+1)}\bO_{(-k+1)}~g_{(-k+1)}}}{\norm{g_{(-k+1)}}2}-\frac{\abs{ A_i^\T\bR \bDelta_k g_{(-k+1)}}}{\norm{g_{(-k+1)}}2}\right)\nonumber\\
				&\ge  \frac{\|\bO_{(-k+1)}~g_{(-k+1)}\|_\infty}{\norm{g_{(-k+1)}}2}-  \|\bDelta_k\|_\op.
			\end{align*}
			Since $\norm{g_{(-k+1)}}2 = \norm{\bO_{(-k+1)} g_{(-k+1)}}2 \le \sqrt{r-k+1}\|\bO_{(-k+1)} g_{(-k+1)}\|_\infty$, invoking \eqref{def_event_A} and condition \eqref{cond_omega_n_bar_tilde} yields  
			\begin{align*}
				\max_{k\le i \le r}  \abs{ A_i^\T \wh Q_k^{(0)}} \ge {1\over \sqrt{r-k+1}} - 3  \eps_n\sqrt{k-1}   \ge {1\over \sqrt{r}}.
			\end{align*}
			On the other hand, we find that
			\begin{align*}
				&\max_{1\le i\le k-1} \abs{ A_i^\T \wt Q_k^{(0)}}\\ &\le \max_{i\le k-1} {\abs{ A_i^\T \bR \wh\bV_{(-k+1)}g_{(-k+1)}}\over \norm{g_{(-k+1)}}{2}}\\
				&\le \max_{i\le k-1}\left( {\abs{ A_i^\T \bA_{(-k+1)} \bO_{(-k+1)}g_{(-k+1)}}\over \norm{g_{(-k+1)}}{2}} + {\abs{ A_i^\T\bR \bDelta_kg_{(-k+1)}}\over \norm{g_{(-k+1)}}{2}}\right)\\
				&\le \| \bDelta_k\|_\op &&\text{by } A_i^\T  \bA_{(-k+1)} = 0\\
				&\le 3 \eps_n\sqrt{k-1} &&\text{by \eqref{def_event_A}}\\
				&\le 1/\sqrt{r} &&\text{by condition \eqref{cond_omega_n_bar_tilde}}.
			\end{align*}
			This completes the proof. 
		\end{proof}

		\begin{proof}[Proof of \cref{lem_gap_init}]
			We work on the event  $\cE_A(k, \eps_n)$ under which \eqref{def_event_A} holds. Further recall the definition of $\bDelta_k$ from \eqref{def_Delta_k}.
			Fix any $k\le i\le r$ such that \eqref{def_i_argmax} holds. By \cref{lem_sup_init}, we also know that 
			\[
				|A_i^\T\wt Q_k^{(0)}| =\|\bA^\T\wt Q_k^{(0)}\|_\infty.
			\]
			
			Pick any $k\le j\le r$ with $j\ne i$. Without loss of generality, we assume both
			$ A_i^\T \wt Q_k^{(0)}$ and $ A_j^\T \wt Q_k^{(0)}$ are positive. 
			The definition of $\wh Q_k^{(0)}$ in \eqref{def_q_init_all_cols} gives
			\begin{align*}
				 A_i^\T\wt Q_k^{(0)}- A_j^\T \wt Q_k^{(0)} &= \abs{  A_i^\T\bR \wh Q_k^{(0)}- A_j^\T \bR  \wh Q_k^{(0)}}\\
				&=\frac{\abs{( A_i- A_j)^\T\bR  \wh\bV_{(-k+1)}g_{(-k+1)}}}{\norm{g_{(-k+1)}}{2}}\\
				&\geq \frac{\abs{( A_i- A_j)^\T  \bA_{(-k+1)} \bO_{(-k+1)}g_{(-k+1)}}}{\norm{g_{(-k+1)}}{2}} - \frac{\abs{( A_i- A_j)^\T \bR  \bDelta_k g_{(-k+1)}}}{\norm{g_{(-k+1)}}{2}}\\
				&\geq \frac{\abs{(\mb e_i-\mb e_j)^\T  \bO_{(-k+1)}g_{(-k+1)}}}{\norm{g_{(-k+1)}}{2}} -  \sqrt{2}\|\bDelta_k\|_\op .\label{ieq:bound_2:gamma_1_2}
			\end{align*}
			Since $g_{(-k+1)}$ is independent of $\bO_{(-k+1)}$, by recognizing that 
			$$(\mb e_i-\mb e_j)^\T\bO_{(-k+1)}g_{(-k+1)} \mid \bO_{(-k+1)} \sim \cN(0, 2),$$ \cref{lem:Anti_concen_gauss} gives that, for any $\delta>0$,
			\begin{equation}\label{eq_anti}
				\P\left\{
				\left|(\mb e_i-\mb e_j)^\T\bO_{(-k+1)}g_{(-k+1)}\right| \ge \delta\sqrt{2/ e} ~~ \big| ~ \bO_{(-k+1)}
				\right\} \ge 1-\delta,
			\end{equation}
			which also holds unconditionally. By invoking \cref{lem_quad}, we also have that
			\[
			\P\left\{
			\|g_{(-k+1)}\|_2 \le 
			3\sqrt{r} 
			\right\} \ge 1-  e^{-2r}.
			\]
			Together with \eqref{def_event_A}, invoking condition \eqref{cond_omega_n_bar_tilde_prime} yields that, for any $k\le j\le r$, 
			\[
			\abs{  A_i^\T\wt Q_k^{(0)}- A_j^\T \wt Q_k^{(0)}} \ge \sqrt{2\over 9e}{\delta \over \sqrt r} - 3\eps_n\sqrt{k-1}   \ge \sqrt{1\over 9e}{\delta \over \sqrt r}
			\]
			holds with probability at least $1-\delta-e^{-2r}$. 
			
			For any $1\le j<k$, the same arguments can be applied except by using
			\[
			( A_i- A_j)^\T  \bA_{(-k+1)} \bO_{(-k+1)}g_{(-k+1)} = \mb e_i^\T\bO_{(-k+1)}g_{(-k+1)} \mid \bO_{(-k+1)} \sim \cN(0, 1).
			\]
			Finally, taking a union bound over $j\in [r]\setminus \{i\}$ completes  the proof.  
		\end{proof}

		\subsection{Proof of the results using the method of moments based initialization scheme}\label{app_proof_thm_RA_mrs}

		Our proof uses the following lemma which states concentration inequalities of the absolute value of the maximum of $N$ independent standard normal variables. 
		
		\begin{lemma}\label{lem_one_sided_Gaussian}
			Let $W_1,\ldots, W_N$ be i.i.d. from $\cN(0,1)$. Let $C>0$ be some absolute constant.  For any $\delta\ge 0$, one has 
			\begin{align*}
				&\P\left\{ 
				\max_{i\in [N]} |W_i| < \sqrt{2\log(N)} - {\log\log(N) + C \over 2\sqrt{2\log(N)}} - \sqrt{2\log(1/\delta)}
				\right\} \le \delta,\\
				&\P\left\{
				\max_{i\in [N]} |W_i| > \sqrt{2\log(2N)} + \sqrt{2\log(1/\delta)}
				\right\} \le \delta.
			\end{align*}  
		\end{lemma} 
		\begin{proof}
			Its proof can be found in the proof of   \citet[Lemma B.1]{anandkumar2014guaranteed}.
		\end{proof}

		\begin{proof}[Proof of \cref{thm_RA_mrs}]
			 Without loss of generality, we assume the identity permutation $\pi$ and $\bP = \bI_r$. 
			In view of  \cref{thm_RA}, we need to prove that with the desired probability,  for each $k\in [r]$,   the event $\cE_{\init}^{(k)}(\mu, \nu)$ holds for $\bR  \wh Q_k^{(0)}$ with 
			$\mu$ and $\nu$ given in \eqref{rate_mu_nv_mrs}.  We prove this by induction. 
			
			Fix arbitrary $1\le k\le r$ and suppose that $\cE_{\init}^{(\ell)}(\mu, \nu)$ holds for $\bR  \wh Q_\ell^{(0)}$ for all $1\le \ell \le k-1$ whenever $k \ge 2$. Recall from \eqref{def_event_E_k} that $\cE_A(k, \eps_n)$ quantifies the  error of the previous $(k-1)$ estimated columns.  By \cref{thm_one_col}, we know that 
			\begin{equation}\label{prob_event_A}
				\P\left\{
					\cE_A(k, \eps_n) 
				\right\} \ge 1-(k-1)n^{-C},\qquad \text{with}\quad \eps_n = C'(\omega_n + \epsilon^2).
			\end{equation}
			The sequence $\omega_n$ is defined in \eqref{def_omega_n}.
			Our goal is to show that  
			\begin{equation}\label{goal_mrs}
					A_k^\T \bR \wh Q_k^{(0)} \ge 1 -  C''r(\omega_n + \epsilon^2),
			\end{equation}
			which also implies that 
			\begin{align*}
				&A_k^\T \bR \wh Q_k^{(0)} - \max_{\ell \in [r]\setminus \{k\}} A_\ell^\T\bR \wh Q_k^{(0)}\\  &\ge  A_k^\T \bR \wh Q_k^{(0)}  -\sqrt{ \|\bA^\T \bR \wh Q_k^{(0)}\|_2^2- (A_k^\T \bR \wh Q_k^{(0)})^2}\\
				& =   A_k^\T \bR \wh Q_k^{(0)}  -\sqrt{ 1- (A_k^\T \bR \wh Q_k^{(0)})^2}
			\end{align*}  
			as desired. 
			
			 Recall $ \wh \bM_{(k)}(\bG_{(i)})$ from \eqref{def_M_hat_k} for any $i\in [N]$
			with singular values $ \wh \sigma^{(i,k)}_1 \ge  \wh \sigma^{(i,k)}_2 \ge \cdots \ge \wh \sigma^{(i,k)}_r$  and  the leading left singular vector  $\wh u_{1}^{(i,k)}$. In particular, with 
			\begin{equation*} 
				\wh i_*  := 	\wh i_*(k):= \argmax_{i\in [N]}  \left( 
				\wh \sigma_{1}^{(i,k)}  -\wh  \sigma_{2}^{(i,k)} 
				\right),
			\end{equation*}
			the initialization for the $k$th round is 
			$$
			\wh Q_k^{(0)} = \wh u_1^{(\wh i_*,k)}.
			$$ 
			We write $\wh \bM_{(i, k)} := \wh \bM_{(k)}(\bG_{(i)})$ for short in the sequel.
			For notational simplicity, we write 
			\begin{equation}\label{def_bar_A}
				\bar\bA = \bR^\T \bA.	
			\end{equation}
			With $ \bP_{k-1}^{\perp} $ given by  \eqref{eq_P_perp}, we define the population level counterpart of $\wh \bM_{(i, k)}$ as 
			\begin{align}\label{eq_M_ik}\nonumber
				\bM_{(i, k)} &=   \bP_{k-1}^{\perp}   \left( \kappa \sum_{j=1}^r  \bar A_j \bar A_j^\T \bar A_j^\T  \bG_{(i)} \bar A_j  +  \delta_i \bI_r \right) \bP_{k-1}^{\perp} \\    
				& =  \sum_{j=k}^r \bar A_j \bar A_j^\T  \left( \kappa\bar A_j^\T \bG_{(i)} \bar A_j+ \delta_i\right)
			\end{align} 
			with singular values $\sigma^{(i,k)}_1 \ge  \sigma^{(i,k)}_2 \ge \cdots\ge \sigma^{(i,k)}_r$ and the left leading singular vector $u_{1}^{(i,k)}$.
			Here the scalar $\delta_i$ is defined as     
			\begin{equation}\label{def_delta_i}
				\delta_i = {1\over 3} \tr\left(\bR  \bG_{(i)} \bR^\T (\bI_r +  \bSigma_N )\right).
			\end{equation} 
			 Note that $u_1^{(\wh i_*,k)}$ can be treated as the   population level counterpart of $ \wh u_1^{(\wh i_*,k)}$.  The following error  between $\wh \bM_{(i, k)}$ and $\bM_{(i, k)}$ is crucial in our analysis:			 \begin{equation}\label{def_Delta_M}
			 	\Delta_{M,k} :=  \max_{i\in [N]}  ~ \|\wh \bM_{(i, k)} - \bM_{(i, k)}\|_\op.
		 	\end{equation}
	 		For some sufficiently small constant $c>0$, define the following event
	 		\begin{equation*}
	 			\cE_{M,k} = \left\{ \Delta_{M,k} \le  C r(\omega_n + \epsilon^2)\sqrt{\log(N)}  \right\}.
	 		\end{equation*} 
 			For future observation, we note that condition \eqref{cond_omega_n_msr} and $\cE_{M,k}$ ensures 
			\begin{equation}\label{bd_Delta_M}
 				\Delta_{M,k} \le c\sqrt{\log(N)}.
			\end{equation}
	 		Our proof of \eqref{goal_mrs} consists of three main parts.	
	 		\begin{enumerate}
	 			\item[(i)] We argue that on the event $\cE_{M,k}$, $u_{1}^{(\wh i_*,k)}$ must correspond to one column of $\{\bar A_k, \ldots, \bar A_r\}$. Since $\{\bar A_k, \ldots, \bar A_r\}$ are orthonormal, this requires to establish a lower bound on the gap 
	 			\begin{equation}\label{eigen_gap}
	 				\sigma^{(\wh i_*,k)}_1 - \sigma^{(\wh i_*,k)}_2 > 0.
	 			\end{equation}
 				\item[(ii)] We apply a variant of the Davis-Kahan theorem to show \eqref{goal_mrs} on the event $\cE_{M,k}$.
 				\item[(iii)] We establish the upper bound of $\Delta_{M,k}$ and  show that $\cE_{M,k}$ holds  with high probability.
	 		\end{enumerate}

			\noindent{\bf Proof of step (i).} To prove \eqref{eigen_gap}, define the population-level counterpart of $\wh i_*$ as 
			\begin{equation}\label{def_i_bar}   
				i_* := 	i_*(k) = \argmax_{i\in [N]} \left( \sigma_1^{(i,k)} - \sigma_2^{(i,k)}\right).
			\end{equation}
			By using Weyl's inequality in \cref{lem_weyl} twice, we find  that 
			\begin{align}\label{eigen_gap_init_lb}\nonumber
				\sigma^{(\wh i_*,k)}_1 - \sigma^{(\wh i_*,k)}_2  &\ge \wh \sigma^{(\wh i_*,k)}_1 - \wh \sigma^{(\wh i_*,k)}_2 - 2\Delta_{M,k} \\\nonumber
				&\ge  \wh \sigma^{(i_*,k)}_1 - \wh \sigma^{(i_*, k)}_2 - 2\Delta_{M,k} &&\text{by definition of $\wh i_*$ in \eqref{def_arg_index}} \\
				&\ge \sigma^{(i_*, k)}_1 -  \sigma^{(i_*, k)}_2 - 4\Delta_{M,k}.
			\end{align}
			To further bound from below the right hand side, we  introduce   
			 \begin{equation}\label{def_i_j}
			 	i_j = \argmax_{i\in [N]} ~ 	\bar A_j^\T \bG_{(i)} \bar A_j,\quad \text{	for all $ j \in [r]$}.
			 \end{equation}
		 	Fix any $j\in [r]$ and condition on $\bar A_j = \bR^\T  A_j$.
			On the one hand, since $\bar A_j^\T \bG_{(i)} \bar A_j $ for $ i\in [N]$  are independent standard normal variables (see the argument in \eqref{Gauss_AGA}),  invoking \cref{lem_one_sided_Gaussian} gives that for any $\delta \ge 0$,
			\begin{equation}\label{lowbound_first}
				\P\left\{
				 \max_{i\in [N]} |\bar A_j^\T \bG_{(i)} \bar A_j|  < \sqrt{2\log(N)} - {\log\log(N) + C \over 2\sqrt{2\log(N)}} - \sqrt{2\log(1/\delta)}
				\right\} \le \delta.
			\end{equation}
			On the other hand, for any $i\in [N]$, since the index $i_j$ in \eqref{def_i_j} is only defined through $\bar A_j^\T \bG_{(i)} \bar A_j$, it is independent of $\bar A_\ell^\T \bG_{(i)} \bar A_\ell$, for all  $\ell \ne j$, due to the orthogonality of columns in $\bar \bA$. Thus $\bar A_\ell^\T \bG_{(i_j)} \bar A_\ell$, for $k\le   \ell \le r$ with $\ell \ne j$,  are $(r-k)$ independent standard normal variables.  Invoking \cref{lem_one_sided_Gaussian} again yields that for all $\delta\ge 0$,
			\begin{equation}\label{upbound_rest}
				\P\left\{
				\max_{\ell \in [r]\setminus \{j\}}  |\bar A_\ell^\T  \bG_{(i_j)} \bar A_\ell| \ge \sqrt{2\log(2r-2k)} + \sqrt{2\log(1/\delta)}
				\right\} \le \delta.
			\end{equation}
			Fix any $\delta \ge 0$. By choosing $N \ge 4r(r-k) / \delta^2$, combing \eqref{lowbound_first} and \eqref{upbound_rest} gives that  with probability at least $1-\delta/r$, 
			\begin{equation}\label{eigen_gap_lb_draft}
					|\bar A_j^\T \bG_{(i_j)} \bar A_j|   - \max_{\ell \in [r]\setminus \{j\}}   |\bar A_\ell^\T  \bG_{(i_j)} \bar A_\ell|  ~ \ge ~ c\sqrt{\log (N)}.
			\end{equation} 
			From \eqref{eq_M_ik} and by the orthogonality of $\bar \bA$, we observe   that for each $i\in [N]$, the  singular values $\sigma^{(i, k)}_1$ and $\sigma^{(i, k)}_2$ must take values  from $\{\kappa|\bar A_{k}^\T \bG_{(i)} \bar A_{k}|+ \delta_i, \ldots, \kappa|\bar A_{r}^\T \bG_{(i)} \bar A_{r}|+\delta_i\}$.  Invoking the definition of $i_*$ in \eqref{def_i_bar} thus gives that with probability at least $1-\delta/r$,
			\begin{align}\label{eigen_gap_lb}\nonumber
					\sigma^{(i_*, k)}_1 -  \sigma^{(i_*,k)}_2 &\ge   \sigma^{(i_j, k)}_1 -  \sigma^{(i_j,k)}_2\\\nonumber
					& = |\bar A_j^\T \bG_{(i_j)} \bar A_j| - \max_{k\le \ell \le r, \ell \ne j}  |\bar A_\ell^\T  \bG_{(i_j)} \bar A_\ell| &&\text{by \eqref{eigen_gap_lb_draft}}\\ 
					&\ge c\sqrt{\log (N)} &&\text{by \eqref{eigen_gap_lb_draft}}.
			\end{align}
			Since the index $j$ above is chosen arbitrarily, in conjunction with \eqref{eigen_gap_init_lb}, we conclude that with probability at least $1-\delta/r$,
			\begin{equation}\label{eigen_gap_strong}
				\sigma^{(\wh i_*,k)}_1 - \sigma^{(\wh i_*,k)}_2 \ge c\sqrt{\log(N)} - 4\Delta_{M,k} \overset{\eqref{bd_Delta_M}}{\ge}  c'\sqrt{\log(N)},
			\end{equation} 
			which proves \eqref{eigen_gap}.\\
			
			\noindent{\bf Proof of step (ii).} 
			From step (i), we may take $u_{1}^{(\wh i_*,k)} = \bar A_k$ as we have already assumed $\bP = \bI_r$ without loss of generality. Also recall 
			that $\wh u_1^{(\wh i_*,k)} = \wh Q_k^{(0)}$. 
			An application of the Davis-Kahan theorem for rectangle matrices in \cref{lem_sin_theta_asym}  yields that on the event $\cE_{M,k}$, the following holds with probability at least $1-\delta/r$,   
			\begin{align}\label{bd_DK_M}\nonumber
				\sin \Theta\left(
				\wh Q_k^{(0)}, \bar A_k\right) 
				& = 
				 \sin \Theta\left(\wh u_1^{(\wh i_*,k)} , u_1^{(\wh i_*,k)}\right)\\\nonumber 
				  &\le    {2\bigl(2\sigma^{(\wh i_*,k)}_1+ \Delta_{M,k}\bigr) \Delta_{M,k} \over 
				 (	\sigma^{(\wh i_*,k)}_1 + \sigma^{(\wh i_*,k)}_2)(\sigma^{(\wh i_*,k)}_1 - \sigma^{(\wh i_*,k)}_2)
				}&&\text{by \eqref{def_Delta_M}}\\\nonumber
			&\le  {6\Delta_{M,k} \over 
			 	 \wh \sigma^{(\wh i_*,k)}_1 -\wh  \sigma^{(\wh i_*,k)}_2 - 2\Delta_{M,k}  }  &&\text{by \eqref{eigen_gap_strong} and Weyl's inequality}\\\nonumber
			&\le   {6\Delta_{M,k}  \over 
			\sigma^{(i_*, k)}_1 -  \sigma^{(i_*,k)}_2 - 4\Delta_{M,k} } &&\text{by \eqref{eigen_gap_init_lb}}\\
		&\le    {6\Delta_{M,k} \over c'\sqrt{\log (N)} }&&\text{by \eqref{eigen_gap_lb} and \eqref{bd_Delta_M}}.
			\end{align} 
		The notation $\sin \Theta(u,v)$ above operates $\sin$ to the acute angle between any vectors $u \in \Sp^r$ and $v\in \Sp^r$. By further using the identity 
			\begin{align*}
				\left[\sin\Theta(u, v)\right]^2 = 1 - (u^\T v)^2 = {1\over 2}\|u - v\|_2^2
			\end{align*}
			whenever $u^\T v > 0$, we conclude that with probability at least $1-\delta/r$,
			\begin{align*}
					\|\wh  Q_k^{(0)}- \bar A_k\|_2 & \le \sqrt{ 2}\sin \Theta\left(\wh  Q_k^{(0)} , \bar A_k\right) \le  {6\sqrt{2} \over c'} {\Delta_{M,k} \over \sqrt{\log (N)}}
			\end{align*} 
			up to a sign, as well as 
			\begin{align*}
				A_k^\T \bR \wh Q_k^{(0)}    \overset{\eqref{def_bar_A}}{=}    \bar A_k^\T \wh  Q_k^{(0)}
				&  \ge  \sqrt{1 - \sin \Theta\left(\wh  Q_k^{(0)}, \bar A_k\right)}\\
				& \ge \sqrt{1 - {6\Delta_{M,k} \over c'\sqrt{\log(N)}}}\\
				&\ge 1 - {C\Delta_{M,k} \over \sqrt{\log(N)}} &&\text{by \eqref{bd_Delta_M}}.
			\end{align*} 
			This in conjunction with $\cE_{M,k}$ implies  \eqref{goal_mrs}.\\  

			\noindent{\bf Proof of step (iii).} In the third step we prove that $\cE_{M,k}$ holds with high probability. This is technically involved and deferred to   \cref{lem_eps_M} of \cref{app_sec_lem_eps_M}  in which we prove that  
			\[
				\P\Bigl\{
					\cE_{M,k} \cap \cE_A(k, \eps_n)
				\Bigr\} \ge 1-n^{-C}-N^{-C}.
			\]
			\smallskip 
			
			Finally, by collecting the results of  steps (i) -- (iii) as well as \eqref{prob_event_A}, we conclude that \eqref{goal_mrs} holds with probability at least 
			\[
				 1- kn^{-C} - N^{-C} - \delta/r.
			\]
			Finally, taking the union bounds over $k\in [r]$ completes the proof. 
		\end{proof}
		
		\begin{remark}
			Similar results of \cref{thm_RA_mrs} can be proved under the general setting in \cref{app_sec_theory_varimax} with $\bU_{(r)}$ replaced by $\wh \bY$  that satisfies model \eqref{model_Y_hat}. Specifically,  consider the event $\cE_{\Omega}(\eta_n)$ given by \eqref{def_event_Omega}. For each $k\in [r]$, \cref{app_thm_critical} ensures that 
			\[
				\P\left\{\cE_A(k, \eps_n) \cap \cE_{\Omega}(\eta_n)\right\} \ge 1-n^{-C},\quad \text{with} ~ ~
				\eps_n = C\left(
			\eta_n + \sigma_N^2 + \sqrt{r\log(n) \over n}
			\right).
			\]
			By repeating the same arguments in the proof of \cref{thm_RA_mrs} together with \cref{lem_M_op}, we can verify that 
			the event $\cE_{\init}^{(k)}$ holds for $\bR  \wh Q_k^{(0)}$ with 
			\[
				  \mu = 1 - C \left(
				\eta_n +\|\bSigma_N\|_\op+ \sqrt{r\log(n) \over n}
				\right)r,\qquad \nu =2\mu - 1 
			\]  
			provided that condition \eqref{cond_omega_n_msr} is replaced by $1-\mu \le c'$ for some sufficiently small constant $c'>0$.
		\end{remark}

		\subsubsection{The technical lemma that bounds $\Delta_{M,k}$}\label{app_sec_lem_eps_M}

		By replacing $U_t$ in \eqref{def_M_hat_k} by $\wh Y_t$, we have  for each $i\in [N]$ and $k\in [r]$,
		\begin{equation}\label{eq_hat_M_ik}
			\wh \bM_{(i, k)} =  \wh \bP_{k-1}^{\perp}  \left( {1\over 3n}\sum_{t=1}^n \wh Y_t \wh Y_t^\T  \wh Y_t^\T \bG_{(i)} \wh Y_t   - \bG_{(i)} - \bG_{(i)}^\T \right) \wh \bP_{k-1}^{\perp}. 
		\end{equation}
		With $\bM_{(i, k)}$ defined in \eqref{eq_M_ik}, the following lemma establishes the rate of $	\|\wh  \bM_{(i, k)} - \bM_{(i, k)}\|_\op$ under the general setting in \cref{app_sec_theory_varimax}.

		\begin{lemma}\label{lem_eps_M}
			Fix any $k\in [r]$. Let $\wh \bM_{(i, k)}$ and $\bM_{(i, k)}$ be defined in \eqref{eq_hat_M_ik} and \eqref{eq_M_ik}, respectively. On the event $\cE_A(k, \eps_n) \cap \cE_{\Omega}(\eta_n)$ with any $\eps_n \le 1/\sqrt{r}$, the following holds with probability at least $1- n^{-C} - N^{-C}$ 
			\[ 
			\|\wh  \bM_{(i, k)} - \bM_{(i, k)}\|_\op \lesssim \left\{  r\left( \eta_n + \sqrt{r\log(n) \over n} \right) +  \sqrt{r}\left( \eps_n \sqrt{k-1} + \|\bSigma_N\|_\op\right)\right\} \sqrt{\log(N)}.
			\]
		\end{lemma}
		\begin{proof}
			Pick any $k\in [r]$ and $i\in [N]$.  Start with \[
			\|\wh  \bM_{(i, k)} - \bM_{(i, k)}\|_\op  \le \|\wh \bM_{(i, k)} - \overline \bM_{(i, k)}\|_\op + \|\overline \bM_{(i, k)} - \bM_{(i, k)}\|_\op
			\]
			where  we introduce the intermediate quantity
			\[
			\overline \bM_{(i, k)} = \wh \bP_{k-1}^{\perp}   \left( \kappa \sum_{j=1}^r \bar A_j \bar A_j^\T \bar A_j^\T \bG_{(i)} \bar A_j  + \delta_i \bI_r\right) \wh \bP_{k-1}^{\perp}
			\]			
			with $\delta_i$ given by \eqref{def_delta_i}.
			To bound $\|\wh \bM_{(i, k)} - \overline \bM_{(i, k)}\|_\op$, recall that $\wh Y_t = \bR^\T \wt Y_t$ and $\bar\bA = \bR^\T \bA$.  By writing 
			\[
			\bar \bG_{(i)} = \bR \bG_{(i)} \bR^\T,
			\]
			the term $\|\wh \bM_{(i, k)} - \overline \bM_{(i, k)}\|_\op $ is bounded from above by 
			\begin{align*}
				& \left\|
				{1\over 3n}\sum_{t=1}^n \wh Y_t \wh Y_t^\T  \wh Y_t^\T \bG_{(i)} \wh Y_t  - \bG_{(i)} - \bG_{(i)}^\T - \kappa \sum_{j=1}^r  \bar A_j \bar A_j^\T \bar A_j^\T  \bG_{(i)} \bar A_j   - \delta_i \bI_r
				\right\|_\op \\
				&= \left\|
				{1\over 3n}\sum_{t=1}^n  \wt Y_t \wt Y_t^\T  \wt Y_t^\T \bar \bG_{(i)}  \wt Y_t - \bar \bG_{(i)}  - \bar \bG_{(i)}^\T - \kappa \sum_{j=1}^r    A_j   A_j^\T   A_j^\T \bar \bG_{(i)} A_j   -\delta_i \bI_r
				\right\|_\op\\
				&\le   \rI + \rII 
			\end{align*}   
			where 
			\begin{align*}
				&\rI   ={1\over 3n}\left\|
				\sum_{t=1}^n  \left( \wt Y_t \wt Y_t^\T  \wt Y_t^\T\bar  \bG_{(i)}  \wt Y_t  -   Y_t   Y_t^\T    Y_t^\T \bar \bG_{(i)}   Y_t    \right)
				\right\|_\op \\
				&\rII  = \left\|
				{1\over 3n}\sum_{t=1}^n  Y_t   Y_t^\T    Y_t^\T \bar \bG_{(i)}   Y_t  -\bar \bG_{(i)}-\bar \bG_{(i)}^\T  - \kappa \sum_{j=1}^r    A_j   A_j^\T   A_j^\T \bar \bG_{(i)}^\T  A_j    -\delta_i \bI_r
				\right\|_\op.
			\end{align*}
			Invoking \cref{lem_diff_YGY_tilde} together with \eqref{def_event_Omega} and \eqref{def_event_1} gives that for all $t\ge 0$, 
			\begin{equation}\label{bd_I_epsM}
				\P\Bigl(
				\bigl\{
				\rI \lesssim r t \eta_n 
				\bigr\} \cap \cE_{\Omega}(\eta_n) \cap \cE_F
				\Bigr) \ge 1- 6e^{-t^2}.
			\end{equation}
			To bound from above $\rII $, let $\otimes$ denote the Kronecker product. In  \cref{fact_k_prod} we review some of the basic properties of the Kronecker product. 
			By using the notion
			\[
			Y_t   Y_t^\T    Y_t^\T \bar \bG_{(i)}   Y_t =  Y_t   Y_t^\T  (Y_t\otimes Y_t)^\T  \rvec(\bar \bG_{(i)})
			\]
			from part (4) of \cref{fact_k_prod}, 
			we decompose it into two terms $\rII  \le \rII_1 + \rII_2$ where 
			\begin{align*}
				\rII_1 	 &=  \left\|
				{1\over 3n}\sum_{t=1}^n \left( Y_t   Y_t^\T   (Y_t\otimes Y_t)^\T\rvec(\bar \bG_{(i)})- 
				\EE\left[Y_t   Y_t^\T   (Y_t\otimes Y_t)^\T\right]\rvec(\bar \bG_{(i)})
				\right)\right\|_\op   \\
				\rII_2 	&=  \left\|
				{1\over 3}\EE\left[Y_t   Y_t^\T   (Y_t\otimes Y_t)^\T\right]  \rvec(\bar \bG_{(i)})
				-\bar \bG_{(i)}-\bar \bG_{(i)}^\T  -  \kappa  \sum_{j=1}^r     A_j   A_j^\T   A_j^\T \bar \bG_{(i)}^\T  A_j   - \delta_i \bI_r 
				\right\|_\op.
			\end{align*} 
			Invoking \cref{lem_diff_YGY} gives that for all $t\ge 0$, 
			\begin{equation}\label{bd_II_1_epsM}
				\P\left\{
				\rII_1 \ge   C t \sqrt{r^3\log(n)\over n} 
				\right\} \le   \exp\left\{
				-t^2 +\log(r)
				\right\} - r^2 n^{-C}.
			\end{equation}
			For $\rII_2$, by  using \cref{lem:obj_exp} with $\bM = \bar \bG_{(i)}$, it is easy to verify that 
			\begin{align*}
				\EE\left[Y_t   Y_t^\T   (Y_t\otimes Y_t)^\T\right]   \rvec(\bar \bG_{(i)}) &  = 3\kappa  \sum_{i=1}^r   A_i  A_i^\T   A_i^\T  \bar \bG_{(i)}  A_i   +   \tr\left(\bar \bG_{(i)}(\bI_r +  \bSigma_N )\right) \left(\bI_r +  \bSigma_N\right)   \\
				&\quad     +   \left(\bI_r +  \bSigma_N\right)    \left( \bar \bG_{(i)}  +  \bar \bG_{(i)}^\T\right) \left(\bI_r +  \bSigma_N\right) 
			\end{align*}
			so that, by \eqref{def_delta_i}, 
			\begin{align*}
				\rII_2  &\le{1\over 3} \left\|
				\tr\left(\bar \bG_{(i)}(\bI_r +   \bSigma_N )\right)   \bSigma_N \right\|_\op     +{1\over 3} \left\|  \bSigma_N    \left( \bar \bG_{(i)}  +  \bar \bG_{(i)}^\T\right) \bSigma_N  
				\right\|_\op\\
				&\quad  + {2\over 3} \left\|  \bSigma_N    \left( \bar \bG_{(i)}  +  \bar \bG_{(i)}^\T\right) 
				\right\|_\op\\
				&\le {\|\bSigma_N\|_\op \over 3}\left|\tr(\bar \bG_{(i)})  +  \tr( \bar \bG_{(i)}\bSigma_N)\right| +  2\|\bSigma_N\|_\op \|\bar \bG_{(i)}\|_\op &&\text{by $\|\bSigma_N\|_\op \le 1$}\\
				&=  {\|\bSigma_N\|_\op \over 3}\left|\tr( \bG_{(i)})  +  \tr( \bG_{(i)}\bR^\T \bSigma_N\bR )\right| +  2\|\bSigma_N\|_\op \|\bG_{(i)}\|_\op.
			\end{align*}
			Classical concentration inequalities of i.i.d. standard normal variables ensure that for all $t\ge 0$,
			\begin{align}\label{dev_G_op}
				&\P\left\{
				\|\bG_{(i)}\|_\op \ge 2 \sqrt{r} + \sqrt{2t}
				\right\} \le 2e^{-t},\\\label{dev_G_tr}
				&\P\left\{
				\tr(\bG_{(i)})  \ge \sqrt{2rt} 
				\right\} \le e^{-t};\\\label{dev_G_Sigma_tr}
				&\P\left\{
				\tr(\bG_{(i)} \bR^\T \bSigma_N \bR) = \sum_{a,b=1}^r [\bG_{(i)}]_{ab} [ \bR^\T \bSigma_N \bR]_{ab} \ge \sqrt{2 t \|\bR^\T \bSigma_N \bR\|_F^2}
				\right\}\le e^{-t}.
			\end{align}
			By the fact that $\|\bR^\T \bSigma_N \bR\|_F^2 = \|\bSigma_N\|_F^2 \le r \|\bSigma_N\|_\op^2 \le r$, we conclude that for all $t\ge 0$, 
			\begin{equation}\label{bd_II_2_epsM}
				\P\left\{
				\rII_2 \le   4 \|\bSigma_N\|_\op\sqrt{r}\left( t+  1 \right)
				\right\} \ge 1 - 4e^{-t^2}.
			\end{equation}
			In conjunction with \eqref{bd_II_1_epsM}, we obtain that forall $t_1,t_2\ge 0$,
			\begin{equation}\label{bd_II_epsM}
				\P\left\{
				\rII \lesssim   t_1 \sqrt{r^3\log(n)\over n}  +   \left( t_2 + 1  \right) \|\bSigma_N\|_\op\sqrt{r}
				\right\} \ge  1 - e^{-t_1^2 +\log(r)} - 4e^{-t_2^2}.
			\end{equation}
			
			It remains to bound $ \|\overline \bM_{(i, k)} - \bM_{(i, k)}\|_\op$. We use some results already derived in \cref{app_sec_init_rand}. Recall from \eqref{bd_P_k_diff} that on the event $\cE_A(k, \eps_n)$,
			\[
			\|  \wh \bP_{k-1}^{\perp}  -   \bP_{k-1}^{\perp} \|_\op \le 3\eps_n \sqrt{k-1}.
			\]
			It then follows that 
			\begin{align*}
				\|\overline \bM_{(i, k)} - \bM_{(i, k)}\|_\op &\le  2 \Bigl\|  \wh \bP_{k-1}^{\perp}  -   \bP_{k-1}^{\perp} \Bigr\|_\op \Bigl \|
				\kappa \sum_{j=1}^r \bar A_j \bar A_j^\T \bar A_j^\T \bG_{(i)} \bar A_j + \delta_i \bI_r 
				\Bigr\|_\op \\
				&\le 6 \eps_n\sqrt{k-1}  \left(\kappa\Bigl \|
				\sum_{j=1}^r \bar A_j \bar A_j^\T \bar A_j^\T \bG_{(i)} \bar A_j
				\Bigr\|_\op+|\delta_i|\right).
			\end{align*} 
			Note that \eqref{dev_G_tr} and \eqref{dev_G_Sigma_tr} imply 
			\[
			\P\left\{|\delta_i| \le {1\over 3}\left|\tr(\bG_{(i)})| + \tr(\bG_{(i)} \bR^\T \bSigma_N\bR )\right| \le {2\sqrt{2}\over 3}t\sqrt{r}
			\right\} \ge 1-2e^{-t^2}.
			\]
			By also using \cref{lem_M_op},
			we conclude that on $\cE_A(k, \eps_n)$, for all $t_1,t_2\ge 0$,
			\begin{align}\label{bd_diff_Mbar_M}
				\|\overline \bM_{(i, k)} - \bM_{(i, k)}\|_\op \lesssim \eps_n \sqrt{r} \left(
				t_1 + t_2\sqrt{r}
				\right)
			\end{align}
			holds with probability at least $  1-2e^{-t_2^2} - e^{-t_1^2 + \log(r)}.$

			Collecting  \eqref{bd_I_epsM}, \eqref{bd_II_epsM} and \eqref{bd_diff_Mbar_M} concludes that, on the event $\cE_A(k, \eps_n) \cap \cE_{\Omega}(\eta_n) \cap \cE_F$,  for all $t_1, t_2\ge 0$,			
			\[
			\|\wh \bM_{(i, k)} - \bM_{(i, k)}\|_\op \lesssim 	 t_1     \eps_n\sqrt{r(k-1)}   +t_2  r\eta_n  + t_1 \sqrt{r^3\log(n)\over n}  +   \left( t_2 + 1  \right) \|\bSigma_N\|_\op\sqrt{r}
			\]
			holds with probability at least 
			$
			1- 4e^{-t_1^2 +\log(r)} - 10e^{-t_2^2} - n^{-C'}.
			$
			The proof is completed by  using $\P(\cE_F^c) \le n^{-C}$, taking the union bounds over $i\in [N]$  and choosing $t_1 = t_2= C\sqrt{\log(N)}$.
		\end{proof}


		\subsubsection{Other technical lemmas used in the proof of \cref{lem_M_op}}
			
		The following lemma bounds from above $\| \sum_{j=1}^r A_j A_j^\T A_j^\T \bG A_j\|_\op$ for  $\bG$ being an $r\times r$ matrix with i.i.d.  entries of $\cN(0,1)$. 
		
		\begin{lemma}\label{lem_M_op}
			Let $\bG \in \R^{r\times r}$ contain i.i.d. entries of $\cN(0,1)$.  For any fixed $\bA \in \bbO_{r\times r}$, one has that for all $t\ge 0$
			\[
				\P\left\{\Bigl\| \sum_{j=1}^r A_j A_j^\T A_j^\T \bG A_j \Bigr\|_\op \ge t
				\right\}  \le \exp\left(-t^2/2 + \log r \right).
			\]
		\end{lemma}
		\begin{proof}
			Note that by using parts (4) and (3) of \cref{fact_k_prod}, respectively,
			\begin{equation}\label{Gauss_AGA}
				A_j^\T \bG A_j =  (A_j \otimes A_j)^\T \rvec(\bG) \sim \cN(0,  (A_j \otimes A_j)^\T   (A_j \otimes A_j) ) = \cN(0, 1).
			\end{equation}
			By $\bA \in \bbO_{r\times r}$, we also know that $A_j^\T \bG A_j$ for $j\in [r]$ are i.i.d. $\cN(0,1)$.  The claim thus follows by an application of \cref{lem_op_Gauss_mat} with $d = r$ and 
			$
				\nu = \| \sum_{j=1}^r A_j A_j^\T \|_\op = 1.
			$ 
		\end{proof}

		\begin{lemma}\label{lem_diff_YGY_tilde}
			Grant model \eqref{model_Y_hat} and $n\ge r^2\log(n)$. Let $\bG\in \R^{r\times r}$ contain i.i.d. standard normal variables and $\bO\in \bbO_{r\times r}$ be any orthogonal matrix. On the event $\cE_{\Omega}(\eta_n) \cap \cE_F$ given by \eqref{def_event_Omega} and \eqref{def_event_1} with $\eta_n \le 1$, one has that for all $t\ge 0$,
			\[
			\left\|
			{1\over n}\sum_{t=1}^n  \left( \wt Y_t \wt Y_t^\T  \wt Y_t^\T \bO  \bG  \bO^\T \wt Y_t  -   Y_t   Y_t^\T    Y_t^\T \bO \bG  \bO^\T Y_t    \right)
			\right\|_\op \lesssim r \eta_n~ t 
			\]
			holds with probability at least $1-6e^{-t^2}$. 
		\end{lemma}
		\begin{proof}
			Let us write $\bar \bG = \bO \bG \bO^\T$. We bound from above  
			\begin{align*}
				\rI =  \sup_{\bq \in \Sp^r}
				{1\over 3n}\sum_{t=1}^n  \left[ (\bq^\T \wt Y_t)^2   \wt Y_t^\T\bar \bG  \wt Y_t  -   (\bq ^\T Y_t)^2 Y_t^\T \bar \bG  Y_t    \right].
			\end{align*}
			Fix any $\bq \in \Sp^r$. We bound from above 
			\begin{align*}
				\rI_1 &= {1\over n}\sum_{t=1}^n  (\bq^\T \wt Y_t)^2  \left( \wt Y_t^\T\bar \bG \wt Y_t - Y_t^\T \bar \bG  Y_t \right) \\
				\rI_2 &= {1\over n}\sum_{t=1}^n  (\bq^\T \wt Y_t +\bq^\T Y_t)(\bq^\T \wt Y_t -\bq^\T Y_t)Y_t^\T\bar \bG Y_t.
			\end{align*}
			Recall from \eqref{model_Y_tilde} that $\wt Y_t = Y_t + \Omega_t$. We have 
			\begin{align*}
				\rI_1 &\le {2\over n}\sum_{t=1}^n  (\bq^\T \wt Y_t)^2 |Y_t^\T\bar \bG  \Omega_t| + {1\over n}\sum_{t=1}^n  (\bq^\T \wt Y_t)^2 |\Omega_t^\T\bar \bG \Omega_t|\\
				&\le {2\over n}\sqrt{\sum_{t=1}^n (\bq^\T \wt Y_t)^4 }\sqrt{\sum_{t=1}^n (Y_t^\T \bar \bG \Omega_t)^2} + {1\over n}\sqrt{
					\sum_{t=1}^n (\bq^\T \wt Y_t)^4} \sqrt{\sum_{t=1}^n (\Omega_t^\T\bar \bG  \Omega_t)^2}.
			\end{align*}
			On the one hand, recall the  event $\cE_{\Omega}(\eta_n)$ from \eqref{def_event_Omega} and $\cE_F$ from \eqref{def_event_1}.  By using \eqref{bd_obj_L4}, we have that, on the event $\cE_{\Omega}(\eta_n)\cap \cE_F$, 
			\begin{align}\label{bd_q_Y_td_four}\nonumber
				{1\over n}\sum_{t=1}^n  (\bq^\T \wt Y_t)^4 &\lesssim
				{1\over n}\sum_{t=1}^n  (\bq^\T Y_t)^4 +  {1\over n}\sum_{t=1}^n  (\bq^\T  \Omega_t)^4\\\nonumber
				&\le \kappa + 4  + C\sqrt{r^2\log(n)\over n}+ \eta_n^4\\
				&\lesssim 1 &&\text{by $\eta_n \le 1$, $r^2\log(n)\le n$}.
			\end{align}
			On the other hand, since 
			$Y_t^\T \bar \bG \Omega_t = Y_t^\T \bO \bG  \bO^\T \Omega_t \sim \cN(0, \|Y_t\|_2^2 \|\Omega_t\|_2^2)$ (see, for instance, the arguments in \eqref{Gauss_AGA}), we have that for all $t\ge 0$,
			\begin{align}\label{bd_quad_Y_G}
				\P\left\{
				|Y_t^\T \bar \bG  \Omega_t|^2 \ge  2t^2 \|Y_t\|_2^2 \|\Omega_t\|_2^2  
				\right\} \le 2e^{-t^2}.
			\end{align}  
			Similar argument also yields that for all $t\ge 0$, 
			\[
			\P\left\{
			|\Omega_t^\T \bar \bG  \Omega_t|^2 \ge  2t^2   \|\Omega_t\|_2^4 
			\right\} \le 2e^{-t^2}.
			\]
			Finally, since 
			\begin{align}\label{bd_Omega_2_four}
				\sum_{t=1}^n \|\Omega_t\|_2^4 &=  \sum_{i,j  = 1}^r \sum_{t=1}^n  \Omega_{it}^2 \Omega_{jt}^2  \le \sum_{i,j=1}^r \left(\sum_{t=1}^n \Omega_{it}^4 \sum_{t=1}^n \Omega_{jt}^4\right)^{1/2}\le  r^2 \sup_{q \in \Sp^r} \sum_{t=1}^n  (\bq^\T \Omega_t)^4
			\end{align}
			and similarly 
			\begin{align}\label{bd_Y_2_four}
				\sum_{t=1}^n \|Y_t\|_2^4 \le r^2 \sup_{q \in \Sp^r} \sum_{t=1}^n (\bq^\T Y_t)^4,
			\end{align}
			we conclude that on the event $\cE_{\Omega}(\eta_n)\cap \cE_F$, for all $t\ge 0$, 
			\begin{align}\label{bd_I_1_epsM}\nonumber
				\rI_1 &\lesssim  t\left( {1\over n}\sum_{t=1}^n    \|\Omega_t\|_2^4  +   {1\over n}\sum_{t=1}^n \|\Omega_t\|_2^2\|Y_t\|_2^2   \right)^{1/2}\\\nonumber
				&\lesssim     r t \left( \eta_n^2 + {1\over n}\sqrt{\sum_t \|Y_t\|_2^4}\sqrt{\sum_t \|\Omega_t\|_2^4} \right) \\
				&\lesssim r t \eta_n 
			\end{align} 
			with probability at least $1-4e^{-t^2}$.  
			
			We proceed to bound $\rI_2$ as 
			\begin{align*}
				\rI_2 &\le   {2\over n}\sum_{t=1}^n \left|(\bq^\T \Omega_t)(\bq^\T  Y_t)Y_t^\T\bar \bG  Y_t\right|  + {1\over n}\sum_{t=1}^n (q^\T \Omega_t)^2 |Y_t^\T\bar \bG Y_t|\\
				&\le {1\over n} 
				\left(\sum_{t=1}^n (q^\T \Omega_t)^4 \right)^{1/2} \left(\sum_{t=1}^n (Y_t^\T\bar \bG Y_t)^2\right)^{1/2}\\
				& \quad + {2\over n} \left(
				\sum_{t=1}^n (\bq^\T \Omega_t)^4
				\right)^{1/4} \left(
				\sum_{t=1}^n (\bq^\T Y_t)^4
				\right)^{1/4}\left(
				\sum_{t=1}^n (Y_t^\T \bar \bG Y_t)^2
				\right)^{1/2}
			\end{align*}
			where we use H\"{o}lder's inequality twice in the last step.  
			Similar to \eqref{bd_quad_Y_G}, we have for all $t\ge 0$
			\[
			\P\left\{
			|Y_t^\T \bar \bG  Y_t|^2 \ge  2t^2 \|Y_t\|_2^4 
			\right\} \le 2e^{-t^2}
			\]
			so that in conjunction with  \eqref{bd_Y_2_four}, we have that on the event $\cE_{\Omega}(\eta_n)\cap \cE_F$, for all $t\ge 0$,
			\[
				\rII_2 \lesssim r t \eta_n 
			\]
			with probability at least $1-2e^{-t^2}$.   Together with \eqref{bd_I_1_epsM}, by observing that the arguments of bounding $\rI_1$ and $\rI_2$ hold uniformly over $\bq\in \Sp^r$, the proof is complete. 
		\end{proof}

		\section{Proofs of the decomposition of principal components and the results on estimating the loading matrix}

		Throughout the proof we write 
		\[
		\wh\bSigma_X := {1\over n}\bX^\T\bX,\qquad \bSigma_X := \EE[\wh\bSigma_X] = \sigma^2\bLambda\bLambda^\T + \sigma^2\epsilon^2\bSigma_E
		\]
		under Assumptions \ref{ass_Z}, \ref{ass_A_general} and \ref{ass_E_general}.

		\subsection{Proof of \cref{thm_PCs}: the decomposition of principal components}\label{app_proof_thm_PCs}
		
		We first state  two lemmas that are used in the proof of \cref{thm_PCs}. Their proofs are stated after the proof of \cref{thm_PCs}.
		The first lemma bounds from above  $\wh\bSigma_X - \bSigma_X$ in operator norm. 
		
		\begin{lemma}\label{lem:cov_concentration}
			Under Assumptions  \ref{ass_Z}, \ref{ass_A_general} and \ref{ass_E_general}, assume $n\ge r\log(n)$. For some large constant $C\ge 1$, one has
			\begin{align*}
				\P\left\{\|\wh\bSigma_X- \bSigma_X\|_\op\lesssim  ~ \sigma^2 \left( \sqrt{\epsilon^2 p\log(n)\over n} +  {\epsilon^2 p\log(n)\over n} + \sqrt{r\log(n)\over n}\right) \right\} \ge 1-n^{-C}. 
			\end{align*}
		\end{lemma}
	 	\cref{lem:cov_concentration} together with  $\omega_n \le 1$  ensures that 
	 	the event 
	 	\begin{equation}\label{def_event_X}
	 		\cE_X = \left\{
	 		\|\wh\bSigma_X - \bSigma_X\|_\op \lesssim ~ \sigma^2 \omega_n
	 		\right\}
	 	\end{equation}
	  	holds with probability $1-n^{-1}$.  Recall that $\bV_{(r)} \in \bbO_{p\times r}$ contains the  first $r$ eigenvectors of $\wh\bSigma_X$. Further recall  from \eqref{svd_A} that
		$
		\bLambda = \bL \bS \bA 
		$
		with  $\bL\in \bbO_{p\times r}$ containing the left $r$ singular vectors.  
		The following lemma gives upper bounds of $\|\bV_{(r)} - \bL\bR\|_\op$ for some $\bR\in\bbO_{r\times r}$. 

		\begin{lemma}\label{lem_Vr_VA}
			Under Assumptions \ref{ass_Z}, \ref{ass_A_general}  and  \ref{ass_E_general}, assume 
			\begin{equation}\label{cond_Sigma_A}
				\sigma_r^2(\bLambda) \ge C\left(\omega_n + \epsilon^2\|\bSigma_E - \bI_p\|_\op\right)
			\end{equation}
			for sufficiently large constant $C>0$. 
			On the event $\cE_X$,
			there exists some $\bR \in \bbO_{r\times r}$ such that  
			\begin{align*}
				\norm{\bV_{(r)} -\bL\bR }{\op} \lesssim ~   {\omega_n +  \epsilon^2 \|\bSigma_E - \bI_p\|_\op \over \sigma_r^2(\bLambda)}.
			\end{align*} 
		\end{lemma}
		
		\begin{proof}[Proof of \cref{thm_PCs}]
			We work on the event  $\cE_X$ in \eqref{def_event_X} which, according to  \cref{lem_Vr_VA}, ensures that 
			\begin{align}\label{bd_Vr_AR} 
				\|\bV_{(r)} - \bL \bR\|_\op   
				&\lesssim {\omega_n + \epsilon^2\|\bSigma_E-\bI_p\|_\op \over \sigma_r^2(\bLambda)}
			\end{align}  
			 for some $\bR\in \bbO_{r\times r}$.
			Write the $r\times r$ matrix
			\[
				\bM := \sigma ~ \bS^{1/2}\bR\bD_{(r)}^{-1/2}\bR^\T \bS^{1/2}.
			\]
			By adding and subtracting terms, the matrix $\bU_{(r)}$  in \eqref{def_PCs} satisfies
			\begin{align*}\nonumber
				\bU_{(r)} &= \bD_{(r)}^{-1/2}\bV_{(r)}^\T \bX \\\nonumber
				&= {1\over \sigma} \bR^\T\bS^{-1/2}  \bM \bS^{-1/2}\bR \bV_{(r)}^\T \bX\\\nonumber
				&= {1\over \sigma}\bR^\T \bS^{-1}\bR \bV_{(r)}^\T \bX + \bR^\T\bS^{-1/2}\left(\bM - \bI_r \right){1\over \sigma}\bS^{-1/2}\bR \bV_{(r)}^\T \bX\\\nonumber
				&= {1\over \sigma}\bR^\T \bS^{-1}\bL^\T \bX+{1\over \sigma}\bR^\T\bS^{-1}\bR(\bV_{(r)}  - \bL \bR)^\T\bX\\\nonumber
				&\qquad + \bR^\T\bS^{-1/2}\left(\bM - \bI_r\right){1\over \sigma}\bS^{-1/2} \bL^\T \bX\\ 
				&\qquad + 
				\bR^\T\bS^{-1/2}\left(\bM- \bI_r\right){1\over \sigma}\bS^{-1/2}\bR(\bV_{(r)}   - \bL \bR)^\T\bX.
			\end{align*}
			By further writing 
			\begin{equation}\label{def_Y_prime}
				\bX' := {1\over \sigma }\bS^{-1}\bL^\T \bX = \bA {1\over \sigma }\bZ + {1\over \sigma }\bS^{-1}\bL^\T \bE := bA {1\over \sigma }\bZ  + \bN
			\end{equation}
			and
			\begin{equation}\label{eq_Delta_DY}
				\begin{split}
					\bDelta &=  \bS^{-1}\bR(\bV_{(r)}  - \bL\bR)^\T\bL\bS  + \bS^{-1/2} \left(\bM - \bI_r\right)\bS^{1/2} \\ 
					&\qquad + 
					\bS^{-1/2}\left(\bM - \bI_r\right) \bS^{-1/2}\bR(\bV_{(r)} - \bL \bR)^\T \bL\bS
				\end{split}
			\end{equation}
			we obtain 
			\begin{equation}\label{eq_DY}
				\bU_{(r)}  = \bR^\T \left(\bX' + \bDelta \bX'\right).
			\end{equation}
			Furthermore, we have 
			\[
				 N_t \sim \cN_r\left(0,  \epsilon^2 \bSigma_N\right),\quad \text{for all }t\in [n],
			\]
			with 
			\begin{equation}\label{def_Sigma_prime} 
				\bSigma_N = \bS^{-1}\bL^\T \bSigma_E \bL \bS^{-1}\quad \text{and}\quad 
				\left\|\bSigma_N\right\|_\op \le  {\|\bSigma_E\|_\op \over  \sigma_r^2(\bLambda)} \le {1\over c_\Lambda^2}.
			\end{equation}
			In view of \eqref{bd_Vr_AR},  
			it remains to  study the term $\bM - \bI_r$ for bounding  from above $\|\bDelta\|_\op$. 
			Since 
			\begin{equation}\label{eq_Dr}
				\begin{split} 
					\bD_{(r)} &= {1\over n}\bV_{(r)}^\T \bX^\T \bX\bV_{(r)} \\
					&=  \bV_{(r)}^\T \bSigma_X \bV_{(r)} + \bV_{(r)}^\T \left(\wh\bSigma_X -\bSigma_X\right)\bV_{(r)} \\
					&= \bV_{(r)}^\T \left(\sigma^2 \bL\bS^2\bL^\T + \epsilon^2\sigma^2\bSigma_E\right) \bV_{(r)} + \bV_{(r)}^\T \left(\wh\bSigma_X -\bSigma_X\right)\bV_{(r)}\\
					&= \sigma^2 \bR^\T \bS^2 \bR + \sigma^2 (\bV_{(r)} - \bL \bR)^\T \bL\bS^2\bL^\T (\bV_{(r)} - \bL \bR) \\
					&\quad  +\sigma^2 (\bV_{(r)} - \bL \bR)^\T \bL\bS^2 \bR+ \sigma^2 \bR^\T \bS^2 \bR \bL^\T (\bV_{(r)} - \bL \bR)\\
					&\quad + \bV_{(r)}^\T \left(\wh\bSigma_X -\bSigma_X\right)\bV_{(r)} + \epsilon^2\sigma^2\bV_{(r)}^\T\bSigma_E \bV_{(r)} ,
				\end{split} 
			\end{equation} 
			rearranging terms yields 
			\begin{align}\label{eq_decomp}
				{1\over \sigma^2}\bS^{-1}\bR \bD_{(r)} \bR^\T \bS^{-1} := \bI_r + {\rm\bf Rem}
			\end{align}
			where  the remainder matrix ${\rm \bf Rem}$ satisfies 
			\begin{align}\label{bd_rem}\nonumber
				\|{\rm\bf Rem}\|_\op &\le   
				\|\bV_{(r)}-\bL\bR\|_\op {\sigma_1(\bLambda) \over \sigma_r(\bLambda)}\left(
				2 + \|\bV_{(r)}-\bL\bR\|_\op {\sigma_1(\bLambda) \over \sigma_r(\bLambda)}
				\right)\\\nonumber
				&\quad   + {\|\wh\bSigma_X - \bSigma_X\|_\op\over  \sigma^2\sigma_r^2(\bLambda)}  +  \epsilon^2  \left\|\bS^{-1}\bR \bV_{(r)}^\T\bSigma_E \bV_{(r)}\bR^\T \bS^{-1}\right\|_\op \\\nonumber
				&\lesssim {\omega_n + \epsilon^2\|\bSigma_E - \bI_p\|_\op \over \sigma_r^3(\bLambda)}  + {\omega_n \over \sigma_r^2(\bLambda)}+  {\epsilon^2 \over \sigma_r^2(\bLambda)}\\
				&\lesssim {\omega_n + \epsilon^2 \over \sigma_r^3(\bLambda)} 
				&&\text{by $\sigma_r(\bLambda)\le 1$}.
			\end{align}
			In the penultimate step, we used $\cE_X$, $\|\bSigma_E-\bI_p\|_\op \le 1$, $\sigma_1(\bLambda)=1$ and 
			\eqref{bd_Vr_AR}. 
			Let $\lambda_1\ge \lambda_2\ge \cdots \ge \lambda_r$ denote the eigenvalues of the left hand side in \eqref{eq_decomp}. 
			Since the eigenvalues of $\bM$ satisfies that
			\begin{align*}
				\lambda_k\left(
				\bM
				\right) = {1\over \sqrt{\lambda_k}},\qquad \text{for all $k\in [r]$.}
			\end{align*}
			By Weyl's inequality and \eqref{eq_decomp}, we further have
			\begin{equation*}
				|1-\lambda_k| \le \|{\rm \bf Rem}\|_\op \lesssim {\omega_n + \epsilon^2 \over \sigma_r^3(\bLambda)},\qquad \lambda_k \ge  1 - \|{\rm \bf Rem}\|_\op \ge 1-C ~{\omega_n + \epsilon^2 \over \sigma_r^3(\bLambda)} \ge c
			\end{equation*}
			using $\sigma_r(\bLambda)\ge c_\Lambda$ and $\omega_n + \epsilon^2 \le c'$.
			The above two displays together with  Weyl's inequality yield
			\begin{align*}
				\left\|\bM  - \bI_r\right\|_\op & = \max_{k\in [r]}\abs{{1\over \sqrt{\lambda_k}}-1}  \le \max_{k\in [r]}{|1-\lambda_k|\over  \lambda_k} \lesssim {\omega_n + \epsilon^2 \over \sigma_r^3(\bLambda)}.
			\end{align*}
			Finally, in view of \eqref{eq_Delta_DY}, we conclude that
			\begin{align*}
				\|\bDelta\|_\op &\le {\|\bV_{(r)} - \bL \bR\|_\op \over \sigma_r(\bLambda)} + {\|\bM - \bI_r\|_\op \over \sigma_r^{1/2}(\bLambda)} + {\|\bV_{(r)} - \bL \bR\|_\op \|\bM - \bI_r\|_\op \over \sigma_r(\bLambda)}\\ 
				& \lesssim  {\omega_n + \epsilon^2\|\bSigma_E - \bI_p\|_\op \over \sigma_r^3(\bLambda)} + {\omega_n + \epsilon^2 \over \sigma_r^{7/2}(\bLambda)} + {\omega_n + \epsilon^2 \over \sigma_r^3(\bLambda)} {\omega_n + \epsilon^2 \|\bSigma_E - \bI_p\|_\op\over \sigma_r^3(\bLambda)}\\
				&\lesssim \omega_n + \epsilon^2.
			\end{align*}
			The proof is complete.
		\end{proof}

		\subsubsection{Proof of \cref{lem:cov_concentration}} 
		
		\begin{proof} 
			Assume $\sigma^2=1$ for simplicity. By the identity 
			\begin{equation*}
				\bSigma_X =  \bLambda\bLambda^\T + \epsilon^2 \bSigma_E,
			\end{equation*}
			and using $\sigma_1(\bLambda) = 1$, we find that
			\begin{align*}
				\norm{\wh\bSigma_X- \bSigma_X}{\op} &\le \norm{\frac{1}{n}\mb E\mb E^\T- \epsilon^2 \bSigma_E}{\op} + \norm{\bLambda\left(\frac{1}{n}\bZ\bZ^\T-  \bI_r\right)\bLambda^\T}{\op}  + {2\over n}\norm{\bLambda\bZ\mb E^\T}{\op}\\
				&\le \norm{\frac{1}{n}\mb E\mb E^\T- \epsilon^2 \bSigma_E}{\op} + \norm{\frac{1}{n}\bZ\bZ^\T-  \bI_r}{\op}  + {2\over n}\norm{\bZ\mb E^\T}{\op}.
			\end{align*}
			Provided that $r\log(n) \le  n$, applying \cref{lem_op_diff} twice with 
			$t = \sqrt{r\log(n)/n}$ and $t = \sqrt{p\log(n)/n}$, separately, gives that, with probability at least $1-n^{-Cr}-n^{-C'p}$, the first two terms in the right hand side of the above display are bounded from above by 
			\begin{equation}\label{bd_diff_X_N_op}
				\epsilon^2 \left(\sqrt{p\log(n)\over n} +  {p\log(n)\over n}\right) + \sqrt{r\log(n)\over n}.
			\end{equation}
			It remains to show that  
			\[
			\P\left\{
			{1\over n}\norm{\bZ\mb E^\T}{\op}  \lesssim \sqrt{\epsilon^2p\log(n)\over n}
			\right\} \ge 1-n^{-C}.
			\]
			When $p\log(n) \gtrsim n$, this follows by noting that 
			\begin{align*}
				{1\over n}\|\bZ\mb E^\T\|_\op^2 &\le  {1\over n}\|\bZ\bZ^\T\|_\op {1\over n}\|\bE\bE^\T\|_\op\\ 
				&\le  \left(1 + \left\| {1\over n}\bZ\bZ^\T - \bI_r\right\|_\op\right) \left(\epsilon^2\|\bSigma_E\|_\op + \left\|{1\over n} \bE\bE^\T - \epsilon^2 \bSigma_E\right\|_\op \right)
			\end{align*}
			together with the bounds in \eqref{bd_diff_X_N_op}, $r\log(n)\lesssim n$ and $\|\bSigma_E\|_\op =1$.
			
			When $p\log(n)\lesssim n$, we have 
			\begin{align*}
				{1\over n}\|\bZ\mb E^\T\|_\op = \sup_{ u\in \Sp^r,  v\in \Sp^p} {1\over n}  u^\T \bZ\bE^\T  v  &\le \max_{ u\in \cN_\eps(r),  v\in \cN_\eps(p)} {2\over n}  u^\T \bZ\bE^\T  v\\
				& = \max_{ u\in \cN_\eps(r),  v\in \cN_\eps(p)} {2\over n}\sum_{t=1}^n  u^\T   Z_t   E_t^\T  v 
			\end{align*}
			where $\cN_\eps(r)$ is the $\eps$-net of $\Sp^r$. Since $ u^\T   Z_t$ 
			is sub-Gaussian with sub-Gaussian constant $1$ and  $ v^\T   E_t\sim \cN(0, \epsilon^2 v^\T\bSigma_E v$) with $ v^\T\bSigma_E v\le \|\bSigma_E\|_\op = 1$ for all $1\le t\le n$, applying the Bernstein inequality (for instance, \cite[Corollary 5.17]{vershynin_2012}) together with union bounds over $\cN_\eps(r)$ and $\cN_\eps(p)$ yields that, for any $t>0$,
			\begin{align*}
				\P\left\{ {1\over n}\norm{\bZ\mb E^\T}{\op} \lesssim   t  \sqrt{\epsilon^2 \over n}\right\}
				& \ge 1 - \abs{\cN_\eps(r)}\abs{\cN_\eps(p)} 2 \exp\left\{
				-c\min\left(
				{t^2 \over n}, {t\over \sqrt n}
				\right)n
				\right\}\\
				&\ge 1 -   2 \exp\left\{
				-c\min\left(
				{t^2 \over n}, {t\over \sqrt n}
				\right)n + (p+r)\log(3/\eps)
				\right\}.
			\end{align*}
			Taking $t = C\sqrt{p\log(n)}$ for some sufficiently large constant $C>0$ completes the proof. 
		\end{proof}

		\subsubsection{Proof of \cref{lem_Vr_VA}}

		\begin{proof}  
			We work on the event $\cE_X$. Start with
			\begin{equation}\label{decomp_Sigma_Y}
				\begin{split}
					\bSigma_X 
					&=  \sigma^2 \bLambda\bLambda^\T + \epsilon^2\sigma^2 \bI_p + \epsilon^2\sigma^2( \bSigma_E - \bI_p)\\
					&=  \bL \left( \sigma^2 \bS^2 + \epsilon^2\sigma^2\bI_r\right) \bL^\T + \epsilon^2\sigma^2  \bL_{\perp}\bL_{\perp}^\T + \epsilon^2\sigma^2(\bSigma_E - \bI_p)
				\end{split}
			\end{equation}
			where we write $\bL_{\perp}$ such that $(\bL,  \bL_{\perp})\in \bbO_{p\times p}$. Write 
			\[
				\bW := \sigma^2 \bL \left(  \bS^2 + \epsilon^2 \bI_r\right) \bL^\T + \epsilon^2\sigma^2  \bL_{\perp}\bL_{\perp}^\T
			\]
			and observe that $\bL$ coincides with the first $r$ eigenvectors of 
			$\bW$. Also note that 
			\[
				\lambda_r(\bW) = \sigma^2\sigma_r^2(\bLambda) + \epsilon^2\sigma^2,\qquad \lambda_{r+1}(\bW) =  \epsilon^2\sigma^2.
			\]
			An application of the Davis-Kahan theorem in \cref{lem_sin_theta} thus yields
			\begin{align*}
				\norm{\bV_{(r)} -\bL \bR}{\op}
				&\le  {2^{3/2}\over  \sigma^2\sigma_r^2(\bLambda)}\|
				\wh\bSigma_X  - \bSigma_X  + \epsilon^2\sigma^2  (\bSigma_E - \bI_p)
				\|_\op\\
				& \le  {2 \over \sigma_r^2(\bLambda)}\left(C \omega_n +  \epsilon^2 \|\bSigma_E - \bI_p\|_\op\right),
			\end{align*}  
			completing the proof. 
		\end{proof}
		
		\subsection{Proof of \cref{thm_A_general}: upper bounds of the estimation error of the proposed PCA deflation varimax estimator}\label{app_proof_thm_A_general}
		
		\begin{proof}
			Without loss of generality, we assume $\bP = \bI_r$. Write 
			\[
			\wt\bLambda := {1\over \sigma}\bV_{(r)}\bD_{(r)}^{1/2}\wc\bQ,\quad \text{such that}\quad  \wh \bLambda = {\wt\bLambda \over \|\wt\bLambda\|_\op}.
			\]
			Recall that $\|\bLambda\|_\op = 1$. By adding and subtracting terms, we find 
			\begin{align*}
				\|\wh\bLambda - \bLambda\|_F &\le  {\left|
					\|\bLambda\|_\op - \|\wt\bLambda\|_\op
					\right| \over \|\bLambda\|_\op\|\wt\bLambda\|_\op}\|\wt\bLambda\|_F  + {\left\|\wt\bLambda - \bLambda \right\|_F\over \|\bLambda\|_\op}\\
				&\le \left|
				1 - \|\wt\bLambda\|_\op
				\right|\sqrt{r} + \|\wt\bLambda - \bLambda\|_F &&\text{by }\|\wt\bLambda\|_F\le \|\wt\bLambda\|_\op\sqrt{r}.
			\end{align*}
			To bound the two terms on the right hand side of the above display, we work on the event that $\cE_X$ in \eqref{def_event_X} and either \cref{thm_RA_rand} or \cref{thm_RA_mrs} holds.

			Regarding the first term, we have 
			\begin{align*}
				\|\wt\bLambda\|_\op = {1\over \sigma}\|\bR\bD_{(r)}^{1/2}\bR^\T\|_\op =  \left\|\bS^{1/2}\left({1\over \sigma}\bS^{-1/2}\bR\bD_{(r)}^{1/2}\bR^\T \bS^{-1/2} - \bI_r\right)\bS^{1/2} + \bS\right\|_\op. 
			\end{align*}
			By $\|\bS\|_\op = 1$ and Weyl's inequality, we conclude 
			\[
			\abs{1- \|\wt\bLambda\|_\op} ~ \le   \left\|{1\over \sigma}\bS^{-1/2}\bR\bD_{(r)}^{1/2}\bR^\T \bS^{-1/2} - \bI_r\right\|_\op.
			\]
			Recall from \eqref{eq_decomp} and \eqref{bd_rem} that, on the event $\cE_X$, 
			\[
			\max_{1\le k\le r}\left|
			1 - \lambda_k\right| \lesssim \omega_n + \epsilon^2
			\]
			where $\lambda_1\ge \cdots \ge \lambda_r$ are the eigenvalues of 
			$
			\sigma^{-1}\bS^{-1}\bR \bD_{(r)}\bR^\T\bS^{-1}.
			$
			We obtain that 
			\begin{equation}\label{bd_diff_A_op}
				\abs{1- \|\wt\bLambda\|_\op} \le \max_{1\le k\le r}  \left|1-\sqrt{\lambda_k}\right| \le  \max_{1\le k\le r}  \left|1-\lambda_k\right| \lesssim \omega_n + \epsilon^2.
			\end{equation}

			For the second term $\|\wt\bLambda - \bLambda\|_F$, by adding and subtracting terms, we have
			\begin{align*}
				\wt\bLambda - \bLambda & = {1\over \sigma}\bV_{(r)}\bD_{(r)}^{1/2}\wc\bQ -  \bL\bS\bA   &&\text{by \eqref{svd_A}}\\
				& = {1\over \sigma}(\bV_{(r)}-\bL\bR)\bD_{(r)}^{1/2}\wc\bQ + 
				{1\over \sigma}\bL\bR \bD_{(r)}^{1/2} (\wc\bQ - \bR^\T \bA)\\
				&\quad +  \bL\bS^{1/2}\left({1\over \sigma}\bS^{-1/2}\bR \bD_{(r)}^{1/2}\bR^\T\bS^{-1/2} -  \bI_r\right)\bS^{1/2} \bA\\
				&\le \|\bV_{(r)}-\bL\bR\|_\op {1\over \sigma}\|\bD_{(r)}^{1/2}\|_F + {1\over \sigma}\|\bD_{(r)}^{1/2}\|_\op \|\wc\bQ - \bR^\T \bA\|_F\\
				&\quad + \|\bS\|_F \max_{1\le k\le r}  \left|1-\sqrt{\lambda_k}\right|
			\end{align*}
			Note that \eqref{bd_diff_A_op} implies
			\[
			{1\over \sigma}\|\bD_{(r)}^{1/2}\|_\op ~  \le  {1\over \sigma_r(\bLambda)}\left(1 + \max_{1\le k\le r}  \left|1-\sqrt{\lambda_k}\right| \right) \lesssim 1.
			\]
			By further invoking \eqref{bd_Vr_AR} and \cref{thm_RA_rand} (or \cref{thm_RA_mrs}) together with 
			$\|\bS\|_F \le \sqrt{r}$ and $\|\bD_{(r)}^{1/2}\|_F \le \|\bD_{(r)}^{1/2}\|_\op \sqrt{r}$, we conclude that
			\begin{align*}
				\|\wt\bLambda -  \bLambda\|_F  \lesssim  (\omega_n + \epsilon^2)\sqrt{r}.
			\end{align*}
			Combining the above display with \eqref{bd_diff_A_op} completes the proof. 
		\end{proof}

		\subsection{Proof of \cref{thm_lowerbounds}: the minimax lower bounds of estimating the loading matrix}\label{app_proof_thm_lowerbounds}

		\begin{proof}
			By taking $\epsilon^2= 0$, the lower bound $\sqrt{r^2 / n}$ is proved in \cite{auddy2023}. It remains to prove the other lower bound $\sqrt{\epsilon^2p / n}$.  
			
			Since the result is invariant to $\sigma^2$, we assume $\sigma^2 = 1$ without loss of generality.  We fix any $\epsilon^2 \ge 0$ below and write $\P_{\bLambda}$  for $\P_{\bLambda,\epsilon}$. 
			Let $c'>0$ be some fixed, absolute constant. Write
			$
			\bD = \diag (d_1,\ldots, d_r)
			$
			with 
			$
			c' \le d_j \le 1
			$
			for all $j\in [r]$. We fix the $\ell_2$-norm of columns in $\bLambda$ by considering the following subspace of $\cA$:
			\[
			\bar\cA := \bar\cA (\bD) = \left\{
			\bLambda \in \cA:  \|\Lambda_i\|_2 = d_i,  \text{ for all }i\in [r]
			\right\}.
			\]
			For any sequence $\beta_n>0$ to be specified later, we have  
			\begin{align*}
				&\inf_{\bar\bLambda \in \bar\cA}\sup_{\bLambda\in \bar\cA}
				\P_{\bLambda}\left\{
				\|\bar\bLambda - \bLambda\|_F \ge {2 \beta_n}
				\right\}\\
				& ~ \le  ~ \inf_{\wh\bLambda}\inf_{\bar\bLambda \in \bar\cA}\sup_{\bLambda\in \bar\cA}
				\P_{\bLambda}\left\{
				\|\bar\bLambda - \wh \bLambda\|_F + \|\wh\bLambda -  \bLambda\|_F \ge {2 \beta_n}
				\right\}\\
				& ~ \le  ~ \inf_{\wh\bLambda}\sup_{\bLambda\in \bar\cA}
				\P_{\bLambda}\left\{
				\|\wh\bLambda - \bLambda\|_F \ge  \beta_n  
				\right\} &&\text{by }\bar\cA\subseteq\cA\\
				& ~ \le  ~ \inf_{\wh\bLambda}\sup_{\bLambda\in \cA}
				\P_{\bLambda}\left\{
				\|\wh\bLambda - \bLambda\|_F \ge \beta_n  
				\right\} &&\text{by $\bar\cA\subseteq\cA$ again}.
			\end{align*}
			It thus suffices to bound from below the first line in the above display. To this end, define the following parameter space 
			\[
			{\mc L} := \left\{
			\bL \in \R^{p\times r}: \rank(\bL) = r,~  \sigma_r(\bL) \ge c, ~ \|L_i\|_2 = 1,  \text{ for all }i\in [r]
			\right\}.
			\]
			Notice that 
			$
			\mc L =  \{\bLambda\bD^{-1}: \bLambda \in \bar\cA\}.
			$
			We find that 
			\begin{align}\label{lb_Lambda_L}\nonumber
				\inf_{\bar\bLambda \in \bar\cA}\sup_{\bLambda\in \bar\cA}
				\P_{\bLambda}\left\{
				\|\bar\bLambda - \bLambda\|_F \ge 2\beta_n
				\right\} & =   \inf_{\bar\bLambda \in \bar\cA}\sup_{\bLambda\in \bar\cA}
				\P_{\bLambda}\left\{
				\left\|\left(\bar\bLambda \bD^{-1}- \bLambda \bD^{-1}\right)\bD \right\|_F \ge 2 \beta_n
				\right\}\\\nonumber
				&\ge \inf_{\bar\bLambda \in \bar\cA}\sup_{\bLambda\in \bar\cA}
				\P_{\bLambda}\left\{
				\left\|\bar\bLambda \bD^{-1}- \bLambda \bD^{-1}\right\|_F \ge {2\beta_n \over   \sigma_r(\bD)}
				\right\}\\
				&= \inf_{\wh\bL \in \mc L}\sup_{\bL \in \mc L}
				\P_{\bL}\left\{
				\|\wh\bL- \bL\|_F \ge {2\beta_n \over   \sigma_r(\bD)}
				\right\}.
			\end{align}
			The parameter space $\mc L$ is considered in \cite{Jung2016}\footnote{Although \cite{Jung2016} did not put $\sigma_r(\bL)\ge c$ in the parameter space, the set of hypotheses $\bD_0,\bD_1,\ldots,\bD_L$ constructed in their proof does satisfy this constraint, see, for instance, display (47) and their upper bound for $\|\bD_l - \bD_0\|_F^2$ in page 1509}. By adopting to our notation,  invoking their Theorem \rom{3.1} with $\mb \bSigma_x = \bD \Cov(Z) \bD  = \sigma^2\bD^2$, $p = r$, $m = p$, $N = n$ and $\sigma^2 = \epsilon^2\sigma^2$ gives
			\[
			\inf_{\wh\bL \in \mc L~}\sup_{\bL \in \mc L}\EE \left[\|\wh\bL-\bL\|_F^2 \right] ~ \gtrsim  ~ 
			{\epsilon^2 pr \over n \sigma_1^2(\bD)}.
			\]
			In fact, by inspecting their proofs (displays (78) -- (79) in conjunction with  display (2.9) in \cite{tsybakov2008introduction}), we also have 
			\[
			\inf_{\wh\bL \in \mc L}\sup_{\bL \in \mc L}\P \left\{\|\wh\bL-\bL\|_F^2  ~ \ge   ~ 
			c {\epsilon^2 pr \over n \sigma_1^2(\bD)}\right\} \ge c'.
			\]
			for some absolute constants $c>0$ and $c'\in (0,1)$. In view of \eqref{lb_Lambda_L}, we complete the proof by taking $\beta_n = \epsilon^2pr / (2n\sigma_1^2(\bD))$ and using $c\le \sigma_r(\bD) \le \sigma_1(\bD) \le 1$.
		\end{proof}

		\subsection{Proof of \cref{thm_PCs_hat}: the decomposition of improved low-dimensional representation}\label{app_proof_thm_PCs_hat}
		
		\begin{proof}
			We work on the event $\cE_X$ in \eqref{def_event_X}. Recall that $d_1\ge \cdots \ge d_p$ are eigenvalues of $\wh\bSigma_X$ and $\epsilon^2\sigma^2$ is the $j$th largest eigenvalue of $\bSigma_X$  for $j > r$. 
			By Weyl's inequality, we have 
			\[
			\left|d_j - \epsilon^2\sigma^2\right| \le \|\wh\bSigma_X - \bSigma_X\|_\op \lesssim \sigma^2 \omega_n,\quad \text{for all }j>r.
			\]
			As a result, we have 
			\begin{equation}\label{rate_epsilon_sigma}
				\left|\wh{\epsilon^2\sigma^2} - \epsilon^2\sigma^2\right| \le \max_{j>r} \left|d_j - \epsilon^2\sigma^2\right| \lesssim \sigma^2\omega_n.   
			\end{equation}
			The rest of the proof largely follows from that of \cref{thm_PCs} with slight modifications. First note that, due to $\bSigma_E =\bI_p$, \eqref{bd_Vr_AR} gets improved to 
			\[
			\|\bV_{(r)} - \bL \bR\|_\op   
			~ \lesssim {\omega_n \over \sigma_r^2(\bLambda)}.
			\]
			Second, display \eqref{eq_Dr} yields
			\begin{align*}
				\wh \bD_{(r)}  
				& = \bD_{(r)} - \wh\epsilon^2\bI_r \\  
				&=   \sigma^2\bR^\T \bS^2 \bR +  \sigma^2 (\bV_{(r)} - \bL \bR)^\T \bL\bS^2\bL^\T (\bV_{(r)} - \bL \bR)  +  \sigma^2(\bV_{(r)} - \bL \bR)^\T \bL\bS^2 \bR\\
				&\quad +   \sigma^2\bR^\T \bS^2 \bR \bL^\T (\bV_{(r)} - \bL \bR)+ \bV_{(r)}^\T \left(\wh\bSigma_X -\bSigma_X\right)\bV_{(r)}  + (\sigma^2\epsilon^2 - \wh{\epsilon^2\sigma^2})\bI_r
			\end{align*}
			so that 
			\begin{equation}\label{eigen_lambda_hat}
				{1\over \sigma^2}\bS^{-1}\bR\wh \bD_{(r)} \bR^\T \bS^{-1} = \bI_r + \wh {\rm\bf Rem}
			\end{equation} 
			where, by using the arguments of proving \eqref{bd_rem}, we have 
			\begin{equation}\label{bd_rem_hat}
				\|\wh {\rm\bf Rem}\|_\op ~ \lesssim ~ {\omega_n \over \sigma_r^3(\bLambda)} + {|\sigma^2\epsilon^2 - \wh{\epsilon^2\sigma^2}|\over \sigma^2} \lesssim \omega_n.
			\end{equation}
			Let $\wh\lambda_1\ge \cdots \ge  \wh\lambda_r$ denote the eigenvalues of the left hand side of \eqref{eigen_lambda_hat}. Then for any $k\in [r]$,
			\[
			|1-\wh\lambda_k|  ~ \lesssim \omega_n,\qquad 
			\wh\lambda_k \ge 1 - \omega_n \ge c.
			\]
			Next, by writing  
			\[
			\wh\bM = \sigma\bS^{1/2}\bR\wh\bD_{(r)}^{-1/2}\bR^\T \bS^{1/2},
			\]
			we similarly have  
			\begin{align*}
				\wh\bU_{(r)} &= \bR^\T(\bX'+\wh \bDelta\bX')
			\end{align*}
			with $\bX' = \bA\bZ/\sigma+\bN$ and 
			\begin{equation}\label{eq_Delta_DY_hat}
				\begin{split}
					\wh \bDelta &=  \bS^{-1}\bR(\bV_{(r)}  - \bL\bR)^\T\bL\bS  + \bS^{-1/2} \left(\wh\bM - \bI_r\right)\bS^{1/2} \\ 
					&\qquad + 
					\bS^{-1/2}\left(\wh\bM - \bI_r\right) \bS^{-1/2}\bR(\bV_{(r)} - \bL \bR)^\T \bL\bS.
				\end{split}
			\end{equation}
			To bound $\|\wh\bDelta\|_\op$, similar arguments of bounding $\|\bDelta\|_\op$ together with   
			\begin{align*}
				\|\wh \bM  - \bI_r\|_\op & \le \max_{k\in[r]} {|1-\wh\lambda_k|\over \wh\lambda_k} \lesssim \omega_n 
			\end{align*}
			yield 
			$\|\wh \bDelta\|_\op  \lesssim \omega_n$, completing the proof. 
		\end{proof}

		\subsection{Proof of \cref{thm_RA_hat}: guarantees of the improved deflation varimax rotation}\label{app_proof_thm_RA_hat}
		
		\begin{proof}
			We first link any iterate  $\{\wh\bq^{(\ell)}\}_{\ell\ge 0}$ in \eqref{iter_PGD_corrected} to any iterate  $\{\wt\bq^{(\ell)}\}_{\ell\ge 0}$ from
			\begin{equation}\label{iter_PGD_corrected_td}
				\wt  Q_k^{(\ell+1)}= P_{\Sp^r}\left(\wt  Q_k^{(\ell)} -\gamma~ \wh\grad  F(\wt  Q_k^{(\ell)}; \wt\bY)\right),\quad \text{for }\ \ell = 0,1,2,\ldots 
			\end{equation}
			where, for any $\bq\in\Sp^r$, 
			\begin{equation*} 
				\wh\grad F(\bq; \wt \bY) := \grad F(\bq; \wt \bY) + \left(
				1 + \bq^\T \bR \wh\bSigma_N \bR^\T \bq 
				\right)\bP_{\bq}^\perp\bR \wh\bSigma_N\bR^\T \bq.
			\end{equation*} 
            Write $$\bar \bSigma_N:=\bR \wh\bSigma_N\bR^\T$$ for simplicity.
			Suppose $\wh\bq^{(\ell)} = \bR^\T \wt \bq^{(\ell)}$ for any $\ell\ge 0$. Recall that 
			\begin{align*}
				\wh\bq^{(\ell+1)}  
				&= {1\over \sqrt{1 + \gamma^2 \|\wh\grad  F(\wh\bq^{(\ell)};\wh\bY)\|_2^2}} \left(\wh\bq^{(\ell)}  - \gamma~ \wh\grad  F(\wh\bq^{(\ell)};\wh\bY)\right),\\
				\wt\bq^{(\ell+1)}  
				&= {1\over \sqrt{1 + \gamma^2 \|\wh\grad F(\wt\bq^{(\ell)};\wt\bY)\|_2^2}} \left(\wt\bq^{(\ell)}  - \gamma~ \wh\grad  F(\wt \bq^{(\ell)};\wt\bY)\right)
			\end{align*}
			In the proof of \cref{lem_connection}, we have argued that 
			\[
			\grad  F(\wh\bq^{(\ell)}; \wh\bY) = \bR^\T \grad   F(\wt\bq^{(\ell)};\wt\bY)
			\]
			so that 
			\begin{align*}
				&\wh \grad  F(\wh\bq^{(\ell)}; \wh\bY)\\ &= \bR^\T \grad F(\wt\bq^{(\ell)};\wt\bY)+ \bP_{\wh\bq^{(\ell)}}^\perp \left(
				1 + \wh\bq^{(\ell)\T} \wh\bSigma_N \wh\bq^{(\ell)}
				\right)\wh\bSigma_N\wh\bq^{(\ell)}\\
				&= \bR^\T \grad F(\wt\bq^{(\ell)};\wt\bY)+ \bP_{\wh\bq^{(\ell)}}^\perp \left(
				1 + \wt \bq^{(\ell)\T} \bar \bSigma_N\wt\bq^{(\ell)}
				\right)\wh\bSigma_N\bR^\T \wt\bq^{(\ell)} &&\text{by $\wh\bq^{(\ell)} = \bR^\T \wt \bq^{(\ell)}$}\\
				&= \bR^\T \grad F(\wt\bq^{(\ell)};\wt\bY) + \bR^\T\bP_{\wt\bq^{(\ell)}}^\perp \left(
				1 + \wt \bq^{(\ell)\T} \bar \bSigma_N\wt\bq^{(\ell)}
				\right)\bar \bSigma_N\wt\bq^{(\ell)} &&\text{by \eqref{ident_P_perp}}\\
				&= \bR^\T \wh\grad F(\wt\bq^{(\ell)};\wt\bY).
			\end{align*} 
			This implies that $\wh\bq^{(\ell+1)} = \bR^\T \wt \bq^{(\ell+1)}$. Therefore, it suffices to analyze the solution obtained from \eqref{iter_PGD_corrected_td}. By inspecting the proof of \cref{app_thm_critical}, we see that the same arguments can be used provided that, for any $\bq\in\Sp^r$,
			\begin{equation}\label{def_event_tilde_bar} 
				\|\wh \grad F(\bq; \wt \bY) - \grad h(\bq)\|_2    ~ \lesssim ~\eta_n+\alpha_n +   \|\grad F(\bq; \bY) - \grad f(\bq)\|_2.
		\end{equation} 
			Furthermore, the same arguments of proving Theorems \ref{thm_one_col} \& \ref{thm_RA} can be used with  $ \grad F(\bq; \wt \bY)$ and $ \grad  F(\bq)$ replaced by  $ \wh\grad F(\bq; \wt \bY)$ and $ \wh\grad  F(\bq)$, respectively. 
			
			It remains to verify \eqref{def_event_tilde_bar}. Pick any $\bq\in\Sp^r$. Recall $\grad f$ and $\grad h$ from \eqref{eq_grad_f} -- \eqref{eq_grad_h}. Observe that 
			\begin{align*}
				&\|\wh \grad F(\bq; \wt \bY) - \grad h(\bq)\|_2\\
				&\le \|\grad F(\bq; \wt \bY) - \grad f(\bq)\|_2 +\|\bP_{\bq}^\perp \left(
				1 + \bq^\T \bar \bSigma_N \bq 
				\right)\bar \bSigma_N\bq + \grad f(\bq)- \grad h(\bq)\|_2\\
				&\le  \|\grad F(\bq; \wt \bY) - \grad f(\bq)\|_2  +\left\|\bP_{\bq}^\perp \left(
				1 + \bq^\T \bar \bSigma_N \bq 
				\right)\bar \bSigma_N \bq - \left( 1 +  \norm{\bq}{\bSigma_N}^2\right)   \bP_{\bq}^\perp \bSigma_N \bq \right\|_2.
			\end{align*}
			In the proof of \cref{lem_bd_grad_F_tilde} we have shown that, on the event $\cE_\Omega(\eta_n)$,
			\[
		 		\|\grad F(\bq; \wt \bY) - \grad f(\bq)\|_2 \lesssim \eta_n +  \|\grad F(\bq;\bY) - \grad f(\bq)\|_2.
			\] 
			It remains to bound the second term from above by 
			\begin{align*}
				&\left\| \left(
				1 + \bq^\T \bar \bSigma_N \bq 
				\right)\bP_{\bq}^\perp \bar \bSigma_N  \bq - \left( 1 + \bq^\T  \bSigma_N \bq \right) \bP_{\bq}^\perp  \bSigma_N  \bq \right\|_2\\
				&\le \left|\bq^\T \left(\bar \bSigma_N   - \bSigma_N\right)\bq\right| 
				\left\|\bP_{\bq}^\perp \bar \bSigma_N \bq\right\|_2  + \left( 1 + \bq^\T  \bSigma_N \bq \right) \left\|  \bP_{\bq}^\perp (\bar \bSigma_N  -   \bSigma_N) \bq \right\|_2\\
				&\le \left\|
				\wh\bSigma_N  - \bar \bSigma_N
				\right\|_\op \|\wh\bSigma_N\|_\op  + 2 \left\|  \wh\bSigma  - \bar \bSigma_N\right\|_\op\\
				&\le 4 \alpha_n
			\end{align*}
			where in the last step we use  the event $\cE_{\bSigma}(\alpha_n)$ together with
			\[
			\|\wh\bSigma_N\|_\op \le \alpha_n +  \|\bSigma_N\|_\op \le 2.
			\]  
			The proof is complete.
		\end{proof}

		\subsection{Proof of \cref{thm_A_corr}: improved estimation of the loading matrix}\label{app_proof_thm_A_corr}
		
		\begin{proof}
			In view of \cref{thm_RA_hat}, it suffices to show 
			$
				\|\wh\bSigma_N - \bR^\T  \bSigma_N \bR
				\|_\op  \lesssim \omega_n 
			$
			 on the event $\cE_X$ given by \eqref{def_event_X}.
			Start with the decomposition
			\begin{align*}
				\wh\bSigma_N - \bR^\T  \bSigma_N \bR & = \wh{\epsilon^2\sigma^2} \wh\bD_{(r)}^{-1} -  \bR^\T\epsilon^2 \bS^{-2}\bR\\
				& = \left(\wh{\epsilon^2\sigma^2}- \epsilon^2\sigma^2\right)\wh\bD_{(r)}^{-1} + \epsilon^2 \left(
				\sigma^2\wh\bD_{(r)}^{-1} - \bR^\T \bS^{-2}\bR
				\right)\\
				& = {1\over \sigma^2}\left(\wh{\epsilon^2\sigma^2}- \epsilon^2\sigma^2\right)\sigma^2\wh\bD_{(r)}^{-1} + \epsilon^2\bR^\T \bS^{-1}\left(
				\sigma^2\bS \bR \wh\bD_{(r)}^{-1} \bR^\T \bS - \bI_r
				\right)\bS^{-1} \bR.
			\end{align*} 
			Recall that \eqref{eigen_lambda_hat} and \eqref{bd_rem_hat} give
			\[
			{\sigma_r(\wh\bD_{(r)})\over \sigma^2} \ge 1 - \|\wh {\rm\bf Rem}\|_\op \gtrsim 1- \omega_n\ge c.
			\]
			Using \eqref{rate_epsilon_sigma}, \eqref{eigen_lambda_hat} and \eqref{bd_rem_hat}  yields 
			$\|\wh\bSigma_N - \bR^\T  \bSigma_N \bR
			\|_\op  \lesssim \omega_n$, completing the proof. 
%
		\end{proof}

		\section{Concentration inequalities related with the $\ell_4$ minimization}
		
		Recall from \eqref{model_Y_hat} that $\bY = \bA\bZ+\bN$ with $\bA$, $\bZ$ and $\bN$ satisfying the specifications therein. Columns of $\bY$, $\bZ$ and $\bN$ are i.i.d. copies of the random vectors $Y, Z$ and $N$, respectively. 
		Throughout this section, we write
		\begin{equation}\label{def_Sigma_Y}
			\bSigma_Y := \Cov(Y) = \Cov(Z) + \Cov(N) =  \bI_r + \bSigma_N.
		\end{equation}

		\subsection{A calculation of the 4th order moments of columns in $\bY$}
		
			For  any fixed $\bM \in \R^{r\times r}$, the following lemma computes the expectation of $Y^\T \bM Y  Y Y^\T$ under \eqref{model_Y_hat}. In particular, it implies $f(q) = \EE[F(q)]$ with $f$ and $F$ given in \eqref{eq_f} and \eqref{obj_L4_Y}, respectively. 
		
			\begin{lemma}\label{lem:obj_exp} 
			Grant model \eqref{model_Y_hat} with $\bSigma_Y$ given by  \eqref{def_Sigma_Y}. For any fixed matrix $\bM \in \R^{r\times r}$, we have that 
			\begin{align*}
				\EE\left[ Y^\T \bM  Y   ~  Y  Y^\T   \right] &= 3\kappa  \sum_{i=1}^r  A_i^\T \bM  A_i  A_i  A_i^\T    +    \tr(\bM \bSigma_Y)  \bSigma_Y +     \bSigma_Y    (\bM + \bM^\T)\bSigma_Y.
			\end{align*}
			Furthermore, for any fixed $\bq\in \Sp^r$, we have $f(\bq) = \EE[F(\bq;\bY)]$.
		\end{lemma}      
		\begin{proof} 
			We only need to prove the first statement as the second one follows immediately. 
			Pick any $u,v \in \Sp^r$.
			By $Y = \bA   Z +  N$, the independence between $  Z$ and $ N$ as well as $\EE[  Z] =  0 =\EE[ N]$, we find that
			\begin{align*}
				\EE\left[ Y^\T \bM  Y   u^\T  Y  Y^\T  v \right] &= \rI + \rII + \rIII + \rIV 
			\end{align*}
			where 
			\begin{align*}
				& \rI  = \EE\left[  Z^\T \bA^\T \bM \bA   Z   u^\T \bA  Z   Z^\T \bA^\T  v \right]+\EE\left[ N^\T \bM  N   u^\T  N  N^\T  v \right]\\
				&\rII = \EE[  Z^\T \bA^\T \bM \bA   Z  u^\T  N N^\T  v] + \EE[ N^\T \bM N  u^\T \bA   Z   Z^\T \bA^\T  v]\\
				&\rIII =  \EE[  Z^\T \bA^\T \bM  N  u^\T \bA   Z  N^\T  v] + \EE[ N^\T \bM \bA   Z  u^\T   \bA   Z  N^\T  v]\\
				&\rIV =   \EE[  Z^\T \bA^\T \bM  N  u^\T  N   Z^\T \bA^\T  v] + \EE[ N^\T \bM \bA   Z  u^\T    N   Z Z^\T \bA^\T  v].
			\end{align*}
			We proceed to calculate each term separately.
			
			For $\rII$, using $\bA \in \bbO_{r\times r}$, $\EE[  Z  Z^\T]= \bI_r$ and $\EE[ N N^\T] =   \bSigma_N$ from \cref{ass_Z} and \cref{ass_E_general} gives that 
			\[
			\rII =   u^\T \bSigma_N  v  ~ \tr(\bM)+    u^\T  v ~ \tr(\bM\bSigma_N).
			\]
			The terms $\rIII$ and $\rIV$ can be calculated similarly. Indeed,  since
			\begin{align*}
				\EE[  Z^\T \bA^\T \bM  N  u^\T \bA   Z  N^\T  v]  &= \EE[  Z^\T \bA^\T \bM  N  N^\T  v   u^\T \bA   Z] = 
				u^\T \bM  \bSigma_N  v,
			\end{align*}
			using analogous arguments gives 
			\[
			\rIII + \rIV =  
			u^\T \left( \bM  \bSigma_N + \bSigma_N \bM + \bM^\T \bSigma_N + \bSigma_N \bM^\T\right)  v.
			\]
			Finally, to calculate $\rI$, by writing $\bM' = \bA^\T \bM \bA$, $ u' = \bA^\T  u$ and $ v' = \bA^\T  v$, we have 
			\begin{align*}
				& \EE\left[  Z^\T \bA^\T \bM \bA   Z   u^\T \bA  Z   Z^\T \bA^\T  v \right]\\
				&= \sum_{i=1}^r M_{ii}' u_i' v_i' \EE[Z_i^4] + \sum_{i\ne j} M_{ii}' u_j' v_j' \EE[Z_i^2Z_j^2]  + \sum_{i\ne j} M_{ij}' u_i' v_j' \EE[Z_i^2Z_j^2] + \sum_{i\ne j} M_{ij}' u_j' v_i' \EE[Z_i^2Z_j^2]\\
				&= 3\kappa  \sum_{i=1}^r  A_i^\T \bM  A_i  u^\T  A_i  A_i^\T  v + \tr(\bM)  u^\T  v +  u^\T \bM  v +  v^\T \bM  u.
			\end{align*}
			In the last step, we used the excess kurtosis of $  Z_{it}$ in \cref{ass_Z}. 
			By repeating the same arguments and by writing $\bM' = \bSigma_N^{1/2} \bM \bSigma_N^{1/2}$, $ u' = \bSigma_N^{1/2} u$ and $ v' = \bSigma_N^{1/2}  v$, we have 
			\begin{align*}
				\EE\left[ N^\T \bM  N   u^\T  N  N^\T  v \right] &=    \tr(\bM \bSigma_N )  u^\T \bSigma_N  v +  u^\T \bSigma_N  \bM \bSigma_N  v +  v^\T \bSigma_N  \bM \bSigma_N  u,
			\end{align*}
			implying that 
			\begin{align*}
				\rI &= 3\kappa  \sum_{i=1}^r  A_i^\T \bM  A_i  u^\T  A_i  A_i^\T  v + \tr(\bM)  u^\T  v +  u^\T \bM  v +  v^\T \bM  u\\
				&\quad +   \tr(\bM \bSigma_N )  u^\T \bSigma_N  v +  u^\T \bSigma_N  \bM \bSigma_N  v +  v^\T \bSigma_N  \bM \bSigma_N  u.
			\end{align*} 
			Combining the expressions of $\rI$, $\rII$, $\rIII$ and $\rIV$ yields
			\begin{align*}
				\EE\left[ Y^\T \bM  Y   u^\T  Y  Y^\T  v \right] &= 3\kappa  \sum_{i=1}^r  A_i^\T \bM  A_i  u^\T  A_i  A_i^\T  v + \tr(\bM)  u^\T  v +  u^\T \bM  v +  v^\T \bM  u\\
				&\quad +   \tr(\bM \bSigma_N )  u^\T \bSigma_N  v +  u^\T \bSigma_N  \bM \bSigma_N  v +  v^\T \bSigma_N  \bM \bSigma_N  u \\
				&\quad +     u^\T \bSigma_N  v  ~ \tr(\bM)+   u^\T  v ~ \tr(\bM\bSigma_N)\\
				&\quad +  
				u^\T \left( \bM  \bSigma_N + \bSigma_N \bM + \bM^\T \bSigma_N + \bSigma_N \bM^\T\right)  v
			\end{align*}
			thereby completing the proof after rearranging terms. 
		\end{proof}

		\subsection{Concentration inequalities of quantities related with the 4th moments of $Y$}\label{app_sec_concentration}
		
		In this section we state several important  concentration inequalities related with the $4$th moments of columns in $\bY = (Y_1, \ldots, Y_t)\in \R^{r\times n}$.  They in particular imply various controls of  the Riemannian gradient $\|\grad F(q) - \grad f(q)\|_2$ and the objective functions $|F(q) - f(q)|$.  Throughout this section, with $\gamma_Z$ defined in \cref{ass_Z}, let 
		\begin{equation}\label{def_sigma_Y}
			\sigma_Y^2 = \gamma_Z^2 +   \|\bSigma_N\|_\op
		\end{equation} 
		denote the sub-Gaussian constant of columns of $\bY$.

		The following lemma bounds from above $\|\sum_{t=1}^n[ (\bq^\T  Y_t)^3   Y_t  - \EE[(\bq^\T  Y_t)^3   Y_t  ]]\|_2$ for any fixed $\bq \in \Sp^r$, thereby providing a bound of $\norm{\grad F(\bq) - \grad f(\bq)}2$. 
		
		\begin{lemma}\label{lem_bd_grad_pointwise}
			Grant model \eqref{model_Y_hat}.   For any fixed $\bq \in \Sp^r$, with probability at least $1- n^{-C}$,  
			\begin{align*}
				{1\over    n}\left\|\sum_{t=1}^n\left[ (\bq^\T  Y_t)^3   Y_t  - \EE[(\bq^\T  Y_t)^3   Y_t  ]\right] \right\|_2 & \lesssim  \sqrt{r\log(n)\over n}.
			\end{align*}
		\end{lemma}
		\begin{proof}
			Fix $\bq \in \Sp^r$. For all $t\in [n]$, we know that 
			$ Y_t$ is sub-Gaussian  with sub-Gaussian constant $\sigma_Y$ given by \eqref{def_sigma_Y}.  
			Consider the event $ \cE = \bigcap_{t=1}^n \cE_t$ with $ \cE_t := \cE_{t1} \cap \cE_{t2}$ and 
			\begin{equation}\label{def_event_Y}
				\cE_{t1} = \left\{
				| \bq^\T  Y_t | \le 2\sigma_Y\sqrt{\log (n)}
				\right\}, \quad \cE_{t2} = \left\{\| Y_t\|_2 \le   \sigma_Y\sqrt{r+4\log(n)}\right\}.
			\end{equation}
			Invoking \cref{lem_quad} and taking the union bounds over $t\in [n]$ ensure that
			$\P(\cE) \ge 1 - 1/n$. (Note that this probability can be made $1-n^{-C}$ for some large constant $C>0$ with $\log(n)$ in \eqref{def_event_Y} replaced by $C\log(n)$). 
			
			In the rest of the proof, we work on  the event $\cE$ and bound from above  
			\[
			{1\over n}\left\|\sum_{t=1}^n\left[ (\bq^\T  Y_t)^3   Y_t 1\{\cE_t\}   - \EE[(\bq^\T  Y_t)^3   Y_t  1\{\cE_t\} ]\right] \right\|_2 + {1\over n}\left\|\sum_{t=1}^n \EE[(\bq^\T  Y_t)^3   Y_t  1\{\cE_t^c\}] \right\|_2.
			\]
			Note that the second term is no greater than 
			\begin{align}\label{bd_grad_pointwise_small_event}\nonumber
				\max_{t\in [n]}\sup_{ u \in \Sp^r} \EE\left[(\bq^\T  Y_t)^3   u^\T  Y_t  1\{\cE_t^c\}\right] &\le 
				\sup_{ u \in \Sp^r}\sqrt{\EE[(\bq^\T  Y_t)^6 ( u^\T  Y_t)^2]} ~  \max_{t\in [n]} \sqrt{\P(\cE_t^c)}\\
				&\lesssim  {\sigma_Y^4 \over \sqrt n}
			\end{align}
			where the last step invokes \cref{lem_moment} with $W = W_t$ and $\sigma_W = \sigma_Y$. We proceed to  apply the vector-valued Bernstein inequality in \cref{lem_bernstein_vector} to bound the first term. Note that 
			$$
			\| (\bq^\T  Y_t)^3   Y_t 1\{\cE_t\} \|_2 ~ \lesssim~  \sigma_Y^4 \sqrt{r + \log(n)}
			$$
			almost surely. Further notice that 
			\begin{align*}
				\EE \left[  \left\|(\bq^\T  Y_t)^3   Y_t  1\{\cE_t\}\right\|_2^2 \right] &\le \EE\left[  (\bq^\T  Y_t)^6  Y_t^\T  Y_t \right]\\
				&\le \sqrt{\EE[ (\bq^\T  Y_t)^{12}]}\sqrt{\EE[\| Y_t\|_2^4]}\\
				&\lesssim \sigma_Y^8  ~ r.
			\end{align*}
			The last step applies \cref{lem_moment} with $W = Y_t$, $\sigma_W = \sigma_Y$, $a = 12$,  $b = 0$ and uses
			\begin{equation}\label{bd_Y_four}
					\EE[\| Y_t\|_2^4] = \sum_{i,j = 1}^r \EE[Y_{it}^2 Y_{jt}^2] \lesssim \sigma_Y^4 r^2.
			\end{equation}
			Apply \cref{lem_bernstein_vector} with $\sigma^2 = n r \sigma_Y^8$, $U = \sigma_Y^4 \sqrt{r + \log(n)}$ and 
			$$
			t = {1\over 6}\sqrt{\log(n)}(U + \sqrt{U^2 + 36\sigma^2}) \asymp \sigma_Y^4 \sqrt{n r\log(n)}  
			$$ 
			to obtain 
			\[
			{1\over n}\left\|\sum_{t=1}^n\left[ (\bq^\T  Y_t)^3   Y_t 1\{\cE_t\}   - \EE[(\bq^\T  Y_t)^3   Y_t  1\{\cE_t\} ]\right] \right\|_2 \le \sigma_Y^4\sqrt{r\log(n)\over n}
			\]
			with probability at least 
			\[
			1 - 28\exp\left(
			- {t^2/2\over  n r \sigma_Y^8 + \sigma_Y^4 t \sqrt{r + \log(n)}  / 3}
			\right)  \ge 1 - n^{-C}.
			\]
			In conjunction with \eqref{bd_grad_pointwise_small_event}, the proof is complete. 
		\end{proof}

		 The lemma below bounds from above $\|\sum_{t=1}^n( \bq^\T  Y_t \bar \bq^\T Y_t  Y_t   Y_t^\T  - \EE[\bq^\T  Y_t \bar \bq^\T Y_t    Y_t  Y_t^\T ])\|_\op$ for any fixed $\bq, \bar \bq \in \Sp^r$. 
		 
		 \begin{lemma}\label{lem_bd_YY_pointwise}
		 	Grant model \eqref{model_Y_hat}. Assume $n \ge  r$.   For any fixed $\bq, \bar \bq\in \Sp^r$, with probability at least $1- n^{-C}$,  
		 	\begin{align*}
		 		{1\over    n}\left\|\sum_{t=1}^n\left(  \bq^\T  Y_t \bar \bq^\T  Y_t   Y_t Y_t^\T  - \EE\left[ \bq^\T  Y_t \bar \bq^\T  Y_t   Y_t Y_t^\T  \right] \right)\right\|_\op & \lesssim  \sqrt{r\log(n)\over n}.
		 	\end{align*}
		 \end{lemma}
		 	\begin{proof}
		 		Let $\otimes$ denote the Kronecker product and define 
		 		$$
		 		H_t :=  Y_t \otimes  Y_t \in \R^{r^2},\quad \text{for each $t\in [n]$.}
		 		$$   
		 		In  \cref{fact_k_prod} we review some basic properties of the Kronecker product. 
		 		By using the fact that 
		 		\begin{equation}\label{eq_kronecker}
		 			\bq^\T  Y_t  ~ \bar \bq^\T  Y_t = (\bq \otimes \bar\bq)^\T ( Y_t\otimes  Y_t) = (\bq \otimes \bar\bq)^\T H_t
		 		\end{equation}
	 			from part (3) of \cref{fact_k_prod},
		 		we obtain that 
		 		\begin{align*}
		 			&\left\|\sum_{t=1}^n\left(  \bq^\T  Y_t \bar \bq^\T  Y_t   Y_t Y_t^\T  - \EE\left[ \bq^\T  Y_t \bar \bq^\T  Y_t   Y_t Y_t^\T  \right] \right)\right\|_\op \\
		 			&= \sup_{ u \in \Sp^r} \sum_{t=1}^n\left(  \bq^\T  Y_t \bar \bq^\T Y_t  (u^\T Y_t)^2     - \EE\left[ \bq^\T  Y_t \bar \bq^\T  Y_t   (u^\T Y_t )^2 \right] \right)\\
		 			&= \sup_{ u \in \Sp^r}   \sum_{t=1}^n (\bq \otimes u)(  H_t H_t^\T - \EE[H_tH_t^\T])  (\bar \bq \otimes u).
		 		\end{align*}
	 			Thus by Cauchy Schwarz inequality, it remains to bound from above 
	 			\begin{align*}
	 				&\sup_{ u \in \Sp^r}   \sum_{t=1}^n (\bq \otimes u)(  H_t H_t^\T - \EE[H_tH_t^\T])  (\bq \otimes u) \\ &= \sup_{ u \in \Sp^r} \sum_{t=1}^n\left(  (\bq^\T  Y_t)^2  (u^\T Y_t)^2     - \EE\left[(\bq^\T  Y_t)^2  (u^\T Y_t)^2  \right] \right)\\
	 				&= \left\|\sum_{t=1}^n\left(  (\bq^\T  Y_t)^2  Y_t Y_t^\T  - \EE\left[ (\bq^\T  Y_t)^2 Y_t Y_t^\T  \right] \right)\right\|_\op
	 			\end{align*}
		 		To this end, we use a similar truncation argument as the proof of \cref{lem_bd_grad_pointwise}. Recall $\cE_t$ from \eqref{def_event_Y}. We work on $\cap_{t=1}^n \cE_t$ to bound from above 
		 		  \begin{align*}
		 		  		& {1\over    n}\left\|\sum_{t=1}^n\left(  (\bq^\T  Y_t)^2   Y_t Y_t^\T 1\{\cE_t\} - \EE\left[(\bq^\T  Y_t)^2   Y_t Y_t^\T 1\{\cE_t\} \right] \right)\right\|_\op \\
		 		  		&\qquad + 
		 		  		{1\over    n}\left\|\sum_{t=1}^n  \EE\left[(\bq^\T  Y_t)^2   Y_t Y_t^\T 1\{\cE_t^c\} \right]\right\|_\op.
		 		  \end{align*} 
	 		  	By similar argument in the proof of \eqref{bd_grad_pointwise_small_event}, 
	 		  	the second term is no greater than 
	 		  	\begin{align}\label{bd_YY_pointwise_small_event}
	 		  		\sup_{ u \in \Sp^r}\sqrt{\EE[(\bq^\T  Y_t)^4 ( u^\T  Y_t)^4]} \max_{t\in [n]} \sqrt{\P(\cE_t^c)} ~ \lesssim  ~ {\sigma_Y^4 \over \sqrt n}.
	 		  	\end{align}
 		  		We proceed to apply the matrix-valued Bernstein inequality  in \cref{lem_bernstein_mat}  to bound the first term. Note that
 		  		$$
 		  		\| (\bq^\T  Y_t)^2   Y_t  Y_t^\T 1\{\cE_t\} \|_2 ~ \lesssim~  \sigma_Y^4 ( r + \log(n))
 		  		$$
 		  		almost surely. Further notice that 
 		  		\begin{align*}
 		  			& \left\| \EE \left[ (\bq^\T  Y_t)^4   Y_t  Y_t^\T Y_t Y_t^\T 1\{\cE_t^c\}  \right]\right\|_\op\\  &\le \sup_{ u \in \Sp^r} \EE\left[  (\bq^\T  Y_t)^4  (u^\T  Y_t)^2 \|Y_t\|_2^2 \right]\\
 		  			&\le  \sqrt{\EE\left[\| Y_t\|_2^4\right]} \sup_{ u \in \Sp^r}\sqrt{\EE[ (\bq^\T  Y_t)^8 (u^\T  Y_t)^4]} \\
 		  			&\lesssim \sigma_Y^8  ~ r &&\text{by \cref{lem_moment} \& \eqref{bd_Y_four}}.
 		  		\end{align*}
 	  			Thus, using $n \ge  r $ and applying \cref{lem_bernstein_mat} with $d = r$, $\sigma^2 = n r \sigma_Y^8$, $U = \sigma_Y^4 (r + \log(n))$ and 
 		  		$ 
 		  		t = C\sigma_Y^4 \sqrt{n r\log(n)}  
 		  		$
 		  		yield 
 		  		\[
 		  		{1\over n}\left\|\sum_{t=1}^n\left( (\bq^\T  Y_t)^2   Y_t Y_t^\T 1\{\cE_t\}   - \EE\left[(\bq^\T  Y_t)^2   Y_t  Y_t 1\{\cE_t\} \right]\right) \right\|_2 \lesssim  \sigma_Y^4\sqrt{r\log(n)\over n}
 		  		\]
 		  		with probability at least 
 		  		\[
 		  		1 - 14\exp\left(
 		  		- {C^2\sigma_Y^8 n r \log(n) \over \sigma_Y^8  n r  + \sigma_Y^8 (r + \log(n)) 
 		  		C \sqrt{nr\log(n)}} + \log(r)
 		  		\right)  \ge 1 - n^{-C}.
 		  		\]
 		  		In conjunction with \eqref{bd_grad_pointwise_small_event}, the proof is complete. 
		 	\end{proof}

		The following lemma bounds from above $ \sum_{t=1}^n  ((\bq ^\T Y_t)^4 - \EE[(\bq ^\T Y_t)^4
		])$ uniformly over $q\in \Sp^r$, thereby establishing uniform upper bounds of $|F(q) - f(q)|$.

		\begin{lemma}\label{lem_HH}
			Under model \eqref{model_Y_hat}, assume $n\ge  r(r + \log n)$. With probability at least $1- n^{-1}$, we have 
			\[
				\sup_{q \in \Sp^r}{1\over n}\left| \sum_{t=1}^n \left(
				(\bq ^\T Y_t)^4 - \EE\left[(\bq ^\T Y_t)^4
				\right]\right)\right|~  \lesssim \sqrt{r^2 \log(n)\over n}.
			\] 
		\end{lemma}
		\begin{proof}
			Recall that 
			$
			H_t :=  Y_t \otimes  Y_t \in \R^{r^2}
			$  
			for each $t\in [n]$.  Write its centered version as 
			 \begin{equation}\label{def_bar_Ht}
			 	\bar H_t = H_t - \EE[H_t].
			\end{equation}
			 By adding and subtracting terms, it is easy to verify that for any $\bq \in \Sp^r$,
			 \begin{align*}
			 	&	{1\over n} \sum_{t=1}^n \left\{
			 		(\bq ^\T Y_t)^4 - \EE\left[(\bq ^\T Y_t)^4
			 		\right]\right\}\\
			 		&=   {1\over  n}  \sum_{t=1}^n  (\bq \otimes \bq)^\T \left( H_t  H_t^\T   - \EE[ H_t  H_t^\T]\right)(\bq \otimes \bq)\\
			 		&\le {1\over  n}  \sum_{t=1}^n  (\bq \otimes \bq)^\T \left( \bar H_t \bar H_t^\T   - \EE[ \bar H_t \bar H_t^\T]\right)(\bq \otimes \bq) + {2\over n} \left|
			 		 (\bq \otimes \bq)^\T \EE[H_t] 
			 		\right| \left|
			 			\sum_{t=1}^n  (\bq \otimes \bq)^\T \bar H_t
			 		\right|
			 \end{align*}
			 so that it suffices to bound from above 
			 \begin{align*}
			 	&\rI = {1\over  n} \left\| \sum_{t=1}^n   \left( \bar H_t \bar H_t^\T   - \EE[ \bar H_t \bar H_t^\T]\right)\right\|_\op,\\
			 &	\rII  = \sup_{q \in \Sp^r} {1\over n} \left|
			 	(\bq \otimes \bq)^\T \EE[H_t] 
			 	\right| \left|
			 	\sum_{t=1}^n  (\bq \otimes \bq)^\T \bar H_t
			 	\right|.
			 \end{align*}
			 
			 To bound term $\rII$, recall $\bSigma_Y =\Cov(Y_t)$ from \eqref{def_Sigma_Y}.
			 By part (4) of \cref{fact_k_prod}, we note that 
			 \begin{equation}\label{exp_Ht}
			 	\EE[H_t] = \EE[Y_t \otimes Y_t] = \rvec(\bSigma_Y) \in \R^{r^2}.
			 \end{equation}
			 By further using  part (3) of \cref{fact_k_prod} and $\|\bSigma_Y\|_\op \le 1 + \|\bSigma_N\|_\op \le 2$, we find
			 \[
			 	\sup_{\bq \in \Sp^r}	(\bq \otimes \bq)^\T \EE[H_t]  = 		\sup_{\bq \in \Sp^r} (\bq \otimes \bq)^\T\rvec(\bSigma_Y) = 
			 		 	\sup_{\bq \in \Sp^r} \bq^\T  \bSigma_Y  \bq  \le 2.
			 \]
			 Moreover, for any $\bq \in \Sp^r$, we have 
			 \[
			 	 (\bq \otimes \bq)^\T \bar H_t = (\bq \otimes \bq)^\T\left(
			 	 Y_t\otimes Y_t - \rvec(\bSigma_Y)
			 	 \right) = q^\T \left(
			 	 	Y_t Y_t^\T - \bSigma_Y
			 	 \right)q
			 \]
			 by part (3) of \cref{fact_k_prod} again. 
			 The above two displays imply that 
			 \[
			 	\rII \le    \sup_{q \in \Sp^r}  {2\over n} \sum_{t=1}^n q^\T \left(
			 	Y_t Y_t^\T - \bSigma_Y
			 	\right)q  = 2 \left\|  \bY \bY^\T - \bSigma_Y \right\|_\op.
			 \]
			 By $n \ge r\log(n)$, invoking \cref{lem_op_diff} with  $\bG = \bY^\T$,  $\bSigma_G = \bSigma_Y$, $d = r$ and $t = C\sqrt{r\log(n)}$  gives that , with probability at least $1- 2n^{-cr}$, 
			 \begin{equation}\label{bd_II_HH} 
			 		\rII \lesssim \sqrt{r\log(n) \over n}.
			 \end{equation}
			 
			To bound the term $\rI$, recall the event $\cE_{t2}$ from \eqref{def_event_Y} and $\P(\cap_{t=1}^n \cE_{t2}) \ge 1-1/n$.   Using part (3) of \cref{fact_k_prod} gives the identity 
			 \begin{align}\label{eq_Ht_Yt}
			 	\| H_t\|_2 = \sqrt{( Y_t \otimes  Y_t)^\T ( Y_t \otimes  Y_t)} = \sqrt{( Y_t^\T  Y_t) \otimes ( Y_t^\T  Y_t)} = \| Y_t\|_2^2. 
			 \end{align}
			 We immediately have that, on the event $\cE_{t2}$,
			 \begin{equation}\label{bd_Ht}
			 	\|\bar H_t\|_2\le 2\| H_t\|_2 \le 2 \sigma_Y^2(r+4\log(n))
			 \end{equation}
			 almost surely. The rest of the proof works on the event $\cap_{t=1}^n \cE_{t2}$ to bound from above
			 \begin{align}\label{decomp_quad_H} 
			 	{1\over n}\left\| \sum_{t=1}^n  \left( \bar H_t  \bar H_t^\T 1\{\cE_{t2}\}    - \EE\left[\bar H_t  \bar H_t^\T 1\{\cE_{t2}\}\right]\right) \right\|_\op + {1\over n}\left\| \sum_{t=1}^n   \EE\left[ \bar H_t  \bar H_t^\T1\{\cE_{t2}^c\} \right]\right\|_\op.
			 \end{align}
			 The second term is bounded from above by  
			 \begin{align}\label{bd_term_2_H}\nonumber
			 	\sup_{ u\in \Sp^{r^2}}{1\over n}  \sum_{t=1}^n   \EE\left[( u^\T  \bar H_t)^2 1\{\cE_{t2}^c\} \right] & \le   \sup_{ u\in \Sp^{r^2}} \sqrt{\EE\left[( u^\T  \bar H_t)^4\right]}\max_{t\in [n]}\sqrt{\P(\cE_{t2}^c)}\\
			 	&\lesssim  ~  {\sigma_Y^4 \over \sqrt n} &&\text{by \cref{lem_u_Ht_four}}.
			 \end{align}
			 To bound the first term in \eqref{decomp_quad_H}, we apply the matrix-valued Bernstein inequality in \cref{lem_bernstein_mat}. From \eqref{bd_Ht}, we find that 
			 \[
			 \| \bar H_t  \bar H_t^\T 1\{\cE_{t2}\}\|_\op   \le \| \bar H_t\|_2^2 1\{\cE_{t2}\} \le  \sigma_Y^4(d+4\log(n))^2 
			 \]
			 almost surely. Also note that 
			 \begin{align*}
			 	\left\|\EE[ \bar H_t  \bar H_t^\T \bar H_t  \bar H_t^\T 1\{\cE_{t2}\}]\right\|_\op &\le \sup_{ u\in \Sp^{r^2}}\EE\left[
			 	( u^\T  \bar H_t)^2 \| \bar H_t\|_2^2
			 	\right]\\
			 	&\le \sup_{ u\in \Sp^{r^2}} \sqrt{\EE[( u^\T  \bar H_t)^4]}\sqrt{\EE[\| \bar H_t\|_2^4]}\\
			 	&\lesssim  \sigma_Y^4 \sqrt{\EE[\|H_t\|_2^4]} &&\text{by \cref{lem_u_Ht_four}}.
			 \end{align*}
		 	Since
			 \begin{equation}\label{bd_H_four}
			 	  \EE[\| H_t\|_2^4] \overset{\eqref{eq_Ht_Yt}}{=} \EE[\|Y_t\|_2^8] = \sum_{i,j,k,\ell \in [r]} \EE[Y_{it}^2 Y_{jt}^2 Y_{kt}^2 Y_{\ell t}^2] \lesssim \sigma_Y^8 r^4.
			 \end{equation}  
			 we conclude that 
			 \[
			 \left\|\EE\left[ \bar H_t  \bar H_t^\T \bar H_t  \bar H_t^\T 1\{\cE_{t2}\}\right]\right\|_\op \lesssim ~  \sigma_Y^8 r^2.
			 \]
			 Using $n \ge    r^2$ and applying \cref{lem_bernstein_mat} with $d = r^2$, $U = 4\sigma_Y^4(r+\log(n))^2$, $\sigma^2 = C \sigma_Y^8 n  r^2$ and $t = C \sigma_Y^4 \sqrt{n r^2 \log(n)}$
			 gives 
			 \[ 
			 {1\over n}\left\| \sum_{t=1}^n  \left( \bar H_t  \bar H_t^\T 1\{\cE_{t2}\}    - \E\left[ \bar H_t  \bar H_t^\T 1\{\cE_{t2}\}\right]\right) \right\|_\op \ge C\sigma_Y^4\sqrt{r^2\log(n) \over n}
			 \]
			 with probability at most 
			 $$
			 14\exp\left(
			 -{C^2 \sigma_Y^8 n r^2 \log(n) \over C \sigma_Y^8 n  r^2 + 16 C\sigma_Y^8(r+\log(n))^2 \sqrt{n r^2\log(n)}} + 2\log(r)
			 \right) \le n^{-C'}.
			 $$
			 Together with \eqref{bd_II_HH}, the proof is complete. 
		\end{proof}

			\cref{lem_HH} can be used to bound  from above $\|\sum_{t=1}^n[ (\bq^\T  Y_t)^2  (\bar\bq^\T  Y_t)  Y_t - \EE[(\bq^\T  Y_t)^2 (\bar\bq^\T  Y_t)  Y_t  ]] \|_2$ uniformly over $\bq, \bar \bq \in \Sp^r$. This is stated in the following lemma. 
			
			\begin{lemma}\label{lem_tensor_bounds}
				Under model \eqref{model_Y_hat}, assume $n\ge  r(r+\log n)$. With probability at least $1- n^{-C}$, one has 
				\[  
				\sup_{\bq,\bar\bq\in \Sp^r} {1\over   n}\left\|\sum_{t=1}^n\left\{(\bq^\T  Y_t)^2  (\bar\bq^\T  Y_t)  Y_t - \EE\left[(\bq^\T  Y_t)^2 (\bar\bq^\T  Y_t)  Y_t  \right]\right\}  \right\|_2   \lesssim ~ \sqrt{r^2\log(n)\over n}.
				\]
			\end{lemma}
			\begin{proof}
				By definition, it suffices to bound from above 
				\[
				\Delta  := \sup_{\bq,\bar{\bq},\bq'\in \Sp^r}  {1\over  n} \sum_{t=1}^n\left\{ (\bq^\T  Y_t)^2  \bar\bq^\T  Y_t  Y_t^\T \bq'   - \EE\left[(\bq^\T  Y_t)^2  \bar\bq^\T  Y_t  Y_t^\T \bq'\right]\right\}. 
				\]
				By \eqref{eq_kronecker}, 
				we have   
				\begin{align*}
					\Delta  &\le \sup_{\bq,\bar{\bq},\bq'\in \Sp^r}  {1\over  n} \sum_{t=1}^n (\bq' \otimes \bar\bq)^\T\left(H_t H_t^\T   - \EE[H_t H_t^\T]\right)(\bq \otimes \bq)  \\
					&\le  \sup_{\bq,\bar{\bq}\in \Sp^r}   {1\over  n} \sum_{t=1}^n  (\bq \otimes \bar\bq)^\T \left(H_t H_t^\T   - \EE[H_t H_t^\T]\right)(\bq \otimes \bar \bq). 
			\end{align*}
			Finally,  since using part (3) of \cref{fact_k_prod} twice gives 
			\begin{align*}
				(\bq \otimes \bar\bq)^\T H_t H_t^\T  (\bq \otimes \bar\bq) &= (\bq^\T Y_t)^2 (\bar \bq^\T Y_t)^2  = (\bq \otimes  \bq)^\T H_t H_t^\T  (\bar\bq \otimes \bar\bq),
			\end{align*}
			we conclude that 
			\begin{align*}
				\Delta & \le  \sup_{\bq \in \Sp^r}   {1\over  n} \sum_{t=1}^n  (\bq \otimes \bq)^\T \left(H_t H_t^\T   - \EE[H_t H_t^\T]\right)(\bq \otimes  \bq)
				\\
				& = \sup_{\bq \in \Sp^r}  {1\over n}\sum_{t=1}^n \left\{
				(\bq ^\T Y_t)^4 - \EE\left[(\bq ^\T Y_t)^4
				\right]\right\} &&\text{by part  (3) of \cref{fact_k_prod}}.
			\end{align*}
			Invoking \cref{lem_HH} proves the claim. 
		\end{proof}   
			
			Applications of \cref{lem_tensor_bounds} are used to  establish  the lipschitz continuity of $\|\sum_{t=1}^n[ (\bq^\T  Y_t)^3  Y_t - \EE[(\bq^\T  Y_t)^3  Y_t]]
			\|_2$ for any $\bq \in \Sp^r$ 
			to further control $\|\grad F(\cdot) - \grad f(\cdot)\|_2$. 
			
			\begin{lemma}\label{lem_grad_lip}
				Grant conditions in \cref{lem_tensor_bounds}. Fix any $\bq_0 \in \Sp^r$. With probability at least $1-n^{-C}$, the following holds for all $\bq\in \Sp^r$,  
				\begin{align*}
					{1\over  n} \left\|\sum_{t=1}^n\left[ (\bq^\T  Y_t)^3  Y_t - \EE[(\bq^\T  Y_t)^3  Y_t]\right]
					\right\|_2   \lesssim \sqrt{r\log(n) \over n} +  \|\bq-\bq_0\|_2 \sqrt{r^2\log(n)\over n} .
				\end{align*}
			\end{lemma}
			\begin{proof}
				Pick any $\bq \in \Sp^r$. 
				Adding and subtracting terms gives 
				\begin{align*}
					{1\over  n} \left\|\sum_{t=1}^n\left[ (\bq^\T  Y_t)^3  Y_t - \EE[(\bq^\T  Y_t)^3  Y_t]\right]
					\right\|_2   \le  \rI+\rII+\rIII + \rIV
				\end{align*}
				where, by writing $\bdelta = \bq - \bq_0$, 
				\begin{align*}
					&\rI = {1\over  n} \left\|\sum_{t=1}^n\left[ (\bq_0^\T  Y_t)^3  Y_t - \EE[(\bq_0^\T  Y_t)^3  Y_t]\right]
					\right\|_2\\
					&\rII = {1\over  n}\left\| 
					\sum_{t=1}^n\left[ (\bdelta^\T  Y_t)^3  Y_t - \EE[(\bdelta^\T  Y_t)^3  Y_t] 
					\right]\right\|_2\\
					&\rIII = {3\over n} \left\|\sum_{t=1}^n\left[ (\bdelta^\T  Y_t)^2 (\bq_0^\T  Y_t)  Y_t - \EE[(\bdelta^\T  Y_t)^2 (\bq_0^\T  Y_t)  Y_t] 
					\right]\right\|_2 \\
					&\rIV = {3\over n} \left\|\sum_{t=1}^n\left[ (\bdelta^\T  Y_t) (\bq_0^\T  Y_t)^2  Y_t - \EE[(\bdelta^\T  Y_t) (\bq_0^\T  Y_t)^2  Y_t] 
					\right]\right\|_2.
				\end{align*}
				Notice that \cref{lem_bd_grad_pointwise} can be used to bound from above the term $\rI$ as 
				\[
					\P\left\{
					\rI \lesssim \sqrt{r\log(n) \over n}
					\right\} \ge 1-n^{-C}.
				\]
				For $\rII$, we have that 
				\[
				\rII \le \|\bdelta\|_2^3 ~ \sup_{ u \in \Sp^r} {1\over  n}\left\| 
				\sum_{t=1}^n\left[ ( u^\T  Y_t)^3  Y_t - \EE[( u^\T  Y_t)^3  Y_t] 
				\right]\right\|_2. 
				\]
				Similarly, we have 
				\[
				\rIII \le \|\bdelta\|_2^2 ~  \sup_{ u \in \Sp^r} {3\over n} \left\|\sum_{t=1}^n\left[ ( u^\T  Y_t)^2 (\bq_0^\T  Y_t)  Y_t - \EE[( u^\T  Y_t)^2 (\bq_0^\T  Y_t)  Y_t] 
				\right]\right\|_2
				\]
				and 
				\[
				\rIV \le \|\bdelta\|_2  ~  \sup_{ u \in \Sp^r} {3\over n} \left\|\sum_{t=1}^n\left[ ( u^\T  Y_t) (\bq_0^\T  Y_t)^2  Y_t - \EE[( u^\T  Y_t) (\bq_0^\T  Y_t)^2  Y_t] 
				\right]\right\|_2.
				\]
				Invoking \cref{lem_tensor_bounds} and using the fact that $\|\bdelta\|_2  \le 2$ gives that 
				\[
				\P\left\{\rII + \rIII  + \rIV \lesssim \|\bdelta\|_2 \sqrt{r^2\log(n) \over n}
				\right\} \ge 1-3 n^{-C},
				\]
				completing the proof. 
			\end{proof}

			The following lemma provides concentration inequality of $\sum_{t=1}^n Y_t   Y_t^\T   (Y_t\otimes Y_t)^\T\rvec(\bR\bG\bR^\T)$ around its mean.  
			
			\begin{lemma}\label{lem_diff_YGY}
				Grant \eqref{model_Y_hat}. Let $\bG\in\R^{r\times r}$ contain i.i.d. standard normal variables and $\bR$ be any $r\times r$ orthogonal matrix that is independent of $\bG$. Then for all $t\ge 0$,  with probability at least $1- \exp(
				-t^2 +\log r) - r^2 n^{-C}$,
				\begin{align*}
					& \left\|
					\sum_{t=1}^n \left( Y_t   Y_t^\T   (Y_t\otimes Y_t)^\T\rvec(\bR\bG\bR^\T)- 
					\EE\left[Y_t   Y_t^\T   (Y_t\otimes Y_t)^\T\right]\rvec(\bR\bG\bR^\T)
					\right)\right\|_\op \\
					&\lesssim t \sqrt{n r^3\log(n)}.
				\end{align*} 
			\end{lemma}
			\begin{proof}
				By conditioning on $\bR$ and $\bY$,  observe that 
				\begin{align*}
					\rvec(\bR \bG\bR^\T ) & = (\bR \otimes \bR) \rvec(\bG)&&\text{by part (4) of \cref{fact_k_prod}}\\
					&\sim \cN_{r^2}\left(0,  (\bR \otimes \bR)^\T  (\bR \otimes \bR)\right).
				\end{align*}
				Since  
				\[
				(\bR \otimes \bR)^\T  (\bR \otimes \bR) = (\bR^\T \bR) \otimes  (\bR^\T \bR)  = \bI_r \otimes \bI_r = \bI_{r^2}
				\]
				we conclude that $ [\bR \bG\bR^\T]_{ab}$ for $a,b\in [r]$ are i.i.d. $\cN(0,1)$.  By further noting that the quantity we want to bound equals to 
				\[
				{1\over n}  \left\|\sum_{a,b=1}^r  \bigl[\bR  \bG\bR^\T \bigr]_{ab}  ~  \sum_{t=1}^n \left( Y_t   Y_t^\T   Y_{at} Y_{bt}   - 
				\EE[Y_t   Y_t^\T Y_{at} Y_{bt}  ]
				\right)\right\|_\op,
				\]
				applying \cref{lem_op_Gauss_mat} with $d =r$, $\bA_k =   n^{-1} \sum_{t=1}^n \left( Y_t   Y_t^\T   Y_{at} Y_{bt}   - 
				\EE[Y_t   Y_t^\T Y_{at} Y_{bt} ]\right)$ and 
				\[
				\nu =  {1\over n^2}\left\| \sum_{a,b=1}^r   \left[ \sum_{t=1}^n \left( Y_t   Y_t^\T   Y_{at} Y_{bt}   - 
				\EE[Y_t   Y_t^\T Y_{at} Y_{bt} ]\right)\right]^2 \right\|_\op
				\]
				yields that  for all $t\ge 0$,  
				\[
				{1\over  n}\left\|
				\sum_{t=1}^n \left( Y_t   Y_t^\T   (Y_t\otimes Y_t)^\T\rvec(\bR\bG\bR^\T)- 
				\EE\left[Y_t   Y_t^\T   (Y_t\otimes Y_t)^\T\right]\rvec(\bR\bG\bR^\T)
				\right)\right\|_\op  \ge t   
				\]
				with probability no greater than $r \exp[
				-t^2  /  (2 \nu)]$.
				Finally, by invoking \cref{lem_bd_YY_pointwise} with  $q = \mb e_a$, $\bar q = \mb e_b$ and taking the union bounds over $a,b\in [r]$, we have that with probability at least $1-r^2 n ^{-C}$,	
				\begin{align*}
					\nu &\le {r^2\over n^2} \max_{a,b \in [r]} \left\|\sum_{t=1}^n \left( Y_t   Y_t^\T   Y_{at} Y_{bt}   - 
					\EE[Y_t   Y_t^\T Y_{at} Y_{bt} ]\right)  \right\|_\op^2 \lesssim ~ {r^3\log(n) \over n},
				\end{align*} 
				thereby completing the proof. 
			\end{proof}

			\subsection{Technical lemmas used in the proof of \cref{app_sec_concentration}}

			The following states simple higher moments bounds of sub-Gaussian random variables. We include it here for completness.

			\begin{lemma}\label{lem_moment} 
				Assume that $W\in \R^r$ is a zero-mean,  sub-Gaussian random vector with sub-Gaussian constant $\sigma_W$. 
				For any $a,b\in \Z^+$, we have 
				\[
				\sup_{ u, v\in\Sp^r}  \EE\left[
				| u^\T   W |^a 	| v^\T   W|^b
				\right] \le  C(a,b)  \sigma_W^{a+b}
				\]
				where $C(a,b)>0$ is some absolute constant depending on $a$ and  $b$ only. 
			\end{lemma}
			\begin{proof}
				Pick any  $ u,  v \in \Sp^r$. 
				By Cauchy-Schwarz inequality, 
				\[
				\EE\left[
				| u^\T   W |^a 	| v^\T   W |^b
				\right]  \le \sqrt{
					\EE\left[
					| u^\T   W |^{2a} \right]}
				\sqrt{
					\EE\left[
					| v^\T   W |^{2b} \right]
				}.
				\]
				Note that $ u^\T   W$ is also sub-Gaussian with sub-Gaussian constant $\sigma_W$. This implies 
				\begin{align*}
					\EE\left[
					| u^\T   W|^{2a} \right] 
					&\le C(a)  \sigma_W^{2a}.
				\end{align*}
				Repeating the same argument for $	\EE[
				| v^\T   W|^{2b} ]$ completes the proof. 
			\end{proof}

			The following lemma contains the bound of the $4$th moment of the kronecker product between two sub-Gaussian random vectors.

			\begin{lemma}\label{lem_u_Wt_four}
				Let $ W = \bSigma_W^{1/2} \wt W\in\R^d$ where $\wt W$ is a $\gamma$ subGaussian random vector containing independent entries and $\bSigma_W$ is any symmetric semi-positive definite matrix. Assume $\EE[\wt W] = 0$ and $\Cov(\wt W) = \bI_d$.  Then there exists   some absolute constant $C>0$ such that 
				\[
				\sup_{ u\in \Sp^{r^2}} \EE\left[
				\left( u^\T (W \otimes W  - \EE[W\otimes W])    \right)^4
				\right] \le C\gamma^8 \|\bSigma_W\|_\op^4.
				\]
			\end{lemma}
			\begin{proof}
				For any $u\in \Sp^{r^2}$, we write $\bU \in \R^{r\times r}$ as the matrix such that $u = \rvec(\bU)$.  By part (4) of \cref{fact_k_prod}, we find that   
				\[
				W \otimes W  = (\bSigma_W^{1/2} \wt W)\otimes (\bSigma_W^{1/2} \wt W) =  (\bSigma_W^{1/2} \otimes \bSigma_W^{1/2}) (\wt W \otimes \wt  W),
				\]
				hence,  
				\begin{align*}
					&u^\T (W \otimes W  - \EE[W\otimes W])\\ &= u^\T (W \otimes W  - \rvec(\bSigma_W)) &&\text{by \eqref{exp_Ht}}\\
					& =  u^\T  (\bSigma_W^{1/2} \otimes \bSigma_W^{1/2}) 
					\left(
					\wt W \otimes \wt W  - \rvec(\bI_d)
					\right) &&\text{by part (4) of \cref{fact_k_prod}}\\
					&=  u^\T  (\bSigma_W^{1/2} \otimes \bSigma_W^{1/2}) 
					\left(
					\wt W \otimes \wt W  - \EE[\wt W \otimes \wt W]
					\right) &&\text{by \eqref{exp_Ht}}.
				\end{align*} 
				Furthermore, we also have
				\[
				u^\T  (\bSigma_W^{1/2} \otimes \bSigma_W^{1/2})  (\bSigma_W^{1/2} \otimes \bSigma_W^{1/2})^\T u = u^\T  (\bSigma_W \otimes \bSigma_W)  u \le \|\bSigma_N\|_\op^2
				\] 
				where the last step uses part (6) of \cref{fact_k_prod}.  It then follows that 
				\[
				\sup_{ u\in \Sp^{r^2}} \EE\left[
				\left( u^\T (W \otimes W  - \EE[W\otimes W])    \right)^4
				\right]  \le  \|\bSigma_W\|_\op^4  \sup_{ u\in \Sp^{r^2}} \EE\left[
				\left( u^\T (\wt W \otimes \wt W  - \EE[\wt W\otimes \wt W])    \right)^4
				\right].
				\]
				Note that part (3) of \cref{fact_k_prod} gives 
				\[
				u^\T (\wt W \otimes \wt W  - \EE[\wt W\otimes \wt W)] = \wt W^\T \bU \wt W -  \EE[ \wt W^\T \bU \wt W].
				\]
				An application of the Hanson-Wright inequality gives that, for any $t>0$,
				\[
				\P\left\{
				\left |\wt W^\T \bU \wt W -  \EE[ \wt W^\T \bU \wt W]\right| > t
				\right\} \le 2\exp\left\{
				-c \min\left\{
				{t^2 \over \gamma^4 \|\bU\|_F^2}, ~ {t\over \gamma^2 \|\bU\|_\op}
				\right\}
				\right\}.
				\]
				Together with $\|\bU\|_F^2 = \|u\|_2^2 = 1$  and $ \|\bU\|_\op\le \|\bU\|_F \le 1$, we conclude that $(\wt W^\T \bU \wt W -  \EE[ \wt W^\T \bU \wt W])$ is sub-Exponential with sub-Exponential parameters $C \gamma^2$ for some absolute constant $C>0$. Therefore, for any $u\in \Sp^{r^2}$,    
				\begin{align*}  
					\EE\left[ \left(\wt W^\T \bU \wt W -  \EE[ \wt W^\T \bU \wt W] 
					\right)^4 \right]  \lesssim    \gamma^8.
				\end{align*}
				This completes the proof.  
			\end{proof}

			The following lemma bounds from above the fourth moment of $u^\T (Y_t\otimes Y_t - \EE[Y_t\otimes Y_t])$ for any fixed $u\in \Sp^{r^2}$ and $t\in [n]$.

			\begin{lemma}\label{lem_u_Ht_four}
				Under model \eqref{model_Y_hat}, there exists some absolute constant $C>0$ such that 
				\[
					\max_{t\in [n]} \sup_{ u\in \Sp^{r^2}} \EE\left[\left( u^\T (Y_t\otimes Y_t - \EE[Y_t\otimes Y_t]) \right)^4\right] \lesssim \sigma_Y^4.
				\]
			\end{lemma}
			\begin{proof}
				Recall $Y_t = \bA Z_t + N_t$. For any $u\in \Sp^{r^2}$, using parts (1) and (2) of \cref{fact_k_prod} gives 
				\begin{align*}
					&u^\T (Y_t\otimes Y_t - \EE[Y_t\otimes Y_t])\\
					 &= u^\T \left(
					(\bA Z_t  + N_t) \otimes 	(\bA Z_t  + N_t)
					\right) - u^\T (\rvec(\bSigma_Y)) &&\text{by \eqref{exp_Ht}}\\
					&= u^\T \left(
					(\bA Z_t)  \otimes 	(\bA Z_t) - \rvec(\bI_r) 
					\right)  + u^\T \left(
					 N_t  \otimes  N_t  - \rvec(\bSigma_N) 
					\right)\\
					&\quad 
					+ 
					u^\T \left(
					(\bA Z_t) \otimes  N_t
					\right)  + u^\T \left(
					 N_t \otimes (\bA Z_t) 
					\right).
				\end{align*} 
				Note that $N_t = \bSigma_N^{1/2} \wt N_t$ with $\wt N_t \sim N(0, \bI_r)$. Applying \cref{lem_u_Wt_four} twice gives that 
				\begin{align*}
					&\max_{t\in [n]} \sup_{ u\in \Sp^{r^2}} \EE\left[
					\left( u^\T \left(
					(\bA Z_t)  \otimes 	(\bA Z_t) - \rvec(\bI_r) 
					\right)\right) ^4
					\right] \lesssim \gamma_Z^4\\
					&\max_{t\in [n]} \sup_{ u\in \Sp^{r^2}} \EE\left[
					\left( u^\T \left(
					 N_t  \otimes  N_t- \rvec(\bSigma_N) 
					\right)\right)^4
					\right] \lesssim  \|\bSigma_N\|_\op^4.
				\end{align*}
				Finally, it remains to bound from above 
				\[
					\max_{t\in [n]} \sup_{ u\in \Sp^{r^2}} \EE\left[
					\left( u^\T \left(
					(\bA Z_t)  \otimes 	N_t
					\right)\right) ^4\right].
				\] 
				To this end, recall that for any $u\in \Sp^{r^2}$ we have $\bU = \rvec(u)$ and
				\[
					u^\T \left(
					(\bA Z_t)  \otimes 	N_t\right) =  N_t^\T \bU \bA Z_t.
				\]
				Conditioning on $Z_t$,  the fact that $N_t^\T \bU \bA Z_t \sim \cN(0, Z_t^\T \bA^\T \bU^\T\bU \bA Z_t)$ gives
				\begin{align*}
					\EE\left[
					\left( u^\T \left(
					(\bA Z_t)  \otimes 	N_t
					\right)\right) ^4\right] &= 3 \EE\left[
					 \left(Z_t^\T \bA^\T \bU^\T \bU \bA Z_t\right)^2 \right]\\
					 &= 3 \left(
					 \EE[Z_t^\T \bA^\T \bU^\T \bU \bA Z_t]
					 \right)^2 + 3 \Var(Z_t^\T \bA^\T \bU^\T \bU \bA Z_t)\\
					 &= 3\left(\tr(\bA^\T \bU^\T \bU \bA)\right)^2 + 3\tr(\bA^\T \bU^\T \bU \bA)\\
					 &=6. 
				\end{align*}
				The proof is complete. 
			\end{proof}

		\section{Auxiliary lemmas}
		
		\subsection{Basic properties on the operation of  kronecker product}
		The following fact contains some basic properties of the kronecker product. 
		
		\begin{fact}\label{fact_k_prod}
			Let $\bA$, $\bB$, $\bC$  and $\bD$ be commensurate matrices whose dimenions may vary across each property. Then 
			\begin{enumerate}
				\item[(1)] $\bA \otimes (\bB+ \bC) = \bA \otimes \bB + \bA \otimes \bC$;
				\item[(2)] $ (\bB+ \bC)  \otimes   \bA   = \bB \otimes \bA  + \bC \otimes \bA $;
				\item[(3)] $(\bA\otimes \bB)  (\bC \otimes \bD) = (\bA \bC) \otimes (\bB \bD)$;
				\item[(4)] $(\bA \otimes \bB) {\rm vec}(\bM) = {\rm vec}(\bB \bM \bA^\T)$ with ${\rm vec}$ being the vectorization operation which stacks  columns of a given matrix.
				\item[(5)] There exists permutation matrices $\bP$ and $\bQ$ such that $\bA\otimes \bB = \bP (\bB \otimes \bA) \bQ$.
				\item[(6)] $\|\bA \otimes \bB\|_\op = \|\bA\|_\op \|\bB\|_\op$. See, for instance, \cite{kron_norm}.
			\end{enumerate}
		\end{fact}

		\subsection{Auxiliary deviation inequalities on random variables, vectors and matrices}
		Our proof has used the following anti-concentration inequality for standard Gaussian random variable, (see, for instance, \cite{carbery2001distributional}).
		
		\begin{lemma}\label{lem:Anti_concen_gauss}
			Let $X\sim \cN(0,1)$. Then, for any $t\ge 0$,
			\[ 
			\P\left\{X^2\leq t^2 \right\} \leq  t\sqrt{e}.
			\]
		\end{lemma}

		The following lemma is the tail inequality for a quadratic form of sub-Gaussian random vectors. We refer to \citet[Lemma 16]{bing2020prediction} for its proof.  

		\begin{lemma} \label{lem_quad}
			Let $\xi\in \R^d$ be a $\gamma_\xi$ sub-Gaussian random vector. Then, for all symmetric positive semi-definite matrices $H$, and all $t\ge 0$, 
			\[
			\P\left\{
			\xi^\T H\xi > \gamma_\xi^2\left(
			\sqrt{{\rm tr}(H)}+ \sqrt{2 t \|H\|_{\rm op} }
			\right)^2
			\right\} \le e^{-t}.
			\] 
		\end{lemma}

		The following lemmas state the well-known vector-valued and matrix-valued Bernstein inequalities (see, for instance, \citet[Theorem 3.1, Corollary 3.1 and Corollary 4.1]{Minsker2017}). 
		\begin{lemma}[Vector-valued Bernstein inequality]\label{lem_bernstein_vector}
			Let $X_1,\ldots, X_n$ be independent random vectors with zero mean and $\max_{i\in [n]}\|X_i\|_2 \le U$ almost surely. Denote $\sigma^2 := \sum_{i=1}^n \EE[\|X_i\|_2^2]$. Then for all $t \ge {1\over 6}(U + \sqrt{U^2 + 36\sigma^2})$, 
			\[
			\P\left(
			\left\|\sum_{i=1}^n  X_i\right\|_2 > t
			\right) \le 28\exp\left(
			- {t^2 / 2\over \sigma^2 + Ut/3}
			\right).
			\]
		\end{lemma}

		\begin{lemma}[Matrix-valued Bernstein inequality]\label{lem_bernstein_mat}
			Let $\bX_1,\ldots, \bX_n \in \R^{d\times d}$ be independent, symmetric random matrices with zero mean and $\max_{i\in [n]}\|\bX_i\|_\op \le U$ almost surely. Denote $\sigma^2 := \|\sum_{i=1}^n \EE[\bX_i^2]\|_\op$. Then for all $t \ge {1\over 6}(U + \sqrt{U^2 + 36\sigma^2})$, 
			\[
			\P\left(
			\left\|\sum_{i=1}^n \bX_i\right\|_\op > t
			\right) \le 14   \exp\left(
			- {t^2 / 2\over \sigma^2 + Ut/3} + \log(d)
			\right).
			\]
		\end{lemma}
	

		The following lemma bounds from above the operator norm of the summation of a  matrix Gaussian series, see, \citet[Theorem 4.6.1]{tropp2015introduction}.
		
		\begin{lemma}[Matrix Gaussian Series]\label{lem_op_Gauss_mat}
			Consider a finite sequence $\bA_k$ of fixed symmetric matrices with dimension $d$, and let $\gamma_k$ be a finite sequence of independent standard normal variables. Introduce the matrix Gaussian series 
			\[
				\bG = \sum_k \gamma_k \bA_k. 
			\]
			Let $\nu(\bG)$ be the matrix variance statistic of the sum:
			$
				\nu(\bG) = \|\EE[\bG^2]\|_\op = \|\sum_k \bA_k^2\|_\op.
			$
			Then for all $t\ge 0$, 
			\[
				\P\left\{
					\|\bG\|_\op \ge t
 				\right\} \le d\exp\left\{
 				-{t^2 \over 2\nu(\bG)}
 				\right\}.
			\]
			
		\end{lemma}

		Another useful concentration inequality of the operator norm of the random matrices with i.i.d. sub-Gaussian rows is stated in the following lemma \cite[Lemma 16]{bing2020prediction}. This is an immediate result of \citet[Remark 5.40]{vershynin_2012}.
		
		\begin{lemma}\label{lem_op_diff} 
			Let $\bG$ be $n$ by $d$ matrix whose rows are i.i.d. $\gamma$ sub-Gaussian  random vectors with covariance matrix $\bSigma_G$. Then, for every $t\ge 0$, with probability at least  $1-2e^{-ct^2}$,
			\[
			\left\|	{1\over n}\bG^\T \bG - \bSigma_G\right\|_{\op}\le \max\left\{\delta, \delta^2\right\} \left\|\bSigma_G\right\|_{\op},
			\]
			with $\delta = C\sqrt{d/n}+ t/\sqrt n$ where $c,C$ are positive constants depending on $\gamma$.
		\end{lemma}
		
		 The following lemma provides an upper bound on the operator norm of $\bG \bH \mb G^\T$ where  $\bG\in \R^{n\times d}$ is a random matrix and its rows are independent sub-Gaussian random vectors. It is proved in Lemma 22 of \cite{bing2020prediction}.
		\begin{lemma}\label{lem_op_norm}
			Let $\bG$ be a $n\times d$ matrix with rows that are independent $\gamma$ sub-Gaussian  random vectors with identity covariance matrix. Then, for all symmetric positive semi-definite matrices $\bH$, 
			\[
			\P\left\{{1\over n}\| \bG \bH \bG^\T \|_{{\rm op}} \le \gamma^2\left( \sqrt{{\rm tr}(\bH) \over n} + \sqrt{6\|\bH\|_{\op}}
			\right)^2\right\} \ge  1 -  e^{-n}
			\]
		\end{lemma}

		\subsection{Auxiliary lemmas on matrix perturbation theory}

		   We first state the Weyl's inequality, widely used for perturbation on singular values of a given matrix \cite{Weyl}. 
		   
		   \begin{lemma}\label{lem_weyl}
		   		Let $\bM,\bDelta\in \R^{p\times q}$ be any matrix. The singular values of $\bM$ and $\bDelta$ satisfy
		   		 \[
		   		 		\max_{1\le k \le \min(p,q)} |\sigma_k(\bM + \bDelta) - \sigma_k(\bM) | \le  \sigma_1( \bDelta).
		   		 \]
		   		 Here $\sigma_1(\bM)\ge \cdots \ge \sigma_{\min(p,q)}(\bM)$ are singular values of $\bM$ in non-increasing order.
		   \end{lemma}

			The following is a variant of the well-known Davis-Kahan sin$\theta$ theorem for symmetric matrices (see, \citet[Corollary 1]{yu2015useful} and also \cite{DavisKahan,stewart1990matrix}).  To state it, let $\bV, \bU \in \bbO_{p\times r}$ have orthonormal columns with principal angles defined as $(\cos^{-1}\sigma_1, \ldots, \cos^{-1}\sigma_d)^\T$ where $\sigma_1 \ge \cdots \ge \sigma_d$ are the singular values of $\bV^\T \bU$. Let 
			$
				\mb\Theta(\bU, \bV) = \diag(\cos^{-1}\sigma_1, \ldots, \cos^{-1}\sigma_d)\in \R^{d\times d}
			$
			and $\sin \mb\Theta(\bU, \bV)$ be defined entrywise. 
			
			\begin{lemma}[Davis-Kahan sin$\theta$ theorem]\label{lem_sin_theta}
				Suppose that $\bSigma, \wh\bSigma\in \R^{p\times p}$ be symmetric with eigenvalues $\lambda_1 \ge \cdots \ge \lambda_p$ and $\wh \lambda_1 \ge \cdots \ge \wh \lambda_p$, respectively. Fix $1\le r\le s \le \rank(\bA)$ and assume that $\min(\lambda_{r-1}- \lambda_r, \lambda_s - \lambda_{s+1}) > 0$, where we define $\lambda_0 = \infty$ and $\lambda_{p+1}= -\infty$. Let $d = s-r+1$, and let $\bV = (v_r, v_{r+1},\ldots, v_s)\in \R^{q\times d}$ and  $\wh \bV = (\wh v_r, \wh v_{r+1},\ldots, \wh v_s)\in \R^{q\times d}$ have orthonormal columns satisfying $\bSigma v_j = \lambda_j v_j$ and $\wh \bSigma  \wh v_j = \wh \lambda_j \wh v_j$  for $j = r, r+1, \ldots, s$. Then 
				\[
				\|\sin \mb\Theta(\wh \bV, \bV)\|_F \le {2 \min(
					d^{1/2} \|\wh \bSigma - \bSigma\|_\op,  \|\wh \bSigma - \bSigma\|_F)
					\over\min(\lambda_{r-1} - \lambda_r, \lambda_s - \lambda_{s+1}) }.
				\]
				Moreover, there exists an orthogonal matrix $\wh \bO \in \bbO_{d\times d}$ such that 
				\[
				\|\wh \bV \wh \bO - \bV  \|_F \le {2^{3/2}\min(
					d^{1/2} \|\wh \bSigma - \bSigma\|_\op,  \|\wh \bSigma - \bSigma\|_F)
					\over\min(\lambda_{r-1} - \lambda_r, \lambda_s - \lambda_{s+1}) }.
				\]
			\end{lemma}
			
			The next lemma contains the Davis-Kahan sin$\theta$ theorem for general rectangle matrices (see, \citet[Theorem 3]{yu2015useful})
			
			\begin{lemma}[Davis-Kahan sin$\theta$ theorem for general rectangle matrices]\label{lem_sin_theta_asym}
				Suppose that $\wh\bA,\bA\in \R^{p\times q}$ have singular values $\sigma_1 \ge \cdots \ge \sigma_{\min(p,q)}$ and $\wh \sigma_1 \ge \cdots \ge \wh \sigma_{\min(p,q)}$, respectively. Fix $1\le r\le s \le \rank(\bA)$ and assume that $\min(\sigma_{r-1}^2 - \sigma_r^2, \sigma_s^2 - \sigma_{s+1}^2) > 0$, where we define $\sigma_0^2 = \infty$ and $\sigma_{\rank(\bA)+1}^2 = -\infty$. Let $d = s-r+1$, and let $\bV = (v_r, v_{r+1},\ldots, v_s)\in \R^{q\times d}$ and  $\wh \bV = (\wh v_r, \wh v_{r+1},\ldots, \wh v_s)\in \R^{q\times d}$ have orthonormal columns satisfying $\bA v_j = \sigma_j u_j$ and $\wh \bA  \wh v_j = \wh \sigma_j \wh u_j$  for $j = r, r+1, \ldots, s$. Then 
				\[
				\|\sin \mb\Theta(\wh \bV, \bV)\|_F \le {2(2\sigma_1 + \|\wh \bA - \bA\|_\op )\min(
					d^{1/2} \|\wh \bA - \bA\|_\op,  \|\wh \bA - \bA\|_F)
					\over\min(\sigma_{r-1}^2 - \sigma_r^2, \sigma_s^2 - \sigma_{s+1}^2) }.
				\]
				Moreover, there exists an orthogonal matrix $\wh \bO \in \bbO_{d\times d}$ such that 
				\[
				\|\wh \bV \wh \bO - \bV  \|_F \le {2^{3/2}(2\sigma_1 + \|\wh \bA - \bA\|_\op )\min(
					d^{1/2} \|\wh \bA - \bA\|_\op,  \|\wh \bA - \bA\|_F)
					\over\min(\sigma_{r-1}^2 - \sigma_r^2, \sigma_s^2 - \sigma_{s+1}^2) }.
				\]
			\end{lemma}

	\end{appendix}

\end{document}